%% file: neurips-arXiv-20260511.tex
\documentclass{article}

\usepackage[preprint]{neurips_2026}

\usepackage[utf8]{inputenc} 
\usepackage[T1]{fontenc}    
\usepackage{url}            
\usepackage{booktabs}       
\usepackage{amsfonts}       
\usepackage{nicefrac}       
\usepackage{microtype}      
\usepackage{xcolor}         
\usepackage{subcaption}
\usepackage{tikz-cd} 
 
\usetikzlibrary{arrows.meta,positioning,fit,calc,
                decorations.pathreplacing,
                shapes.geometric}

\usepackage[english]{babel}
\usepackage{amssymb,dsfont,amsmath,amsthm,mathtools}
\usepackage{graphicx}
\usepackage[colorlinks=true, allcolors=blue]{hyperref}
\usepackage{tikz-cd}
\usepackage{pb-diagram}
\usepackage{comment}

\newtheorem{lemma}{Lemma}
\newtheorem{theorem}{Theorem}
\newtheorem{corollary}{Corollary}
\newtheorem{proposition}{Proposition}
\newtheorem{definition}{Definition}

\newtheorem{remark}{Remark}

\newcommand{\Pp}{\mathcal{P}}

\newcommand{\eqdef}{\mathrel{\mathop:}=}
\DeclareMathOperator{\id}{id}
\newcommand{\W}{W_{2}}

\newcommand{\R}{\mathbb{R}}
\newcommand{\RR}{\mathbb{R}}
\newcommand{\dotp}[2]{\langle #1,\,#2\rangle}
\renewcommand{\d}[1]{\mathrm{d}}

\renewcommand*\d{\mathop{}\!\mathrm{d}}

\newcommand{\Xspace}{X}

\newcommand{\Takashi}[1]{{\color{magenta}[\textit{#1} Takashi]}}\newcommand{\Matti}[1]{{\color{red}[Matti: \textit{#1}]} } \newcommand{\Ivan}[1]{{\color{cyan}[Ivan: \textit{#1}]}}

\newcommand{\footnoteforchecking}[1]{}

\def\bra{\langle} 
 
\def\ket{\rangle} 

\newcommand{\commented}[1]{}

\newcommand{\ba}{\begin{eqnarray*}}
\newcommand{\ea}{\end{eqnarray*}}

\def\R{\mathbb R}
\def\p{\partial}
\def \supp {\hbox{supp }} 
\def \diam {\hbox{diam }} 
\def \dist {\hbox{dist}} 
 
\def \det {\hbox{det}\,} 
\def\bra{\langle}

\DeclareMathOperator*{\argmin}{arg\,min}

\newcommand{\e}{\varepsilon}

\newcommand{\beq}{\begin{eqnarray}}
\newcommand{\eeq}{\end{eqnarray}}

\newcommand{\norm}[1]{\left\|#1 \right\|}

\newcommand{\abs}[1]{\left| #1 \right|}

\newcommand{\N}{{\mathbb N}}
\newcommand{\Z}{{\mathbb Z}}

\newcommand{\beqno}{\begin{eqnarray*}}
\newcommand{\eeqno}{\end{eqnarray*}}

\newcommand{\beqla}[1] {\begin {eqnarray}\label{#1}}

\def \eps {\varepsilon}
\def \e {\varepsilon}

\def \beq {\begin {eqnarray}}
\def \eeq {\end {eqnarray}}
\def \ba {\begin {eqnarray*}}
\def \ea {\end  {eqnarray*}}

\def \d {\delta}

\def \bra{\langle}

\def \supp {\hbox{supp }}

\title{
Function graph transformers universally approximate operators between  function spaces}

\author{
    Takashi Furuya \\
    Doshisha University, RIKEN AIP \\
    \texttt{tfuruya@mail.doshisha.ac.jp}
    \And
    David Mis \\
    Rice University \\
    \texttt{david.mis@rice.edu}
    \And
    Ivan Dokmani\'{c} \\
    University of Basel \\
    \texttt{ivan.dokmanic@unibas.ch}
    \And
    Maarten V. de Hoop \\
    Simons Chair in Computational and Applied Mathematics and Earth Science \\
    Rice University, Houston TX, USA \\
    \texttt{mdehoop@rice.edu}
    \And
    Matti Lassas \\
    University of Helsinki
    \\
    \texttt{matti.lassas@helsinki.fi}
}

\begin{document}

\maketitle


\begin{abstract}
\commented{Here is the earlier version (most likely written by Maarten) that I replaced by a newer version that Ivan suggested. OLDER VERSION: We introduce \emph{function graph transformers}, a subclass of measure-theoretic transformers that is structurally designed for operator learning. The central idea is to represent a function by a measure supported on its graph, so that finite tokenizations arise as empirical approximations of continuum graph measures. Within this representation, operator learning becomes the problem of learning transformers that map input graph measures to output graph measures, thereby realizing nonlinear operators between function spaces through a lifting to graph measure space followed by projection back to functions. Although function graph transformers impose additional structure adapted to operator learning, this structure entails \emph{no loss of universality}: we prove a universal approximation theorem showing that they retain the expressive power needed to approximate broad classes of nonlinear operators. In this sense, function graph transformers identify a mathematically precise subclass of general transformers that is both universal and intrinsically compatible with function-space learning. This framework also clarifies, at a theoretical level, why transformers are suitable for operator learning. It provides a continuum foundation for increasing number of tokens and discretization-invariant regimes. It makes explicit the roles of positional encoding, graph structure, regularization, and stability under refinement of the input discretization. The framework extends beyond classical continuous functions to operators between negative-order Sobolev spaces, thereby accommodating nonsmooth and distributional inputs. The resulting theory provides both a mathematical foundation and architectural guidance for transformer-based operator learning. It suggests training strategies that respect the geometry of function spaces, preserve graph measure structure, and ensure consistency across discretizations.
 }

We study the approximation of nonlinear operators between function spaces by transformers. Our approach is to lift functions to measures supported on their graphs and leverage a recently introduced measure-theoretic view of transformers. A function $h$ is represented by its graph measure $\gamma_h$, with finite tokens $\{(x_j,h(x_j))\}_{j=1}^N$ being its empirical approximations. 
We show that this framework elegantly models discretization refinement via convergence of measures and provides a natural setting for operator learning. 
Within this framework, we introduce \emph{function graph transformers}, a graph-preserving subclass of measure-theoretic transformers that maps graph measures to graph measures, which is to say that outputs remain single-valued functions. Crucially, this additional structure does not reduce generality: we prove that the resulting graph-preserving maps can be approximated by finite compositions of standard softmax self-attention layers and pointwise MLPs, yielding universal approximation results for broad classes of nonlinear operators. 
Unlike existing theoretical approaches to operator learning with transformers, the measure-theoretic framework also accommodates regularized negative-order Sobolev inputs for which discretization invariance is particularly challenging, as well as query points on different output domains. Overall, function graph transformers provide a continuum viewpoint and mathematical toolkit for transformer-based operator learning, clarifying the roles of positional encodings, graph structure, regularization, and ensuring consistency across discretizations.
\end{abstract}


\commented{\color{blue}

SUGGESTION BY CHATGPT FOR THE INTRODUCTION (BASED ON OUR PAPER AND THE EARLIER TEXT BY MAARTEN)

\section{Introduction}

Operator learning concerns the approximation of maps between
infinite-dimensional spaces of functions. In applications to partial
differential equations, inverse problems, and continuum dynamics, the
target map is not merely a map between two vectors on a fixed grid, but
a continuum operator
\[
  A:X\to Y,
\]
where $X$ and $Y$ are function spaces over spatial domains. A central
requirement is therefore \emph{discretization invariance}: the learned
model should depend on the underlying function, not on a particular
choice or ordering of sample points. This requirement is especially
important for irregular meshes, adaptive discretizations, scattered
measurements, and operators whose input and output are sampled on
different point clouds.

Transformers provide a natural language for such problems. The
self-attention mechanism of \citet{vaswani2017attention} acts on a
collection of tokens and is equivariant with respect to permutations of
their ordering. In operator learning, a sampled function can be encoded
as a set of graph tokens
\[
  \{(x_j,h(x_j))\}_{j=1}^N,
\]
where $x_j$ is a spatial coordinate and $h(x_j)$ is the observed value.
The attention weights then define data-dependent interactions between
spatial locations. In the continuum limit, these interactions are
naturally interpreted as nonlinear integral operators acting on
measures. This observation is the starting point of the present work.

Our basic object is the \emph{graph measure}
\[
  \gamma_h
  =
  (\operatorname{id},h)_{\#}
  \bigl(\lambda_{\Omega}/\mathrm{vol}(\Omega)\bigr)
  \in \mathcal{P}(\overline{\Omega}\times[-L,L]^n),
\]
which identifies a function with a probability measure supported on
its graph. A finite tokenization is then the empirical measure
\[
  {\bf M}_N(h)
  =
  \frac1N\sum_{j=1}^N
  \delta_{(x_j,h(x_j))}.
\]
Thus, refinement of a discretization becomes convergence of empirical
measures to $\gamma_h$ in a weak or Wasserstein topology. This point of
view allows us to study transformer operator learning without fixing a
grid, a sequence length, or an ordering of the tokens.

\paragraph{Transformer neural operators.}

Several recent works have adapted transformers to operator learning.
\citet{cao2021choose} interpreted attention in relation to Fourier and
Galerkin-type operator approximations and showed that attention can be
used as a nonlocal kernel mechanism for PDE surrogate modeling.
\citet{li2023transformer} proposed OFormer, an attention-based
operator-learning architecture built from self-attention,
cross-attention, and pointwise Multi-Layer Perceptrons (MLP), with few assumptions on the
sampling pattern of the input function or query locations. GNOT
\citep{hao2023GNOT} extends this direction to irregular meshes,
multiple heterogeneous input functions, and multiscale solution
structure through heterogeneous normalized attention and geometric
gating. PiT \citep{chen2024positional} takes a different approach:
rather than using content-based self-attention, it constructs
attention weights from spatial interrelations of sampling positions,
leading to a position-induced attention mechanism with a direct
discretization interpretation. Relative positional mechanisms such as
\textsc{FIRE} \citep{li2024fire} are not operator-learning models, but
they are relevant to the present paper because they emphasize that the
choice of positional representation can determine how attention
generalizes across sampling regimes.

A second line of work has moved from task-specific neural operators
toward continuous-query and foundation-model architectures for
scientific machine learning. CViT \citep{wang2025cvit} combines
vision-transformer encoders with query-wise decoding for continuous
operator learning. Universal Physics Transformers
\citep{alkin2024universal} propose a unified latent transformer
framework for Eulerian grids, meshes, and Lagrangian particles.
Poseidon \citep{poseidon2023} and Walrus \citep{walrus2024} develop
pretrained transformer-based models for PDEs and continuum dynamics,
emphasizing transfer across equations, domains, and resolutions.
Related architectural developments include domain-decomposition and
geometry-generalization mechanisms for PDE solvers
\citep{huang2026operator}. For low-regularity problems,
\citet{shih2025transformers} studied transformers as neural operators
for differential equations with finite regularity.

The mathematical theory of transformers has also developed rapidly.
Attention has been analyzed through smoothness, continuum, and
measure-theoretic viewpoints
\citep{castin2024smooth,geshkovski2025mathematical}. Universal
in-context approximation by transformers was established by
\citet{furuya2025transformers}, while
\citet{furuya2025transformerslenssupportpreservingmaps} studied
transformers as support-preserving maps between measures. Other
measure-based perspectives include measure-to-measure interpolation
and mean-field approximation results
\citep{geshkovski2024,biswal2024universal}. These works suggest that
attention is not only a finite-dimensional sequence operation, but
also a mechanism for constructing nonlinear maps on spaces of
measures.

\paragraph{Gap addressed in this paper.}

Existing transformer neural operators demonstrate strong empirical
performance and important architectural principles, but most of them
are tied to specific attention mechanisms, tokenization strategies, or
benchmark discretizations. Conversely, existing universal-approximation
results for neural operators often assume a fixed function-space
topology and do not explicitly track the passage from continuum
functions to graph tokens and back. The present paper addresses this
gap by giving a continuum universality theorem for standard
softmax-based transformers acting on graph measures. The result is an
expressivity theorem: it does not prove training convergence, sample
complexity bounds, or efficiency of the resulting architectures. It
also does not assert that softmax-free, linear, or position-only
attention mechanisms such as those used in OFormer, GNOT, or PiT have
the same universal approximation property without additional analysis.

\paragraph{Main idea.}

We view a transformer layer as a map
\[
  G(\mu)=\Gamma(\mu,\cdot)_{\#}\mu
\]
on probability measures, where $\Gamma$ is an in-context map depending
on the current measure $\mu$ and on the token being updated. A
\emph{function graph transformer} is a measure transformer that maps
graph measures to graph measures. In the simplest case, the first
coordinate is preserved and
\[
  \Gamma(\mu,(x,y))=(x,\widehat p(\mu,x)),
\]
so that
\[
  G(\gamma_h)=\gamma_g,
  \qquad
  g(x)=\widehat p(\gamma_h,x).
\]
This graph-preserving constraint is the continuum analogue of
requiring the transformer to output values attached to spatial query
locations. It makes precise the sense in which attention acts as a
learned nonlinear operator on functions rather than merely as a map
between finite vectors.

\paragraph{Contributions.}

Our contributions are as follows.

\begin{enumerate}
\item
\textbf{Graph measure formulation of operator learning.}
We represent functions by graph measures and identify finite token
sets with empirical approximations of these measures. Under regular
sampling assumptions, the empirical graph measures
${\bf M}_N(h)$ converge to $\gamma_h$ in Wasserstein topology. This
turns discretization refinement into a measure-convergence problem and
removes the dependence on a fixed grid or token ordering.

\item
\textbf{Function graph transformers.}
We define a class of graph-preserving measure transformers
$G(\mu)=\Gamma(\mu,\cdot)_{\#}\mu$ and characterize the special form
needed to map graph measures to graph measures. This yields a
continuum model of transformer operator learning in which spatial
coordinates, function values, and positional information are treated
as parts of a single measure-theoretic object.

\item
\textbf{Approximation by standard softmax transformers.}
Using the universal in-context approximation theorem of
\citet{furuya2025transformers}, we show that graph-preserving
measure maps can be approximated by finite compositions of standard
softmax self-attention layers and pointwise MLPs. The approximation is
stated in Wasserstein topology on compact classes of probability
measures and preserves the graph structure needed for operator
learning.

\item
\textbf{Universality for nonlinear operators between function spaces.}
Let $A$ be a possibly nonlinear operator satisfying the continuity
assumptions stated in Section~\ref{sec: Expressivity}. We construct
regularized encoders ${\bf M}_{N,\tau}$, measure-to-function decoders
${\bf F}_{\varepsilon}$, and regularized reconstruction maps
$S_{\eta}$ such that, for
\[
  h\in
  \overline{C(\Omega;[-L,L]^n)}^{\,H^{-s'}(\Omega)},
\]
one has
\[
  A(h)
  =
  \lim_{\tau\downarrow 0}
  \lim_{\eta\downarrow 0}
  \lim_{\varepsilon\downarrow 0}
  \lim_{N\to\infty}
  \lim_{m\to\infty}
  {\bf F}_{\varepsilon}\circ
  G_{\eta,m}^{\mathrm{tran}}\circ
  {\bf M}_{N,\tau}(h)
  \quad\text{in }H^{-s}(\Omega).
\]
For continuous inputs $h\in C(\overline{\Omega};[-L,L]^n)$ the
pre-regularization $R_\tau$ is not needed. Thus the same framework
covers both continuous functions and distributional inputs.

\item
\textbf{Operators between different domains.}
We also describe how to represent maps
\[
  A:C(\overline{\Omega};{\R}^n)\to C^1(\overline D;{\R}^m)
\]
when the input and output domains differ. This is done by introducing
query measures on the output domain and using product graph measures.
The construction separates the representation of the input function
from the choice of output query locations, clarifying the relation
between self-attention, cross-attention, and continuous evaluation of
operators.
\end{enumerate}

Overall, the paper provides a mathematical foundation for transformer
operator learning in which tokenization, positional information,
attention, regularization, and continuum approximation are handled in
one measure-theoretic framework.
}

\section{Introduction}
\label{sec:intro}


In this paper, we study operator learning with transformers from a foundational point of view and revisit universal approximation. Learning operators between function spaces has emerged as a central problem at the interface of analysis, scientific computing, and machine learning. A defining requirement is discretization invariance: the learned model should act consistently across different samplings of the underlying continuum. Transformer architectures, originally introduced for sequence modeling \citep{vaswani2017attention}, provide a natural and increasingly powerful framework for operator learning. By representing a function ${h}$ of $x$ through a collection of tokens $\{(x_j, {h}(x_j))\}_{j=1}^N$, transformers define mappings 
capable of modeling long-range interactions. The self-attention mechanism computes pairwise interactions via learned kernels, which can be interpreted as nonlinear integral operators in the continuum limit. This observation places transformers squarely within the landscape of operator approximation, with attention playing the role of a data-adaptive kernel.

\begin{figure}[t]
  \centering
  \resizebox{\textwidth}{!}{%
   \input{figures/diagramB}
  }
  \caption{Operators between function spaces are universally approximated by function graph transformers. Functions (blue) are identified with their graph measures (red), which are approximated by sums of Dirac measures (vertical black lines). We construct tokens from these empirical approximations, and a transformer outputs tokens corresponding to an approximation of the target graph measure. This diagram illustrates the simplest case of a target operator $A$ between continuous functions; with appropriate mollification and regularization, universal approximation can be extended to operators acting on (generalized) functions in negative-order Sobolev spaces. 
  }
  \label{fig:function-graph-commutative-diagram}
\end{figure}
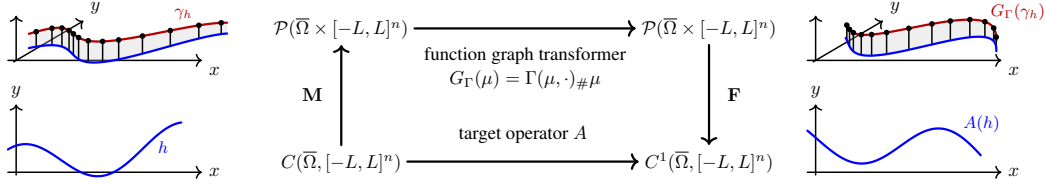

When ${h}:\overline \Omega\to [-L,L]^n$ is a continuous function, we choose a finite set of points $x_j\in \overline \Omega\subset \R^d$
and approximate the graph ${\bf G}({h})=\{(x,{h}(x)):\ x\in \overline \Omega\}$ by a finite sequence of tokens
$\xi_j=(x_j, {h}(x_j))\in \R^{d+n},$ $j=1,\dots,N$.
A mathematically fruitful perspective is to represent functions by empirical (graph) measures,
\begin{equation}\label{discrete measure}
\gamma^{(N)}_{h} =\frac{1}{N} \sum_{j=1}^N \delta_{\xi_j}= \frac{1}{N} \sum_{j=1}^N \delta_{(x_j, {h}(x_j))},
\end{equation}
that approximate a continuous measure $\gamma_{h}\in \mathcal{P}(\overline\Omega \times [-L,L]^n)$ supported on ${\bf G}({h})$, see Appendix \ref{app:tokens and measures}.
For a continuous, possibly nonlinear operator 
$A$, we show that the operator  between graph measures, $\gamma_{h}\to \gamma_{A({h})}$ can be approximated by composition of MLPs and traditional
softmax-based transformers, which in the single-head case are given by
\beq\label{eq xi prime}
    \xi_i'= \xi_i + W \sum_{j = 1}^N
    \frac{
        \exp\left( \frac{1}{\sqrt{k}} \langle Q \xi_i, K \xi_j \rangle \right)
    }{
        \sum_{\ell = 1}^N \exp \left( \frac{1}{\sqrt{k}} \langle Q \xi_i, K \xi_\ell \rangle \right) 
    } V \xi_j.
\eeq
where $Q, K, V, W$ are the query, key, value and head matrices, and $k$ is the dimension of the key/query vectors $Q \xi_i$ and $K \xi_j$.
By denoting the empirical graph measure by $\mu=\gamma^{(N)}_{h}$, we can write \eqref{eq xi prime} 
using a map that takes in a measure $\mu$ and the token $\xi_i$, 
\beq\label{eq xi prime2}
    \xi_i' = \Gamma_\theta (\mu, \xi_i)
\eeq
with parameters $\theta = (Q, K, V,W)$, where  
\beq\label{eq Gamma with xi}
    \Gamma_\theta(\mu, \xi) = \xi + W \int_{\R^{d+n}} 
    \frac{
        \exp\left( \frac{1}{\sqrt{k}} \langle Q \xi, K \xi' \rangle \right) 
    }{
        \int_{\R^{d+n}} \exp \left( \frac{1}{\sqrt{k}} \langle Q \xi, K \zeta \rangle \right) d\mu(\zeta)
    } V \xi' d\mu(\xi').
\eeq
Using this notation,
a transformer layer defines a map from discrete measures to discrete measures,
\beq\label{eq measure to measure}
   G_\Gamma: \mu_N = \frac{1}{N} \sum_{j=1}^N \delta_{\xi_j}
    \mapsto
    \mu_N' =  \frac{1}{N} \sum_{j=1}^N \delta_{(x'_j,y'_j)},
\quad (x'_j,y'_j)=\Gamma_\theta(\mu,(x_j,y_j)),
\eeq
Using the notation of the push-forward $F_\#$ of a map $F$ used in measure theory, we can write
\eqref{eq measure to measure} as
\beq\label{eq measure to measure2}
    G_\Gamma(\mu)=\Gamma_\theta(\mu, \cdot)_\# \mu,
\eeq
that is $G_\Gamma(\mu)$ is obtained by pushing forward the measure $\mu$
using a context function $\Gamma_\theta(\mu, \cdot)$ that depends on $\mu$,
see \eqref{def: pushforward} in Appendix \ref{sec:Basic notions}.
The formula \eqref{eq measure to measure2} defines a map on probability measures,
$G_\Gamma:\Pp({\R^{D}}) \to \Pp({\R^{D}})$ that treats measures $\mu$ and their empirical approximations $\mu_N$ in a uniform way. 
This unified treatment makes it possible to use measure theory to study limits of long token sequences by taking $N\to\infty$. 
In this paper we will consider measure-theoretic
transformers $G_\Gamma$ with a context function 
$\Gamma = (\Gamma^{(x)}, \Gamma^{(y)})$
whose $x$-coordinate function $\Gamma^{(x)}$ 
depends only on the variable $x$ and defines a one-to-one map $x \mapsto \Gamma^{(x)}(x)$.
This property will ensure that graph measures are mapped to graph measures.
We also generalize these considerations to the case when $A:C(\overline \Omega;[-L,L]^n)\to C(\overline D;[-L,L]^m)$ is a map between different domains $\Omega\subset \R^d$ and $D\subset \R^{d'}$. Summarizing, we view transformer layers as maps acting on measures and operators as maps from graph measures to graph measures.
Under this formulation, self-attention induces a nonlinear pushforward of measures through kernels that depend on both spatial location and function values.
\commented{
 \Ivan{I think the below display is too heavy for the introduction; how about introducing general tokens $\xi$ (interpreted as $(x, u(x)$) to lighten notation? (see my comments below) It may also save some space later and improve readability to introduce a notation for softmax, and even for the standard attention layer itself, e.g. for linear maps $Q, K$ we could define general softmax weights
$$
a_{Q,K}[\mu](\xi,\xi')
=
\frac{
\exp(\langle Q\xi,K\xi'\rangle/\sqrt{k})
}{
\int_{\mathcal{X}}
\exp(\langle Q\xi,K\tilde\xi\rangle/\sqrt{k})\,d\mu(\tilde\xi)}
$$
and then for a value map $V$ define one attention head
$$
\operatorname{Att}_{Q,K,V}[\mu](\xi)
=
\int_{\mathcal X}
a_{Q,K}[\mu](\xi,\xi') V\xi'\,d\mu(\xi').
$$
So that if $\mu=\mu_N=N^{-1}\sum_{j=1}^N\delta_{\xi_j}$, then
$$
\operatorname{Att}_{Q,K,V}[\mu_N](\xi_i)
=
\sum_{j=1}^N
\frac{
\exp(\langle Q\xi_i,K\xi_j\rangle/\sqrt{k})
}{
\sum_{\ell=1}^N
\exp(\langle Q\xi_i,K\xi_\ell\rangle/\sqrt{k})
}
V\xi_j.
$$
}
\beq  \label{pre trad-transformer}
& &\hspace{-10mm} \Gamma(\mu,(x_j,y_j)) = \bigg(x_j\ ,\ W_xx_j + W_yy_j
\\ \nonumber
& &\hspace{-10mm} + \frac 1N \sum_{h=1}^H 
   \sum_{\ell=1}^N  \frac{
        \exp\Big(
            \frac{1}{\sqrt{k}}
            \dotp{(Q^h_xx_j,Q^h_yy_j)}{(K^h_xx_\ell,K^h_yy_\ell)} \Big)  }
       { \tfrac 1N \sum_{i=1}^N 
        \exp \Big(
            \frac{1}{\sqrt{k}}
            \dotp{(Q^h_xx_j,Q^h_yy_j)}{(K^h_xx_i,K^h_yy_i)}\Big) 
    } (W^h_xV^h_xx_\ell+W^h_yV^h_yy_\ell).
\eeq
The crucial property of our construction is that these maps extend to maps $F_\Gamma:\Pp({\R^{d+n}}) \times {\R^{d+m}}\to {\R^{d+m}}$ between continuous (i.e., non-discrete) measures that makes it possible to analyze the limits $N\to \infty$ using measure theory. We also generalize these considerations to the case when  $A:C(\overline \Omega;[-L,L]^n)\to C(\overline D;[-L,L]^m)$ is a map between different domains $\Omega\subset \R^d$ and $D\subset \R^{d'}$. Summarizing, we view transformer layers as maps acting on measures.
Under this formulation, self-attention induces a nonlinear pushforward of measures through kernels that depend on both spatial location and function values.
This measure-theoretic viewpoint clarifies the role of positional encoding: It breaks permutation invariance in a controlled way, allowing the model to represent operators that depend on the geometry of the domain.
}

\paragraph{Prior work.}

Early work on operator learning with transformers showed that attention can facilitate discretization-robust PDE surrogate modeling. OFormer \citep{li2023transformer} combines self-attention, cross-attention, and pointwise networks to learn mappings from sampled input functions to queried outputs, while 
\citet{cao2021choose} interpret attention in relation to operator approximation and Galerkin-type projections. 
GNOT \citep{hao2023GNOT} extends this line to irregular meshes, heterogeneous input fields, and multiscale structure through normalized attention and geometric gating. PiT \citep{chen2024positional} builds attention weights from spatial relations among sample locations, giving a position-induced kernel with a direct discretization interpretation. Other work demonstrated that the choice of attention kernel and positional encoding affects generalization across sequence lengths and sampling regimes \citep{choromanski2021rethinking,katharopoulos2020transformers,li2024fire}.

A second line of work has moved from task-specific neural operators toward foundation-model architectures for scientific machine learning. Vision transformers and patch-based representations \citep{dosovitskiy2021image} motivate models such as CViT \citep{wang2025cvit}, which combines transformer encoders with query-wise decoding to evaluate outputs at arbitrary coordinates. Universal Physics Transformers \citep{alkin2024universal} use a latent transformer framework for Eulerian grids, meshes, and Lagrangian particles, while Poseidon \citep{poseidon2023} and Walrus \citep{walrus2024} are pretrained transformer-based models for PDEs and continuum dynamics, emphasizing transfer across equations, domains, resolutions, and physical regimes.

The mathematical theory of transformers has also developed rapidly. Attention has been studied through smoothness, continuum, and measure-theoretic viewpoints \citep{shih2025transformers,castin2024smooth,geshkovski2025mathematical}. Universal in-context approximation by transformers was established by \citet{furuya2025transformers}, while \citet{furuya2025transformerslenssupportpreservingmaps} studied the closure of transformers as support-preserving maps between measures. Related perspectives include measure-to-measure interpolation and mean-field approximation results \citep{geshkovski2024,biswal2024universal,mei2019meanfield,yang2020scaling}. The closest prior work is \citet{calvello2025continuum}, who formulate continuum attention directly on function spaces and prove universality for operators on \(C^s\) and nonnegative-order Sobolev spaces. Our approach is complementary: functions are represented by graph measures and transformers by graph-preserving measure maps, yielding universality results that extend to negative-order Sobolev spaces.

Despite these advances, the theory of transformers for operator learning remains incomplete. 
New perspectives are required to illuminate this problem given the success of transformers on hard operator learning tasks.
Fundamental open questions concern approximation under low regularity, stability under input perturbations, and the role of depth in capturing compositional structure. 

\subsection*{Our contributions}

We develop a mathematical framework for transformers as operator learners. We obtain a universal approximation result through the introduction of function graph transformers imposing a natural structure on the in-context mappings.

\textbf{1. Graph measure formulation of operator learning.}
We represent functions by graph measures and identify finite token sets with empirical approximations of these measures. We show that continuous nonlinear operators, $A : C(\overline{\Omega};{[-L,L]}^n) \to C^1(\overline D;[-L,L]^m)$, can be represented by maps from graph measures to graph measures.

\textbf{2. Function graph transformers.}
We define graph-preserving measure transformers as a subclass of measure transformers with particular 
in-context maps. When graph measures are approximated by discrete measures that are equivalent to a sequence of $N$ tokens,
the graph-preserving measure transformers define maps between sequences of tokens consistent with the conventional presentation of transformers. We prove rigorous continuity results as $N\to \infty$,
corresponding to the limit of long token sequences.

\textbf{3. Universal approximation.}
We show that graph-preserving measure transformers can be approximated by compositions of self-attention layers and pointwise MLPs.

\textbf{4. Generalized functions and regularization.}
We generalize the universal approximation to operators between generalized function spaces through a regularization procedure that can be realized via a composition with a linear neural operator.


\commented{Some notes from Ivan:
To me the central aim of this paper remains a little unclear. Is it a paper about a new object, namely function graph transformers, which vaguely resemble transformers and have interesting approximation-theoretic properties? Or is it a paper about universal approximation of operators (in the modern discretization-invariant sense), which shows that a good way to conceptualize the problem is from a function graph transformer viewpoint?
I'd say currently it's a mix, but that for the ML crowd at NeurIPS the second version is a better fit and currently it's not coming through. I think it's easy to repair though:
    1. I'd suggest making the usual measure-theoretic interpretation of transformers very explicit very early on. (See below for a suggestion.) Most readers, even theory-minded, will not be familiar with it and will probably have a hard time following
    2. I'd suggest identifying the gap in the literature explicltly and making the goal of the paper very clear: there is little in the way of theory of neural operators with transformers, yet they are used to solve operator learning problems all over the place. Outstanding questions are: what is the right way to think about it? What are the good mathematical frameworks to prove results? What are desiderata (e.g., formulaic wrappers for boilerplate UATs like in that Karniadakis paper really don't say anything interesting about transformers or regularity; the Stuart paper is much better but it starts from a continuous point of view; can we start directly from tokens? Or something of the sort?)
We could at this point say that this perspective has given very elegant and strong approximation results for transformers. Now we want to show that a refinement of a related object, can yield similarly elegant theory for operator learning, and explain that the important point there for the community is to be stable with respect to discretizations of the domain, and have suitable convergence under discretization refinements. Argue also that empirically transformers seem to have such properties, but there is little theory. Then perhaps say that what we show is that an elegant or natural way to approach this is by encoding the domain in the tokens (this has been know perhaps), and that the measure-theoretic encoding of this are graph measures.
}

\section{Graph measures and maps between functions and measures}
\label{sec:graph measures}

Let $\Omega \subset \R^d$ be an open bounded set with $C^\infty$ boundary, or a cube $[0,1]^d$, $\mathcal B$ the family of Borel sets in $\Omega$ and $\lambda_{\Omega} : \mathcal B\to [0,\infty)$ be the Borel measure on $\Omega$ that coincides with the Lebesgue measure for all open subsets of $\Omega$. (For basic notions related to this paper, see Appendix~\ref{sec:Basic notions}.) 
Measures, particularly finite measures supported on finite sets, are closely related to sequences of tokens via sums of Dirac measures, see Appendix \ref{app:tokens and measures}. We argue that non-discrete measures can be considered as limits of sequences of $N$ tokens as $N \to \infty$. We use point-sample tokens here for simplicity; Appendix~\ref{section: Patching} describes a patch-based tokenization compatible with the same graph-measure viewpoint.

Let $L^\infty_{\mathcal B}(\Omega;{[-L,L]}^n)$ denote the essentially bounded Borel measurable functions.
For $h: \overline\Omega \to [-L,L]^n$, we define the ``graph measure'',
\begin{align} \label{push forward formula}
  \gamma_h = (\operatorname{id},h)_{\#}\overline \lambda_{\Omega},\quad \hbox{i.e.,} \quad \gamma_h(A)=
  \overline \lambda_{\Omega}(\{x\in\Omega:\ (x,h(x))\in A\}),
\end{align}
where $\overline \lambda_{\Omega}=\tfrac 1{\text{vol}(\Omega)} \lambda_{\Omega}$ and $A\subset \Omega\times {[-L,L]^n}$ is a  Borel set. Thus $\gamma_h$ is a measure supported on the graph of $h$ that satisfies 
\begin{align}\label{integration formula}
  \int_{\Omega \times [-L,L]^n} \phi(x,y) \, \mathrm{d}\gamma_h (x,y)={\frac 1{\text{vol}(\Omega)}} \int_{\Omega} \phi(x,h(x)) \, \mathrm{d}\lambda_{\Omega}(x) \quad\hbox{for all $\phi\in C(\overline\Omega\times {[-L,L]^n})$}.
\end{align}

\commented{
\footnote{
In fact, the operator ${\bf F}$ can be well-defined for $\gamma = (\mathrm{id}, h)_\sharp \lambda_{\Omega}$ with $h \in L^1_{\mathrm{loc}}$: We observe that \eqref{integration formula}  is valid when $h\in L^1_{loc}(\Omega)$, that is, $h:\Omega\to \R^n$ is locally integrable Lebesgue measurable function, that is, $(\mathcal L,\mathcal B_{\R^n})$-measurable where $\mathcal L$ is the $\sigma$-algebra of the Lebesgue measurable sets and $\mathcal B_{\R^n}$ is the  $\sigma$-algebra of the Borel sets of $\R^n$ and  $\phi$ is a bounded Borel-measurable function of $\Omega\times \R^n$, so for example when $h\in L^1(\Omega)$ and $\phi :\ \Omega \times \R^n \to \R$ is a bounded continuous function, see Theorem~3.6.1 in [Bogachev, Vladimir I., Measure Theory, Berlin: Springer Verlag, 2007]. By definition~\eqref{push forward formula}, we see that a function $h \in L^1(\Omega)$ defines a measure $\gamma_h$, and we can define a function ${\bf M} : h \to \gamma_h$. Similarly, a graph measure defines a function by
\begin{align} \label{f at Lebesgue points1}
  h(x) &:= \lim_{r \to 0+} \frac{1}{\omega_d r^d}
  \int_{(B(x,r) \cap {\overline\Omega})\times \R^n} y \,d\gamma_h(x,y)
  \\
  &= \lim_{r\to 0+} \frac{1}{\omega_d r^d}
  \int_{B(x,r) \cap {\overline\Omega}} h(x) \, dx ,
\end{align}
where $\omega_n$ is the volume of the unit ball in $\R^n$, as formula \eqref{f at Lebesgue points1} is valid by formula \eqref{push forward formula} (with $\phi(x,y) = y$) at all Lebesgue points $x \in {\overline\Omega}$ of $h$. This defines a function ${\bf F} :\ \gamma_h \to h$.)
}}
\commented{In zoom we discussed about the sums of delta distributions as a probability distribution: We can think $\mu=\frac 1N \sum_{j=1}^N\delta_{x_j}$ as an oracle the gives out random tokens, in the same way as machine learning algorithms give random outputs with the identical prompts. So, $\mu$ is non-deterministic oracle.
Our measure-theoretic transformers are function that map one non-deterministic oracle to another, more useful non-deterministic oracle.}



\begin{definition}
We define ${\bf M} : L^\infty_{\mathcal B}(\Omega;{[-L,L]}^n) \to \Pp(\Omega \times [-L,L]^n)$ to be the map from a function to its graph measure. That is,
$
  {\bf M}(h) := \gamma_h = (\operatorname{id},h)_{\#} \bar \lambda_{\Omega} .
$
\end{definition}

The left inverse for the map ${\bf M}$ is given by
\beq\label{Lebesgue points}
  {\bf F} :\ \gamma_h \mapsto {\bf F}(\gamma_h)(x)
  :=\lim_{r \to 0+} \frac{{ {\text{vol}(\Omega)}}}{\omega_d r^d}
  \int_{(B(x,r) \cap \Omega)\times [-L,L]^n} y' \, d\gamma_h(x',y') ,\quad
  x \in \Omega,
\eeq
where $B(x,r)$ is the ball centered at $x$ with radius $r$, and $\omega_d$ is the volume of the unit ball in $\R^d$. 
For any $h \in L^\infty_{\mathcal B}(\Omega;{[-L,L]}^n)$ formula \eqref{Lebesgue points} is valid in a Borel set whose complement has measure zero \citep{Bogachev}. 
We note that operators ${\bf M}$ and ${\bf F}$  are defined for graph measures while we also pay attention to sums of delta measures. Because of this, we define the regularized versions of ${\bf M}$ and ${\bf F}$.

\begin{definition} \label{def: F eps map}
For $\e > 0$, we define the map from general measures to continuous functions,
\begin{align} \label{f mollified 1B}
  ({\bf F}_\varepsilon \gamma)(x) &=\frac { {\mathrm{vol}(\Omega)}}{\varepsilon^d}
  \int_{{\overline\Omega} \times [-L,L]^n} \rho\left(\frac{x'-x}{\e}\right)
  y' \, d\gamma(x',y').
\end{align}
\end{definition}

We use a standard mollifier $\rho\in C_c^\infty(\R^d)$, $\int_{\R^d}\rho(x')dx'=1$, $\supp(\rho)\subset B(0,1)$, $\rho(-x)=\rho(x)$ and set $\rho_\tau(x)=\tau^{-d}\rho(x/\tau)$ and $\rho_\tau \ast h(x) = \int_{\R^d} \rho_\tau(x-x') h(x') \, dx'$.

For  functions $h \in  L^\infty_{\mathcal B}(\Omega;{[-L,L]}^n)$ and their graph measures $\gamma_h$ we have $({\bf F}_\varepsilon \gamma_h)(x)= \rho_\varepsilon \ast \tilde h(x)$, where $\tilde h$ is the zero-extension of $h$ to $\R^d.$
For $h\in L^p(\Omega)$, for  all  $1 \le p <\infty$, it holds that
\begin{align} \label{f mollified 2C}
    h &= \lim_{\varepsilon \to 0+} ({\bf F}_\varepsilon \gamma_h) \quad \hbox{in } L^p(\Omega).
\end{align}
For $h \in H^{-s}(\Omega)$, $s\in\mathbb R$ we introduce a regularized version of $h$, 
\begin{equation} \label{def R tau}
   R_\tau h = \rho_\tau \ast (h \cdot (\rho_\tau \ast {\bf 1}_{\Omega_\tau})) ,\quad
   \tau > 0 ,
\end{equation} 
with $\Omega_\tau = \{x \in \Omega\ :\ \hbox{dist}(x,\p \Omega) > 4 \tau\}$.
We note that, here, $\rho_\tau \ast {\bf 1}_{\Omega_\tau}$ vanishes near the boundary and convolution with $\rho_\tau$ makes a function continuous (in fact, $C^\infty$). Below, we will consider functions $h \in H^{-s}(\Omega)$ with $s>0$, $s - \tfrac 12 \not\in \mathbb Z$. Such functions can be discontinuous, finite measures, or even Dirac delta distributions when $-s<-\tfrac d2$. We show in Appendix~\ref{app:prop:main convolutions A}  that for Sobolev functions $h \in H^{-s}(\Omega)$ it holds that $R_\tau h \to h$ in $ H^{-s}(\Omega)$ as $\tau\to 0$.
  
\section{Function graph transformers}
\label{sec:function graph}

Below, let $\pi_1 : {\overline\Omega} \times \R^n \to {\overline\Omega}$ be the projection $\pi_1(x,y)=x$ and $\pi_2 : {\overline\Omega} \times \R^n \to \R^n$ be the projection $\pi_2(x,y)=y$. 
{\color{black} We consider the set of probability measures ${{{\mathcal P}}}(\overline\Omega\times [-L,L]^n)$ as a metric
space with 1-Wasserstein distance ${{\color{black}W_1}}$.}

Below, we will consider measure-theoretic
 transformers
 \commented{\footnote{By \cite{furuya2025transformerslenssupportpreservingmaps}, a map between $F$ mapping measures to measures is a measure-theoretic transformer if  it is support
    preserving and the regular part of derivative is uniformly continuous.}} $G_\Gamma:{{{\mathcal P}}}(\overline\Omega\times [-L,L]^n)\to {{{\mathcal P}}}(\overline\Omega\times [-L,L]^n)$. 
These are continuous  maps in the 1-Wasserstein topology, see \cite{furuya2025transformerslenssupportpreservingmaps},
$$
G_{\Gamma}(\mu)={\Gamma}(\mu,\cdot)_\#\mu
$$
where ${}_\#$ denotes push forward, see \eqref{def: pushforward} in Appendix \ref{sec:Basic notions}, and
\beq 
{\Gamma}:{{{\mathcal P}}}(\overline\Omega\times [-L,L]^n)\times 
({\overline\Omega}\times [-L,L]^n)\to {\overline\Omega}\times [-L,L]^n
\eeq
is a continuous map called the \emph{in-context map}. We will study 
 $G_\Gamma$ with in-context maps $\Gamma$ having the property that for all graph measures $\gamma_h$ of continuous functions
$h: \overline\Omega \to [-L,L]^n$ the image $G_\Gamma(\gamma_h)$ is a graph measure $\gamma_{g}$ 
of some  continuous function $g=J_\Gamma(h): \overline\Omega \to [-L,L]^n$.
Next, we show that  such maps can be represented using a structured form discussed in the following lemma that is proven in Appendix~\ref{app:lem: function graph transformers architecture}. 
%
%
\begin{lemma}\label{lem: function graph transformers architecture}
Assume that the in-context function ${\Gamma}:{{{\mathcal P}}}(\overline\Omega\times [-L,L]^n)\times ({\overline\Omega}\times [-L,L]^n)\to {\overline\Omega}\times [-L,L]^n$ is continuous and  has the property that it preserves graph measures:
\begin{align*}
    (P)& \hbox{ For all continuous $h:\overline \Omega\to [-L,L]^n$ there
is a continuous function $g=J_\Gamma(h):\overline \Omega\to [-L,L]^n$}
\\
& \hbox{such that ${\Gamma}(\gamma_h,\cdot)_\#\gamma_h=\gamma_g$}
\end{align*}
and moreover, that
the map  $h\to J_\Gamma(h)$
is a Lipschitz  map
${{{\mathcal P}}}(\overline\Omega\times [-L,L]^n)\to
C(\overline \Omega;{[-L,L]}^n)$, 
with
$\hbox{Lip}(J_\Gamma(h))$ uniformly bounded for  $h\in C(\overline \Omega;{[-L,L]}^n)$.

%

%
Then, there exists an in-context map  $\widehat {\Gamma}:{{{\mathcal P}}}(\overline\Omega\times [-L,L]^n)\times ({\overline\Omega}\times [-L,L]^n)\to {\overline\Omega}\times [-L,L]^n$ that is equivalent to $\Gamma$ in the sense that it satisfies
\beq
{\Gamma}(\gamma_h,\cdot)_\#\gamma_h=\widehat {\Gamma}(\gamma_h,\cdot)_\#\gamma_h
\eeq
for all continuous functions  $h:{\overline\Omega}\to [-L,L]^n$; the in-context map $\widehat {\Gamma}$ has the coordinate functions
\beq\label{widehat f map}
& &\widehat {\Gamma}(\mu,(x,y))=(\widehat {\Gamma}^{(x)}(\mu,(x,y))\,,\,\widehat {\Gamma}^{(y)}(\mu,(x,y)))\in {\overline\Omega}\times [-L,L]^n,\hbox{ where }\\
\label{widehat f map2}
 & &\widehat {\Gamma}^{(x)}(\mu,(x,y))=x\quad   \hbox{and}\quad  
\widehat {\Gamma}^{(y)}(\mu,(x,y))=\widehat p(\mu,x),
\eeq
and $\widehat p:{{{\mathcal P}}}(\overline\Omega\times [-L,L]^n) \times  {\overline\Omega}\to [-L,L]^n$; i.e., the $x$-coordinate $\widehat {\Gamma}^{(x)}$ is the identity map in the variable $x$ and the $y$-coordinate $\widehat {\Gamma}^{(y)}(\mu,(x,y))$ is independent of the variable $y$. 
\end{lemma}

\medskip

Motivated by the above result, Lemma \ref{lem: function graph transformers architecture part 2} in Appendix~\ref{app:lem: function graph transformers architecture} and Corollary \ref{cor: 1self} 
concerning the continuity of maps, we define function graph transformers based on self-attention.

\begin{definition}\label{def self-attention}
A map $G_{\widehat{\Gamma}} : {{{\mathcal P}}}(\overline\Omega\times [-L,L]^n) \to {{{\mathcal P}}}(\overline\Omega\times [-L,L]^n)$ is a \emph{function graph transformer} if 
$$
G_{\widehat{\Gamma}}(\mu) = \widehat{\Gamma}(\mu,\cdot)_\#\mu,
$$
where $\widehat {\Gamma}(\mu,(x,y))$ is defined by formula \eqref{widehat f map} with
\beq\label{widehat f map2 mod}
  \widehat {\Gamma}^{(x)}(\mu,(x,y))=\Phi(x)\quad\hbox{and}\quad  
  \widehat {\Gamma}^{(y)}(\mu,(x,y))=\widehat p(\mu,x).
\eeq
Here, $\widehat p : {{{\mathcal P}}}(\overline\Omega\times [-L,L]^n)\times  {\overline\Omega}\to [-L,L]^n$ is uniformly continuous in $\mu$ and  uniformly Lipschitz
bounded in $(x,y)$, and the map $\Phi : \overline \Omega\to \overline \Omega$ is a volume preserving $C^1$-diffeomorphism, i.e., the Jacobian determinant satisfies $|\det(D\Phi)|=1$. In the case when $\Phi(x) = x$, $\widehat{\Gamma}$ coincides with \eqref{widehat f map2}.
\end{definition}

\commented{
In \eqref{widehat f map2 mod} we have included the additional map $\Phi$ making a deformation in the $x$-variable that is similar to functions used in positional
encoding map. We note several positional encoding maps, such as block diagonal rotation matrices, are volume preserving{\color{black}  [Note: We could add references to this, see that the two Positional encoding files in this overleaf folder]. [Could we find a Lemma identifying / replacing the function graph transformer with / by one with a particular positional encoding? I am thinking of absolute positional encoding adapted to functional positional encoding. It seems that introducing $\Phi$ would correspond with that.]}

\Takashi{In my opinion, $\Phi$ in \eqref{widehat f map2 mod} is not exactly same with what operator transformers used, but it could be loosely interpreted as PE (in term of "modifying the coordinate"). If you like, please use the following. }
}


The map $\Phi$ in \eqref{widehat f map2 mod} preserves all information of the coordinates, $x$, throughout the network, and can be interpreted as a coordinate feature embedding in a similar spirit to the feature mappings used in \citep[Sec.~4.1.1]{cao2021choose} and \citep[Fig.~1]{li2023transformer}, as well as the absolute positional encoding in \cite[Eq.~(1)]{wang2025cvit}.

\commented{
\footnote{\Matti{Question for Maarten: Should we move this footnote to Appendix or is it fine here?} Note that when $\widehat {\Gamma}^{(x)}(\mu,(x,y))=\Phi(x)$  where  $\Phi$ is volume preserving has the property that  the graph measures are mapped to graph measures. In particular, then $(\pi_1)_*\widehat {\Gamma}(\gamma_h,(\cdot,\cdot))_\#\gamma_h=|\det(D\Phi\circ \Phi^{-1})|^{-1}\cdot \overline\lambda_\Omega$  
We also note that if $\widehat {\Gamma}_1(\mu,(x,y))=
(x,\widehat p(\mu,x))$ and $\widehat {\Gamma}_2(\mu,(x,y))=
(\Phi(x),\widehat p(\mu,\Phi(x)))$ and  $\widehat {\Gamma}_1(\gamma_h,\cdot)_\#\gamma_h=\gamma_g$, then $\widehat {\Gamma}_2(\gamma_h,\cdot)_\#\gamma_h=\gamma_{g\circ \Phi^{-1}}$.  It seems that most of the positional encoding maps have been volume preserving as their Jacobians are the constant 1. Our considerations motivate why these maps have this property.
}}

Below, for the sake of notational simplicity, we assume that $\Phi(x) = x$. However, results can be easily adapted for a non-identity map. In essence, a function graph transformer never generates implicitly multi-valued functions and, hence, preserves the function space aspect of operator learning.
Function graph transformers define maps between functions as follows, 
$$
{\bf F} \circ G_{\widehat{\Gamma}} \circ {\bf M} : \{h \in C(\overline \Omega):\ \|h\|_\infty\le L\}\to  C^1(\overline \Omega).
$$
Later, in Section~\ref{sec: Expressivity}, we study maps where ${\bf F}$ and ${\bf M}$ are replaced by the regularized maps ${\bf F}_\e$ and ${\bf M}_N$, in which case the transformer can map discrete measures to discrete measures.

\paragraph{Approximation of graph measures and tokenization.}
\commented{Morally speaking, all these results should hold for piecewise continuous maps, too.}
Next, our aim is to consider approximation of graph measures by finite sums of Dirac measures yielding a discretization or tokenization; see Appendix~\ref{app:tokens and measures} for a more detailed discussion.

\begin{definition}
\label{def:regularly_distributed}
Let $\Omega \subset \R^d$ be a bounded open set. The points  ${\bar x}_j \in \Omega \subset \R^d$, $j=1,2,\dots$ are \emph{regularly distributed}
if for all continuous functions $h \in C(\overline\Omega)$ it holds that 
$\lim_{N \to \infty} \frac 1N \sum_{j=1}^N h({\bar x}_j) = \frac{1}{\mathrm{vol}(\Omega)} \int_\Omega h(x) \, dx$. 
\end{definition}

Varadarajan's theorem (\citet{Varadarajan1958} and Theorem~11.4.1 in \citet{Dudley}), 
implies that if $X_1,X_2,\dots$ are 
independent, uniformly distributed points in $\Omega$,
then they are regularly distributed.

We denote for $h \in C(\overline \Omega;\R^n)$ and $g \in H^{-s}(\Omega;\R^n)$
\beq \label{discrete M1}
  {\bf M}_N(h) = \sum_{j=1}^N \frac 1N \delta_{({\bar x}_j,h({\bar x}_j))}
  = (\operatorname{id},h)_\# \Bigg(\sum_{j=1}^N \frac 1N \delta_{{\bar x}_j} \Bigg) ,\quad 
  {\bf M}_{N,\tau}(g) = {\bf M}_N(R_\tau g) .
\eeq
Then, by the assumption that ${\bar x}_j \in \Omega \subset \R^d$, $j=1,2,\dots$ are regularly distributed points,
we have for $h \in C(\overline \Omega)$ (see Appendix \ref{app: proof of M2 and M3}),
\beq \label{discrete M2}
  \lim_{N \to \infty} {\bf M}_N(h) = \gamma_h\
  \hbox{in the ${{\color{black}W_1}}$-metric.}
\eeq
%
%
Moreover, for $h \in H^{-s}(\Omega)$ and fixed $\tau>0$ and $h_\tau=R_\tau h$ we have (see Appendix \ref{app: proof of M2 and M3}),
\beq \label{discrete M3}
  \lim_{N \to \infty} {\bf M}_{N,\tau}(h) = \gamma_{h_\tau}\
  \hbox{in the ${{\color{black}W_1}}$-metric.}
\eeq

In the next section, we design function graph transformers to act on graph measures and their empirical approximations in a discretization invariant manner. We illustrate these results in Appendix~\ref{App:experiments}: a function-graph transformer trained at a fixed input resolution is evaluated on increasingly fine input point clouds, and the error decreases as $N$ grows, consistent with convergence of empirical graph measures toward the continuum graph measure; see Table~\ref{tab:fno-teacher-resolution-new}.

\section{Universal approximation of function graph transformers}
\label{sec:5}

In this section, we discuss approximating function graph transformers by traditional transformers. General results on approximation of function graph transformers are considered in Appendix~\ref{App: genreal approximation of function graph transformers}.
We employ a measure-theoretic framework \citep{furuya2025transformers, castin2024smooth} and the following composition. Below, we denote coordinates in the space $\R^{d+d'}$ by $z = (x,y) \in \R^{d} \times \R^{d'}$. 
Let  $\pi_d:\mathbb{R}^d \times \mathbb{R}^{d'}\to \R^d$ be the projection  $\pi_d(x,y) := x$ onto the first $d$ coordinates.
An in-context map as it appears in a single-layer ``measure-theoretic'' transformer, based on multi-head self-attention,\footnote{
Attention can also be written as an expectation, as was done in \citet[Definition~5 and Eq~(14)]{calvello2025continuum}:
    \begin{equation*} \label{eq:Gamma-theta}
        \mathrm{Att}(z) = 
        \mathbb E_{z' \sim \rho_{\theta,\mu,z}}
        \left[
        V z'
        \right],\quad
        d\rho_{\theta,\mu,z}(z')
        =
        \frac{
        \exp\!\left(
        \frac{1}{\sqrt{k}}
        \dotp{Q z}{K z'}
        \right)
        }{
        \int_{\mathbb R^{d+d'}}
        \exp\!\left(
        \frac{1}{\sqrt{k}}
        \dotp{Q z}{K z''}
        \right)
        \,d\mu(z'')
        }
        \,d\mu(z').
    \end{equation*}
    
} is of the form 
$\Gamma_{\theta} : \Pp({\R^{d+d'}}) \times {\R^{d+d'}}\to {\R^{d+d''}}$,
\begin{gather} \label{Gamma theta}
    \Gamma_{\theta}(\mu,z) := 
\bigg(\pi_d(z)\ ,\ W^L z +
    \sum_{h=1}^H W^h
    \mathrm{Att}^h(z)\bigg), \text{ where} \\
    \mathrm{Att}^h(z) := \int_{\R^{d+d'}} \frac{
        \exp\Big(
            \frac{1}{\sqrt{k}}
            \dotp{Q^h z}{K^h z'}
        \Big)
    }{
        \int_{\R^{d+d'}}
        \exp\Big(
            \frac{1}{\sqrt{k}}
            \dotp{Q^h z}{K^h z''}
        \Big)
        d \mu(z'')
    } V^h z'\, d \mu(z').
\end{gather}




Here, $W^L : {\R^{d+d'}} \to {\R^{d''}}$ is a weight matrix corresponding to the linear part of the in-context map (introduced to make skip connections possible),
$K^h$ and $Q^h$ are the multi-head key and query matrices in $\mathbb{R}^{k \times (d+d')}$, $V^h$ are the multi-head value matrices in $\mathbb{R}^{d_\text{head} \times (d+d')}$, and $W^h$ are the multi-head weight matrices in $\mathbb{R}^{d'' \times d_\text{head}}$, respectively.
We have written $\theta = \{ Q^h, K^h, V^h, W^L, W^h \}_{h=1}^H$. By abuse of notation, $z \to \Gamma_{\theta}(\mu)(z) = \Gamma_{\theta}(\mu,z)$ defines a map $\mathbb{R}^{d+d'} \to \mathbb{R}^{d+d''}$.  
For two in-context maps, $\Gamma_j : {\mathcal P}(\R^{d_j}) \times \R^{d_j} \to \R^{d_{j+1}}$, $j=1,2$, the composition $\Gamma_2 \diamond \Gamma_1:
 {\mathcal P}(\R^{d_1}) \times \R^{d_1} \to \R^{d_{3}}$ is defined as 
\begin{align}\label{diamond composition}
    (\mu,z) \mapsto (\Gamma_2 \diamond \Gamma_1)(\mu,z) := \Gamma_2( \nu, \Gamma_1(\mu,z)),\quad \nu := \Gamma_1(\mu)_\sharp \mu.
\end{align}
Moreover, 
$F_{\xi} : \mathbb{R}^d \times \mathbb{R}^{d'} \to \mathbb{R}^{d} \times \mathbb{R}^{d''}$ 
has the form 
\beq\label{F is MLP}
  F_{\xi}(z) = (\pi_d(z), H_{\xi}(z)) ,\quad
  z \in \mathbb{R}^{d+d'},
\eeq
where $H_{\xi} : \mathbb{R}^d \times \mathbb{R}^{d'} \to \mathbb{R}^{d''}$ is an MLP. 
We view $ F_{\xi}$ also as maps $ (\mu,z)\to F_{\xi}(\mu,z)=F_{\xi}(z)$ that are, however, independent of $\mu$. 
This makes $F_\xi \diamond \Gamma_\theta$ well defined. 
Finally, let $y_0\in [-L,L]^n$ and $w_{y_0}:\Omega\times
[-L,L]^n \to \Omega\times [-L,L]^n$
be the map $w_{y_0}(x,y)=(x,y_0)$. Moreover, let $w_{y_0}^*:F\to F\circ w_{y_0}$ be the pull-back, 
or in terminology of ML, a warping operator with $w_{y_0}$, see \cite{muser2026flowerswarpdriveneural}.
In Appendix~\ref{app:thm:universal-graph-trans} we prove the following  approximation theorem. 

\begin{theorem}
\label{thm:approximation-trans self-attention Takashi}
Let $\widehat{\Gamma} : \mathcal P(\overline\Omega \times [-L,L]^n) \times (\overline\Omega \times [-L,L]^n) \to \overline\Omega \times [-L,L]^n$  
be a continuous map of the form \eqref{widehat f map} and \eqref{widehat f map2} and let $G_{\widehat{\Gamma}}(\mu)={\widehat{\Gamma}}(\mu,\cdot)_\#\mu$ be a
function graph transformer.  Then, for any $\varepsilon \in (0,1)$, 
there exists a sufficiently deep measure-theoretic transformer, $G_{\Gamma^{\rm tran}} :\mathcal P(\overline\Omega\times [-L,L]^n) \to \mathcal P(\overline\Omega\times [-L,L]^n)$ of the form
\beq
\label{G tran formula}
  G_{\Gamma^{\rm tran}}(\mu):= {\Gamma}^{\mathrm{tran}}(\mu, \cdot)_\# \mu, \quad {\Gamma}^{\mathrm{tran}}(\mu,\cdot) := w_{y_0}^*(F_{\xi_L} \diamond \Gamma_{\theta_L} \diamond \ldots \diamond F_{\xi_1} \diamond \Gamma_{\theta_1}(\mu,\cdot)),
\eeq
where $y_0\in [-L,L]^n$ and
$\Gamma_{\theta_\ell} : \mathcal P(\mathbb{R}^d \times \mathbb{R}^{d_\ell}) \times \mathbb{R}^d \times \mathbb{R}^{d_\ell} \to \mathbb{R}^d \times \mathbb{R}^{d_{\ell}'}$  and $F_{\xi_\ell} : \mathbb{R}^d \times \mathbb{R}^{d_{\ell}'} \to \mathbb{R}^{d} \times \mathbb{R}^{d_{\ell+1}}$ are  defined by  \eqref{Gamma theta} and \eqref{F is MLP} with $d_1=n$ and $d_\ell= 4d+4n$ for $\ell=2,\dots,L$, and $d_{L+1}=n$ and $d'_{\ell}= 4d+4n$ for $\ell=1,2,\dots,L$, that satisfies 
\beq
\label{sup G tran formula}
  \sup_{\mu \in {\mathcal P}(\overline\Omega\times [-L,L]^{n})} W_1(G_{\Gamma^{\rm tran}}(\mu), G_{\widehat{\Gamma}}(\mu)) \leq \varepsilon.
\eeq
In particular, there exists  $G_{\Gamma^{\rm tran}_m}$ that converges 
uniformly in $W_1$ to $G_{\widehat{\Gamma}}$ as $m\to \infty$.
Moreover, 
\beq\label{first coordinate is x}
    {\Gamma}^{\mathrm{tran}}(\mu,(x,y))
    = (x,\, {\Gamma}^{(y)}_{\mathrm{tran}}(\mu,x)),\quad\hbox{where
    ${\Gamma}^{(y)}_{\mathrm{tran}}(\mu,x)$ is independent of $y$.}
\eeq
\end{theorem}

\commented{
With this composition, the in-context map, $G^{\rm tran}$ say, for a multi-layer measure-theoretic transformer is obtained. To be precise, the composition should alternate between in-context maps and context-free MLPs, with activation functions that are at least Lipschitz functions, or smoother. We denote those by$F(\mu,z) =
 F(z)$.

By using  \citep[Theorem 1]{furuya2025transformers}, we obtain the following universal approximation result.

\commented{
\Takashi{If $g$ is independent of $y$, can we fined $g_{\mathrm{tran}}$ which is also independent of $y$ ? This would be an excellent result.
Can we show two different result: First, if ${\Gamma}(\mu,(x,y))$ is an arbitrary map, then $G(\mu)={\Gamma}(\mu,\cdot)_\#\mu$ can
be approximates with classical transformers based on softmax operation. Second, if  $f$ is of the special form
(that is enough to the function graph traformeres) ${\Gamma}(\mu,(x,y))=(\tilde {\Gamma}^x(\mu,x),\tilde {\Gamma}^y(\mu,x))$
where ${\Gamma}(\mu,(x,y))=x$, then the we can do the approximation using classical transformers and the MLPs that have
the same restricted form as $f.$}
}

The examples for compact sets $K$ in $\mathcal M(\Omega\times [-L,L]^h)$ is that 
\begin{itemize}
    \item $\Pp(\Omega\times [-L,L]^h)$
    \item $\mathcal M^+_C(\overline\Omega\times [-L,L]^h) =\left\{ \mu \in M^+(\Omega\times [-L,L]^h) : \mu(\Omega\times [-L,L]^h) \leq C \right\}$
\end{itemize}

\begin{proof} 
By using \citep[Theorem 1]{furuya2025transformers}, for any $\varepsilon \in (0,1)$, there is a measure-theoretic transformer-style in-context mapping ${\Gamma}^{\mathrm{tran}} := F_{\xi_L} \diamond \Gamma_{\theta_L} \diamond \ldots \diamond F_{\xi_1} \diamond \Gamma_{\theta_1}$ such that 
\beq\label{pointwise approximation}
\sup_{(\mu,(x,y)) \in K \times (\Omega\times [-L,L]^h)} |{\Gamma}^{\mathrm{tran}}(\mu, (x,y)) - {\Gamma}(\mu, (x,y))|\leq \varepsilon,
\eeq
which implies that, by the duality theorem of Kantorovich and Rubinstein, 
\begin{align*}
W_1(G_{\mathrm{tran}}(\mu), G(\mu)) 
&
=\sup_{\mathrm{Lip}(\varphi)\leq 1} \int_{\Omega\times [-L,L]^h} \varphi({\Gamma}^{\mathrm{tran}}(\mu,(x,y))) - \varphi({\Gamma}(\mu,(x,y))) d\mu(x,y)
 \\
 &
\leq \int_{\Omega\times [-L,L]^h} |{\Gamma}^{\mathrm{tran}}(\mu,(x,y)) - {\Gamma}(\mu,(x,y))| d\mu(x,y) \leq \varepsilon.
\end{align*}

For latter half statement, 
as the target map ${\Gamma}: \mathcal M(\Omega\times [-L,L]^n) \times (\Omega\times [-L,L]^n) \to  [-L,L]^n$  
is a continuous map for the form \eqref{widehat f map} and \eqref{widehat f map2},
we see using \eqref{pointwise approximation} that if we compose the map 
${\Gamma}^{\mathrm{tran}}(\mu,(x,y))=({\Gamma}^x_{\mathrm{tran}}(\mu,(x,y)),{\Gamma}^y_{\mathrm{tran}}(\mu,(x,y)))
$ with the maps $F_{\xi_{L+1}}=Id$ and 
$$
\Gamma_{\theta_{L+1}}(\mu,(x,y))=(x-{\Gamma}^x_{\mathrm{tran}}(\mu,(x,y)),y)
$$
we see that 
$\tilde {\Gamma}^{\mathrm{tran}}(\mu,(x,y))=F_{\xi_L} \diamond \Gamma_{\theta_L} \diamond {\Gamma}^{\mathrm{tran}}(\mu,(x,y))$
satisfies the formula \eqref{first coordinate is x}.
\qed
\end{proof}
}

\commented{
[We could aim the following universal approximation result].

\bigskip
{\bf Conjecture 1.}
{\it We conjecture that the following  is valid:

Let ${\Gamma}: \mathcal M^+(\overline\Omega\times [-L,L]^n) \times (\overline\Omega\times [-L,L]^n) 
\to \RR^d\times \RR^n$ 
be continuous map for the form \eqref{widehat f map} and \eqref{widehat f map2} and let $G(\mu)={\Gamma}(\mu,\cdot)_\#\mu$. 

For any $\varepsilon \in (0,1)$, there exists a sufficiently deep measure-theoretic transformer, ${\Gamma}_{\rm tran}:\mathcal M^+(\overline\Omega\times [-L,L]^n)  
\to \mathcal M^+(\overline\Omega\times \R^n)$ of the form
$$
G^{\rm tran}(\gamma)=({\Gamma}^{\mathrm{tran}})_\# \gamma,
$$
where ${\Gamma}^{\mathrm{tran}}$ is the map
\beq
{\Gamma}^{\mathrm{tran}} := F_{\xi_L} \diamond \Gamma_{\theta_L} \diamond \ldots \diamond F_{\xi_1} \diamond \Gamma_{\theta_1}
\eeq
where $\Gamma_{\theta_\ell}$ are maps of the form \eqref{trad-transformer}
and $F_{\xi_L}(x,y)=Id(x,y)+(0,H_{\xi_L}(x))$ where  $H_{\xi_L}:\R^d\to \R^n$ are MLPs
(that is, a deep composition of multi-head self-attention maps and MLPs of restricted form), such that 
formula \eqref{sup G tran formula} is valid.}

\bigskip

\commented{
{\color{black}
[Takashi and Maarten, in you paper \citep[Theorem 1]{furuya2025transformers},
assuming that $\pi_1({\Gamma}(\mu,(x,y)))=x$ that is, $f$ is the form 
${\Gamma}(\mu,(x,y)=(x,{\Gamma}^y(\mu,(x,y))$ where the first coordinate in not changed, can we
choose ${\Gamma}^{\mathrm{tran}}$ be such that it has the same property?]
}

\Takashi{That is true if MLPs can represent the identity. Indeed, if target map has the form ${\Gamma}(\mu,(x,y))=(x,{\Gamma}^y(\mu,(x,y))$, we first apply \citep[Theorem 1]{furuya2025transformers} to the map ${\Gamma}^y$, then there exist ${\Gamma}^y_{\mathrm{tran}} := F_{\xi_L} \diamond \Gamma_{\theta_L} \diamond \ldots \diamond F_{\xi_1} \diamond \Gamma_{\theta_1}$ such that 
\[
{\Gamma}(\mu,(x,y)) \approx (x, {\Gamma}^y_{\mathrm{tran}}(\mu,(x,y))).
\]
If MLPs can represent the identity, 
we can modify $\Gamma_{\theta_\ell}, F_{\xi_L}$ so that they can implement the graph, i.e., we can show that  
\[
{\Gamma}^{\mathrm{tran}}(\mu,(x,y)):= F_{\xi'_L} \diamond \Gamma_{\theta'_L} \diamond \ldots \diamond F_{\xi'_1} \diamond \Gamma_{\theta'_1}
= (x, {\Gamma}^y_{\mathrm{tran}}(\mu,(x,y))).
\]
}

\begin{proof}
As $f$ is of the form \eqref{widehat f map} and \eqref{widehat f map2 modified},
we can write ${\Gamma}(\mu,(x,y))=(x,{\Gamma}^y(\mu,x))\in \R^d\times \R^n$,
where ${\Gamma}^y(\mu,\cdot):\R^d\to\R^n$ is continuous and 
${\Gamma}^y(\mu,(x,y))=y+\widehat p(\mu,x)$. By applying \citep[Theorem 1]{furuya2025transformers} for 
the function ${\Gamma}:\mathcal M^+(\overline\Omega\times [-L,L]^n) \times (\overline\Omega\times [-L,L]^n) 
\to \RR^d\times \RR^n $,
we see that for any $\varepsilon \in (0,1)$, there is a 
in-context mapping $\tilde {\Gamma}^{\mathrm{tran}} := \tilde F_{\xi_L} \diamond \tilde \Gamma_{\theta_L} \diamond \ldots \diamond \tilde F_{\xi_1} \diamond \tilde \Gamma_{\theta_1}$,
where $\tilde \Gamma_{\xi_\ell}: \mathcal M^+(\overline\Omega\times [-L,L]^n) \times (\overline\Omega\times [-L,L]^n) 
\to (\overline\Omega\times [-L,L]^n) $ are  transformers and 
$\tilde F_{\xi_\ell}: (\overline\Omega\times [-L,L]^n) \to (\overline\Omega\times [-L,L]^n)$ are MLPs
described in \citep{furuya2025transformers},
such that 
\beq\label{1st ftran estimate}
\sup_{(\gamma,(x,y)) \in \Pp(\overline\Omega\times [-L,L]^n) \times (\overline\Omega\times [-L,L]^n)} |\tilde {\Gamma}^{\mathrm{tran}}(\gamma, (x,y)) - {\Gamma}(\gamma, (x,y))|\leq \varepsilon.
\eeq
Let us write  $\tilde {\Gamma}^{\mathrm{tran}}$ using coordinates, 
$\tilde {\Gamma}^{\mathrm{tran}}(\mu,(x,y))
=(\tilde {\Gamma}^{\mathrm{tran}}^x(\mu,(x,y)),\tilde {\Gamma}^{\mathrm{tran}}^y(\mu,(x,y)))\in \R^d\times \R^n$.

As ${\Gamma}(\mu,(x,y))=(x,{\Gamma}^y(\mu,x))\in \R^d\times \R^n$, formula \eqref{1st ftran estimate} implies
\beq\label{2nd ftran estimate a}
&&\sup_{(\gamma,(x,y)) \in \Pp(\overline\Omega\times [-L,L]^n) \times (\overline\Omega\times [-L,L]^n)} 
|\tilde {\Gamma}^x_{\mathrm{tran}}(\gamma, (x,y)) - x|\leq \varepsilon,\\
\label{2nd ftran estimate b}
&&\sup_{(\gamma,(x,y)) \in \Pp(\overline\Omega\times [-L,L]^n) \times (\overline\Omega\times [-L,L]^n)} 
|\tilde {\Gamma}^y_{\mathrm{tran}}(\gamma, (x,y)) - {\Gamma}^y(\gamma,x)|\leq \varepsilon,
\eeq
Let us define ${\Gamma}^{\mathrm{tran}}(\mu,(x,y)):=
(x,\tilde {\Gamma}^{\mathrm{tran}}^y(\mu,(x,y_0)))\in \R^d\times \R^n$ with the constant value $y_0=0$.

****TO BE CONTINUED***


Then \eqref{2nd ftran estimate a} and \eqref{2nd ftran estimate b} imply
\beq\label{3rd ftran estimate}
\sup_{(\gamma,(x,y)) \in \Pp(\overline\Omega\times [-L,L]^n) \times (\overline\Omega\times [-L,L]^n)} | {\Gamma}^{\mathrm{tran}}(\gamma, (x,y)) - {\Gamma}(\gamma, (x,y))|\leq 2\varepsilon.
\eeq
By above construction of ${\Gamma}^{\mathrm{tran}}$ implies
$\tilde {\Gamma}^{\mathrm{tran}} := \tilde F_{\xi_L} \diamond \tilde \Gamma_{\theta_L} \diamond \ldots \diamond \tilde F_{\xi_1} \diamond \tilde \Gamma_{\theta_1}$,
where $\tilde \Gamma_{\xi_\ell}: \mathcal M^+(\overline\Omega\times [-L,L]^n) \times (\overline\Omega\times [-L,L]^n) 
\to (\overline\Omega\times [-L,L]^n) $ are  transformers and 
$\tilde F_{\xi_\ell}: (\overline\Omega\times [-L,L]^n) \to (\overline\Omega\times [-L,L]^n)$ are MLPs
****

which implies that, by the duality theorem of Kantorovich and Rubinstein, 
\begin{align*}
 W_1({\Gamma}^{\mathrm{tran}}(\gamma), {\Gamma}(\gamma)) 
 &
 =\sup_{\mathrm{Lip}(\varphi)\leq 1} \int_{\overline\Omega\times [-L,L]^n} \varphi({\Gamma}^{\mathrm{tran}}(\gamma,(x,y))) - \varphi({\Gamma}(\gamma,(x,y))) d\gamma(x,y)
 \\
 &
 \leq \int_{\overline\Omega\times [-L,L]^n} |{\Gamma}^{\mathrm{tran}}(\gamma,(x,y)) - {\Gamma}(\gamma,(x,y))| d\gamma(x,y) \leq \varepsilon.
\end{align*}
This proves the claim.
\hfill $\square$
\end{proof}
}

\subsection{Main result ???}
}

\commented{
\begin{theorem}
\label{thm:approximation-trans self-attention Takashi}
Let $\Omega \subset \mathbb{R}^d$ be a compact set. 
Let ${\Gamma}: \mathcal P(\Omega\times [-L,L]^n) \times (\Omega\times [-L,L]^n) \to \RR^{d}\times  [-L,L]^n$ 
be uniformly continuous with respect to $W_1$ and satisfy
\[
{\Gamma}(\mu, (x,y)) = (x, \widehat p(\mu, (x,y)))
\]
where $\widehat p : \mathcal P(\Omega\times [-L,L]^n) \times (\Omega\times [-L,L]^n) \to  [-L,L]^n$.
Let $G(\mu)={\Gamma}(\mu,\cdot)_\#\mu$. Then, for any $\varepsilon \in (0,1)$, there exists a sufficiently deep measure-theoretic transformer, $G^{\rm tran}:\mathcal P(\Omega\times [-L,L]^n)  
\to \mathcal P(\mathbb{R}^{d} \times  [-L,L]^n)$ of the form
\begin{equation}\label{G tran} 
    G^{\rm tran}(\mu):= {\Gamma}^{\mathrm{tran}}(\mu, \cdot)_\# \mu, \quad {\Gamma}^{\mathrm{tran}} := F_{\xi_L} \diamond \Gamma_{\theta_L} \diamond \ldots \diamond F_{\xi_1} \diamond \Gamma_{\theta_1}
\end{equation}
such that
\begin{equation}
    \label{Wasserstein limit}
\sup_{\mu \in  \mathcal P(\Omega\times [-L,L]^n)} W_1(G^{\rm tran}(\mu), G(\mu)) \leq \varepsilon.
\end{equation}
Moreover, each $\Gamma_{\theta_\ell} : \mathcal P(\mathbb{R}^d \times \mathbb{R}^{d_\ell}) \times \mathbb{R}^d \times \mathbb{R}^{d_\ell} \to \mathbb{R}^d \times \mathbb{R}^{d_{\ell}'}$ 
is   defined with $d_1=n$ and $d_\ell= 4d+4n$ for $\ell=2,\dots,L$,  and $d_{L+1}=n$  
and  $d'_{\ell}= 4d+4n$ for $\ell=1,2,\dots,L$  
has the form
\begin{equation} \label{Gamma theta}
    \Gamma_{\theta_\ell}(\mu,z)
= 
\bigg(\pi_d(z)\ ,\ W^L_{\theta_\ell}z +
    \sum_{h=1}^H W^h_{\theta_\ell}
    \int_{\R^{d+d_\ell}} \frac{
        \exp\Big(
            \frac{1}{\sqrt{k}}
            \dotp{Q^h_{\theta_\ell} z}{K^h_{\theta_\ell} z'}
        \Big)
    }{
        \int _{\R^{d+d_\ell}}
        \exp\Big(
            \frac{1}{\sqrt{k}}
            \dotp{Q^h_{\theta_\ell} z}{K^h_{\theta_\ell} z''}
        \Big)
        d \mu(z'')
    } V^h_{\theta_\ell} z'\, d \mu(z')\bigg), 
\end{equation}
where $(\mu,z) \in \mathcal P(\mathbb{R}^d \times \mathbb{R}^{d_\ell}) \times (\mathbb{R}^d \times \mathbb{R}^{d_\ell})$. Moreover, $F_{\xi_\ell} : \mathbb{R}^d \times \mathbb{R}^{d_{\ell}'} \to \mathbb{R}^{d} \times \mathbb{R}^{d_{\ell+1}}$ 
has the form 
\[
F_{\xi_\ell}(z) = (\pi_d(z), H_{\xi_\ell}(z) ), \ z \in \mathbb{R}^d \times \mathbb{R}^{d_\ell}
\]
where $H_{\xi_\ell} : \mathbb{R}^d \times \mathbb{R}^{d_\ell} \to \mathbb{R}^{d_{\ell+1}}$ is MLP. 
Moreover,  $H_{L} : \mathbb{R}^d \times \mathbb{R}^{d_\ell} \to [-L,L]^{n}$.
\end{theorem}
}

\section{Expressibility results for nonlinear operators}
\label{sec: Expressivity}

\commented{
\subsection{Regularized map from (generalized) functions to measures {\color{black} [(I suggest moving this subsection to Section \eqref{sec: Expressivity}) -- This in now moved in Section \eqref{sec: Expressivity} following Maarten's suggeston (Matti)]}}

For $h \in H^{-s}(\Omega)$, we introduce
\[
  {\bf M}_{\tau}(h) = {\bf M}(R_\tau h) ,
\]
where\commented{Let's consider if we can use $  \tilde R_\tau h = \rho_\tau \ast (h \cdot {\bf 1}_{\Omega_\tau})$ as it is easier to explain }
\begin{equation}
  R_\tau h = \rho_\tau \ast (h \cdot (\rho_\tau \ast {\bf 1}_{\Omega_\tau})) ,
\end{equation}
with $\Omega_\tau = \{x \in \Omega\ :\ \hbox{dist}(x,\p \Omega) > 4 \tau\}$ and 
\begin{equation}
  \rho_\tau \ast h(x) = \int_{\R^d} \rho_\tau(x-y) h(y) \, dy .
\end{equation}
{\color{red} [Is this needed only in preparation for discretization? That is, can this be avoided otherwise?] -- This operator is used in the last step to map a sum of delta distributions
which approximate $\gamma_{A(h)}$ to a function. We two maps from a measure to a function;
The map $S_\eta$ that maps measure close to $\gamma_h$  to  a function close to $h$
and the map ${\bf F}_\varepsilon$ that maps a sum of delta distributions that is
close to $\gamma_{A(h)}$ to  a function close to $A(h)$.}
}

\subsection{Regularization of measures originating from functions in $H^{-s}(\Omega)$, $s \ge 0$}
\label{ssec:regmeastofunc}



Here, we construct function graph transformers that represent nonlinear operators
$A : H^{-s}(\Omega;\mathbb R^n) \to H^{-s'}(\Omega;\mathbb R^{n'})$,
where $H^s$ denotes a Sobolev space, $s \ge 0$, and $s' \in \mathbb R$.
This extends the framework beyond pointwise-defined continuous inputs to
low-regularity objects such as distributions and finite measures. Such
spaces are also natural in microlocal analysis, where one studies how
operators act on singularities; our result suggests that graph-transformer
representations can be formulated in a setting compatible with those
distributional viewpoints.

In Definition~\ref{def: F eps map}, we mapped measures to continuous functions. Below, we map measures to rough Sobolev functions in a way that makes the errors caused by
regularization effects weaker. To this end, we introduce a \textit{regularized map from (non-graph) measures to (generalized) functions}.
\commented{
Let's consider if we could remove operators $S_\eta$ and replace those
by the operator ${\bf F}_\e$.
}


We introduce two parameters, $\eta > 0$ and $K \in \mathbb N$. Let $k \in \mathbb Z_+$ and ${\color{black}{r}} > \frac d2+k$ be an even integer.

Let  $\psi_1, \psi_2, \dots$ be a complete orthonormal basis in $L^2(\Omega;\R^n)$ of eigenfunctions satisfying $(I-\Delta)\psi_j={{\lambda}_j}\psi_j$ and $\psi_j|_{\p \Omega} = 0$. In Appendix ~\ref{app:S-regularizers} we define a regularization operator $\gamma \to S_{\eta,K}(\gamma)$ that maps a measure $\gamma$ to a function that has $r$ (weak) derivatives. When $\gamma = \frac 1N \sum_{j=1}^N \delta_{(x_j,y_j)}$ is a discrete measure, $S_{\eta,K}(\gamma)$ is the solution of  
the minimization problem, 
\begin{align}
\label{eq:def-S discrete}
S_{\eta,K}(\gamma)
= \argmin_{\tilde h\in {\color{black}H_D^r(\Omega)}} \
  \eta\|(I-\Delta)^{r/2}\tilde h\|_{L^2(\Omega)}^2+\sum_{k=1}^K\bigg| 
 \sum_{j=1}^N  \frac{\psi_k(x_j)}N\cdot y_j-
  \int_{{\overline\Omega}} \psi_k(x)\,\tilde h(x) dx   \bigg|^2
 ,
\end{align}
where ${\color{black} H_D^r(\Omega)}$ 
is the Sobolev space of functions with the Navier boundary conditions (that is,
iterated Dirichlet conditions). Formula \eqref{eq:def-S discrete} is generalized to arbitrary probability measures $\gamma$ on ${\overline\Omega} \times [-L,L]^n$ in Appendix ~\ref{app:S-regularizers}. When $h$ is a continuous function and $\gamma_h$ is its graph measure, the operator  $h \to S_{\eta,K}(\gamma_h)$ can be considered as a low-pass frequency filter where $K$ corresponds to sharp frequency cut-off at frequency $C K^{2/d}$, $r$ corresponds to the filter $(1 + |\omega|^2)^{-r/2}$ (where $\omega$ is the Fourier
variable dual to $x$), and $\eta$ determines the balance of these two filters.
As $S_{\eta,K}(\gamma)$ is a solution of a quadratic minimization problem, $\gamma \to S_{\eta,K}(\gamma)$ is a linear operator and, hence, can be considered as a generalization of a linear neural operator that operates on measures. In particular, $h \to S_{\eta,K}(\gamma_h)$ is a linear neural operator. We choose a suitable $K = K(\eta)$ and denote $S_{\eta} := S_{\eta,K(\eta)}$. Then, $S_{\eta}(\gamma_{R_{\tau} h}) \to h$ as $\eta \to 0$, see Lemma \ref{S-regularizers} in Appendix~\ref{app:S-regularizers}.



\subsection{Representation of a nonlinear operator using a function graph transformer}

Let $L_1>L$, $k\ge 0$, $0< s\le  r-\tfrac d2$, $s-\tfrac 12\not \in \mathbb Z$, $s'\ge 0$.
Assume that $A$ is a possibly nonlinear operator with the properties that
\begin{align}
\label{A assumption1}
& \hspace{-6mm}A: \hbox{cl}_{H^{-s}(\Omega;{\R}^n)}(C(\overline \Omega;{[-L_1,L_1]}^n))
\to \hbox{cl}_{H^{-s'}(\Omega;{\R}^n)}(C(\overline \Omega;{[-L,L]}^n)),\\
\label{A assumption3}
& \hspace{-6mm}A: C^k(\overline \Omega;{[-L_1,L_1]}^n)
\to     
 C^1(\overline \Omega;{[-L,L]}^n)
\end{align}
are continuous. We point out that $L^\infty(\Omega;{[-L_1,L_1]}^n)\subset \hbox{cl}_{H^{-s}(\Omega;{\R}^n)}(C(\overline \Omega;{[-L_1,L_1]}^n))$ i.e., we allow discontinuous input functions for $A.$
\footnoteforchecking{
This footnote is just for us to make checking proofs easier: Here are the reasons for the required mapping properties: The 1st and 2nd properties of $A$ imply that
 $$A(\hbox{cl}_{H^{-s'}(\Omega;{\R}^n)}(C(\Omega;[-L,L]^n)))\subset \hbox{cl}_{H^{-s}(\Omega;{\R}^n)}(C(\Omega;[-L,L]^n)).$$
 Also, by we have that $R_\tau((\hbox{cl}_{H^{-s'}(\Omega;{\R}^n)}(C(\Omega;[-L,L]^n))))\subset C^1(\overline \Omega;{[-L,L]}^n)\cap C(\Omega;[-L,L]^n).$
 The 3rd mapping property that $A: C^k(\overline \Omega;{[-L,L]}^n)
\to C^1(\overline \Omega;{[-L,L]}^n)$ is need to make the map $(x,y)\to {\Gamma}(\mu,(x,y))$ a Lipschitz map in $(x,y)$ and this is needed to make the map
 $\mu\to {\Gamma}(\mu,\cdot)_\#\mu$ continuous in Wasserstein metric.
 }

Let $\Psi : \R^n \to \R^n$ be $C^\infty$, that is the identity function in $[-L,L]^n$ and whose image is contained in $[-L_1,L_1]^n$. The function $\Psi$ is defined precisely in Appendix~\ref{sec:Basic notions}. We use $\Psi$ to define a Nemytskii operator \citep{NemytskiiIsaia} that clips the values of $S_\eta(\gamma)$ and define $\tilde S_\eta(\gamma):=\Psi\circ S_\eta(\gamma)$. The following lemmas are proven in Appendices~\ref{app:lem. f and G are continuous} and \ref{app:lem: formula G for gamma h}.

\commented{
\Takashi{Question: In Lemma~\ref{lem. f and G are continuous}, is the map $x \mapsto \Gamma(\mu, (x,y))$ continuous when the range of $A$ is in $H^{-s}$ ? (e.g., If $A(h)$ is discontinuous function, then $x \mapsto A(h)(x)$ is not continuous.) I think the safe setting would be 
\[
A : \hbox{cl}_{H^{-s'}(\Omega;{\R}^n)}(C(\Omega;[-L,L]^n)) \to C(\overline{\Omega};{[-L,L]}^n)
\]? 
} 

\Matti{[This is an excellent observation. I changed the regularization to be in  
$H^{\color{black}{r}}(\Omega) \subset C^k(\overline \Omega)$ with ${\color{black}{r}}>\frac d2+k$ to fix this so that we do not need to change the assumption $A: H^{-s}(\Omega;{\R}^n)\to  H^{-s'}(\Omega;{\R}^n)$.]}
}

\begin{lemma} \label{lem. f and G are continuous}
The in-context map
${\Gamma}_\eta : {{{\mathcal P}}}(\overline\Omega\times [-L,L]^n)\times ({\overline\Omega}\times [-L,L]^n)\to {\overline\Omega}\times [-L,L]^n$ and associated $G_{\Gamma_\eta} : {{{\mathcal P}}}(\overline\Omega\times [-L,L]^n)\to {{{\mathcal P}}}(\overline\Omega\times [-L,L]^n)$, define continuous maps
\beq\label{f eta def}
{\Gamma}_\eta(\mu,(x,y)) = (x,A({\color{black}\tilde S}_{\eta}(\mu))(x))\in {\overline\Omega}\times [-L,L]^n,
%
\quad G_{\Gamma_\eta}(\mu) = {\Gamma}_\eta(\mu,\cdot)_\# \mu.
\eeq
\end{lemma}

\begin{lemma}\label{lem: formula G for gamma h}
For $h\in C(\overline \Omega)$ we have $G_{\Gamma_\eta}(\gamma_h) = \gamma_g$, where $g(x) = A({\color{black}\tilde S}_{\eta}(\gamma_h))(x)$.
\end{lemma}

The discrete measure 
${\bf M}_N(h)$ in \eqref{f mollified 1B} can
be considered as a tokenized version of the function $h$.
When $N$ is large, the continuity of $G_{\Gamma_\eta}$ implies that 
the measure
$\nu=\tfrac 1N\sum_{j=1}^N\delta_{(\bar x_j,g_j)}:=G_{\Gamma_\eta}({\bf M}_N(h))$
approximates $\gamma_g$, where $g$ is given in Lemma \ref{lem: formula G for gamma h}.
The measure $\nu$ can be considered as a tokenized version of $g$.
Then, formula
\beq\label{new F eps and g formula}
{\bf F}_\e\nu =\sum_{j=1}^N
g_j \frac {\mathrm{vol}(\Omega)}N \frac {1}{\e^d}  \rho\left(\frac{x-\bar x_j}{\e}\right)
\eeq
 produces a function approximating $g$ from the tokens $\xi_j=(\bar x_j,g_j)$, $j=1,\dots,N$,
see App.~\ref{App: example}.


Combining Theorem~\ref{thm:approximation-trans self-attention Takashi} and Lemmas~\ref{lem. f and G are continuous} and \ref {lem: formula G for gamma h},
we can prove the following universal approximation theorem; see Appendix~\ref{app:main thm} for the proof.



\begin{theorem}\label{main thm}
For $h \in \hbox{cl}_{H^{-s}(\Omega;{\R}^n)}(C(\Omega;[-L,L]^n))$ it holds that
\beq
  A(h) = \lim_{\tau\to 0}\lim_{\eta \to 0}\lim_{\varepsilon\to 0}\lim_{N\to \infty}\lim_{m\to \infty}{\bf F}_\varepsilon\circ
  G_{\Gamma_{\eta,m}^{\mathrm{tran}}}
\circ {\bf M}_{N,\tau} (h),
\eeq
where the limit is valid in $H^{-s'}(\Omega)$ and 
$G_{\Gamma_{\eta,m}^{\mathrm{tran}}}$
are softmax-based traditional transformers that approximate $G_{\Gamma_\eta}$, given by
Theorem \ref{thm:approximation-trans self-attention Takashi}. 
\footnoteforchecking{Here is comment for us in checking the proofs:
For $h\in \hbox{cl}_{H^{-s'}(\Omega;{\R}^n)}(C(\Omega;[-L,L]^n))$, when we compute first $N\to \infty$ limit we hae 
\beq
A(h)&=&\lim_{\tau\to 0}\lim_{\eta \to 0}\lim_{\varepsilon\to 0}{\bf F}_\varepsilon\circ G_\eta\circ {\bf M} (R_\tau(h))
\\ &=&\lim_{\tau\to 0}\lim_{\eta \to 0}\lim_{\varepsilon\to 0}\rho_\varepsilon*(A((\eta(I-\Delta)^{r}+I)^{-1}P_{K(\eta)}(R_\tau(h)))),
\eeq
where $P_K$ is the projection to $K$ first eigenfunctions of Laplacian. So, in the last formula line we have no measures and are limits are for functions in Sobolev spaces and to analyze the last formula line we need only that $A: H^{-s'}(\Omega;{\R}^n)\to  H^{-s}(\Omega;{\R}^n)$ is continuous.}
Moreover, for $h\in C^r_c( \Omega; [-L,L]^n)$
\beq\label{simplified on C}
A(h)=\lim_{\eta \to 0}\lim_{\varepsilon\to 0}\lim_{N\to \infty}\lim_{m\to \infty}{\bf F}_\varepsilon\circ 
G_{\Gamma_{\eta,m}^{\mathrm{tran}}}
\circ {\bf M}_N (h),
\eeq
where the limit is valid in $H^{-s'}(\Omega)$. 
\end{theorem}


Theorem  \ref{main thm} implies that when the sequence of tokens length, $N$, increases, the mollification parameter, $\tau$, used in the encoding of a rough (or discontinuous) function, $h$, by tokens decreases, and the depth of the softmax based transformers with self attention (modeled by the parameter $m$) increases, the transformer can approximate arbitrarily well $A(h)$ where $A$ acts on a rough input function, $h$.

The above theorem establishes universal approximation with respect to pointwise convergence, which is weaker than uniform convergence, since the proof relies on the smooth mollifier argument in \eqref{f mollified 2C}.

\commented{
\begin{remark}
A nonlinear map between Barron spaces admits an alternative representation by a measure-theoretic transformer. This is described in the Appendix~\ref{App:Barron}. \Matti{I suggest that we remove the Barron space discussion if we do
not have time to polish the appendix.}
\end{remark}
}

\subsection{Representation of a nonlinear operator between different function spaces}
\label{sec:different-spaces}

Let $\Omega \subset \R^d$ and $D \subset \R^{d'}$ be open bounded sets with $C^\infty$-smooth boundary, or cubes. Below, 
we use in $\overline\Omega\times [-L,L]^n$ the coordinates $(x,y)$ and
in $\overline D\times [-L,L]^n$ the coordinates $(x',y')$. When $\mu$ is a measure on $\overline\Omega\times [-L,L]^n$ and $\nu$ is a measure on $\overline D\times [-L,L]^n$, their tensor product, $\mu \otimes \nu$ is just the same as the product of their volume elements, that is, $d \mu \otimes \nu(x,y,x',y') = d\mu(x,y) d\nu(x',y')$. On $X = (\overline\Omega \times [-L,L]^n) \times (\overline D\times [-L,L]^n)$ we use  the maps $\rho_1((x,y),(x',y'))=(x,y)$ and $\rho_2((x,y),(x',y'))=(x',y')$, and denote  $\pi_2(x',y')=y'$ and $\tilde \pi_2=\pi_2\circ \rho_2:((x,y),(x',y'))\to y'$ 

\begin{proposition}\label{prop: Map from Omega to D}
Let $A : L^2(\Omega;{[-L_1,L_1]}^n) \to C^1(\overline D;{[-L,L]}^n)$ be a continuous, possibly nonlinear operator. Then there are function graph transformers $G_{\tilde{\Gamma}_\eta}({\tilde\mu}) = \tilde{\Gamma}_\eta({\tilde\mu},\cdot)_\#{\tilde\mu}$ with in-context functions $\tilde{\Gamma}_\eta : {{{\mathcal P}}}(X)\times X \to X$, $\eta>0$, such that for all $h \in C(\overline \Omega;{[-L,L]}^n)$ and $x' \in D$, the function $A(h)$ at $x'$ can be represented as
\beq\label{A Omega to D representation}
A(h)(x')
{\color{black}=\lim_{\eta\to 0}\bra (\rho_2)_\#  \tilde{G}_{\tilde{\Gamma}_\eta}(\gamma_h \otimes \delta_{(x',0)}), \pi_2 \ket}=\lim_{\eta\to 0}\bra  \tilde{G}_{\tilde{\Gamma}_\eta}(\gamma_h\otimes \delta_{(x',0)}), \tilde \pi_2 \ket.
\eeq
\end{proposition}


The proof is given in Appendix~\ref{proof of prop: Map from Omega to D}. Thus an operator mapping functions on $\Omega$ to functions on $D$ can be represented by a function graph transformer on $X=(\overline\Omega\times[-L,L]^n)\times(\overline D\times[-L,L]^n)$, with inputs sampled at $x_j\in\Omega$ and outputs queried at $x'_k\in D$. In fact, we can generalize formula \eqref{A Omega to D representation} to represent $A(h)$ in such a way that the function $h$ is approximated using its values $h(x_j)$ at the points $x_j \in \Omega$, $j=1,\dots,N$ 
but $A(h)$ can be evaluated at different points $x'_k \in D$, $k=1,\dots,K$. 
\commented{
To accomplish this, we  write \eqref{A Omega to D representation} in the more general form,
\beq\label{positional encoding Psi}
  \lim_{\eta\to 0} (\rho_2)_\#  \tilde {\Gamma}_\eta\big(\gamma_h\otimes \nu ,\cdot\big)_\#\big(\gamma_h\otimes \nu\big)=
 \sum_{k=1}^K \frac 1K\delta_{(x'_k,y_k)},\quad y_k = A(h)(x'_k),\
\nu=\sum_{k=1}^K \frac 1K\delta_{(x'_k,0)}.
\eeq
Then, the fully discrete version of \eqref{positional encoding Psi} is
\beq\label{positional encoding Psi2}
(\rho_2)_\#  \tilde {\Gamma}_\eta\big({\bf M}_N(h)\otimes \nu,\cdot \big)_\#\big({\bf M}_N(h)\otimes \nu\big)=
  \sum_{k=1}^K \frac 1K\delta_{(x'_k,y_k')},\quad y_k' = A(h_{\eta,N})(x'_k)
\eeq
and ${\bf M}_N(h) = \sum_{j=1}^N \tfrac 1N \delta_{(x_j,h(x_j))}$ while $h_{\eta,N} ={\color{black}\tilde S}_\eta({\bf M}_N(h))$ is a regularized version of  $h$. 

We assess the benefits of representations \eqref{positional encoding Psi} and \eqref{positional encoding Psi2}. In turns out (see \eqref{key formula 1} and  \eqref{key formula 2} for details 
}
To accomplish this, we  write \eqref{A Omega to D representation} in the more general form that is closely related to cross attention. It is shown in \eqref{key formula 1} in Appendix \ref{proof of prop: Map from Omega to D} that we can choose
$\tilde{\Gamma}_\eta(\tilde \mu,((x,y),(x',y'))) = \big((x,y),\, {\Gamma}_\eta((\rho_1)_\#\tilde \mu,(x',y'))\big)$. We consider the product measure $\mu \otimes \nu$, where $\mu = \gamma_h$ is a measure on $\overline\Omega \times [-L,L]^n$
and $\nu = \tfrac 1K\sum_{k=1}^K \delta_{(x'_k,0)}$ is a measure on $ \overline D\times [-L,L]^n$.
Then, as seen in \eqref{cross-attention ideas} in Appendix \ref{proof of prop: Map from Omega to D},
\beq\label{new cross attention}
(\rho_2)_\# G_{\tilde{\Gamma}_\eta}(\mu \otimes \nu)
={\Gamma}_\eta(\mu,\cdot)_\# \nu=
\sum_{k=1}^K \frac 1K\delta_{(x'_k,y_k)},\quad y_k = A(h)(x'_k),\
\eeq
Here, ${\Gamma}_\eta(\mu,\cdot)_\# \nu$ can be interpreted as a cross-attention: the measure $\mu=\gamma_h$ determines how the measure $\nu = \tfrac 1K\sum \delta_{(x'_k,0)}$ associated to the grid point $x'_k$ is pushed forward.
\commented{that in \eqref{positional encoding Psi} the map $\tilde {\Gamma}_\eta:(\gamma_h \otimes \nu,((x,y),(x',y')) \to \tilde {\Gamma}_\eta(\gamma_h \otimes \nu,((x,y),(x',y')))$ is a mapping that is independent of the measure $\nu$, and therefore of the points $x'_j$.
} 
This means that in the representation  \eqref{new cross attention} of the function $A(h)$ the roles of the evaluation points $x'_j$ and the image function $\gamma_{A(h)}$ are independent. From the practical point of view, this suggests a simplified architecture that one can use to learn operator $A$.
We discuss the architecture and loss functions that are suitable for function graph transformers 
in Appendix~\ref{app: training transformers architecture}, and illustrate this query-point formulation in
Appendix~\ref{App:experiments}.
 
\commented{Let us discuss more the consequences that in the representation 
\eqref{positional encoding Psi} of the function $A(h)(z)$ the roles of the evaluation
points $z_j$ and the image function $\gamma_{A(h)}$ are independent.
This is different to those operator learning approaches where one uses cross-attention, or other architectures that combine
the values $h(x_j)$ and the points $z_j$ and aims to learn the map 
$$
((h(x_j)_{j=1}^N,(z_k)_{k=1}^K)\to (A(h)(z_k))_{k=1}^K.
$$
In other words, in formula \eqref{positional encoding Psi2} the evaluation points $z_k$ do not influence the way how the image function $A(h)$ is represented. Roughly speaking, the points $z_k$ where the function $A(h)$ will be observed do not change the learning of the object function $A(h)$.}

\section{Conclusions}
\label{sec:conclusion}

We have introduced function graph transformers as a measure-theoretic framework for operator learning with transformers. By representing functions through graph measures, the proposed formulation connects finite tokenized inputs with continuum objects and makes discretization refinement amenable to analysis. Within this framework, transformer layers act as measure maps that preserve graph structure, so that operators between function spaces can be represented without leaving the class of single-valued functions.

The main results show that graph-preserving measure maps can be approximated by standard softmax self-attention layers and pointwise multilayer perceptrons. Consequently, continuous nonlinear operators between function spaces, including operators acting on low-regularity or distributional inputs after suitable regularization through a neural operator, admit approximation by transformer architectures in a discretization-invariant sense. The framework also accommodates operators between different domains, clarifying the role of spatial coordinates, positional encodings, and query points in transformer-based neural operators.

There are several limitations. The results are qualitative: they establish expressivity and convergence, but not sharp rates or sample-complexity bounds. The theory explains how standard transformer layers approximate the proposed measure maps, but does not identify optimal tokenization strategies. The results suggest that the geometry of function spaces, function graph structure, and stability under discretization refinement should be central design principles for scientific transformer models.

\section*{Acknowledgments}

TF is supported by JSPS KAKENHI Grant Number JP24K16949, 25H01453, JST CREST JPMJCR24Q5, JST ASPIRE JPMJAP2329. MVdH gratefully acknowledges the support of the Department of Energy BES, under grant DE-SC0020345, Oxy and the Simons Foundation under the MATH + X Program.
The research of M.L. was partially supported by European Research Council Advanced Grant 101097198 (PDE-Inverse) 
and the Research Council of Finland, grant 359182 (FAME Flagship). I.D. was partially supported by the European Research Council Consolidator Grant 101232533 (PhaseShift). Views and opinions expressed are those of the authors only and do not
necessarily reflect those of the European Union or the other funding
organizations.

\bibliographystyle{plainnat}
\bibliography{ref}

\newpage

\appendix

\section{Basic notions} 
\label{sec:Basic notions} 

Below, $\Omega\subset \R^d$ is assumed to be bounded and open set. 

For a Borel measure $\mu$ on $\overline\Omega\times [-L,L]^n$ and  $\phi\in C_b(\overline \Omega\times \R^n)$ we  denote
$$
\bra \mu,\phi\ket=\int_{\overline\Omega\times [-L,L]^n}\phi(x,y)\,d\mu(x,y).
$$
When $F:\Omega\to D$ is a continuous map (or sometimes, even less regular)
and $\mu$ is a measure on $D$, the push-forward of the measure $\mu$ in the map $F$,
denoted $F_\#\mu$, is a measure on $D$ that gives for measurable sets $B\subset D$
the measure
\beq\label{def: pushforward}
F_\#\mu(B)=\mu(F^{-1}(B)),\quad F^{-1}(B)=\{x\in\Omega:\ F(x)\in B\}.
\eeq

We recall that the probability measures $\mu_n$ converge to $\mu$ in $p$-Wasserstein metric if and only if
 probability measures $\mu_n$ converge to $\mu$ weakly, that is, for all continuous functions
 $\phi\in C(\overline \Omega\times \R^n)$ it holds that $\bra \mu_n,\phi\ket\to \bra \mu,\phi\ket$ as $n\to \infty,$
 see \citep[Theorem 7.12]{Villani}.

We denote $C_0(\overline \Omega)=\{u\in C(\overline \Omega):\ u|_{\p \Omega}=0\}$.
Also, for $s\in \mathbb{N}$, 
$$
H^s( \Omega)=\{u\in L^2(\Omega):\ \p_x^j u\in L^2(\Omega)
\hbox{ for all }j\le s\}
$$
are the Sobolev spaces, 
$$
H^s_0( \Omega)=
\hbox{cl}_{H^s( \Omega)}(C^\infty_c(\Omega)).
$$
Finally, $H^{-s}( \Omega)$ are the dual spaces of $H^s_0( \Omega)$. Sobolev spaces are defined also with $s\in \R$. By Sobolev embedding results in an open bounded set with $C^\infty$-smooth boundary, or a cube $[0,1]^d$,   $H^s( \Omega)\subset 
C^k(\overline\Omega)$ when $s>k+\frac d2.$ On the cube this follows by sine-series expansion and odd reflection to the torus.

\commented{
Let $\Omega\subset\R^d$ be open bounded set with $C^\infty$-smooth boundary, or a cube $[0,1]^d$,  , $r\in\Z_+$, and
let $L:=I-\Delta_D$ denote the Dirichlet realization of $I-\Delta$ in
$L^2(\Omega)$. Then $L$ is a positive selfadjoint operator with
form-domain $\mathcal D(L^{1/2})=H^1_0(\Omega)$. Let
$\{({\color{black}{\lambda}^r_j^r},\psi_j)\}_{j\ge1}\subset\R_+\times L^2(\Omega)$ be the
eigensystem of the positive selfadjoint operator
$L^r$ ordered so that
${\lambda}^r_1\le{\lambda}^r_2\le\cdots\to\infty$, with the $\psi_j$
$L^2$-orthonormal. Define
\[
  P_Kh:=\sum_{j=1}^K\langle h,\psi_j\rangle_{L^2}\,\psi_j,
  \qquad
  S_{\eta,K}h:=(\eta L^r+I)^{-1}P_Kh.
\]
}

Recall that $I-\Delta_D$ is a positive selfadjoint operator
on $L^2(\Omega)$, where $\Delta_D$ is the Dirichlet Laplacian with domain
${\mathcal D}(\Delta_D)=H^2(\Omega;\R^n)\cap H^1_0(\Omega;\R^n)=\{u\in H^2(\Omega;\R^n):\ u|_{\p \Omega}=0\}$.
Let
\[
    A_D:=I-\Delta_D,
    \qquad
    \mathcal L:=A_D^r=(I-\Delta_D)^r,
\]
where $r\in\mathbb Z_+$.  
Let $A_D{\color{blue}\omega}_j={\lambda}_j\psi_j$ where $\{\psi_j\}_{j=1}^\infty$ is an $L^2(\Omega)$-orthonormal basis,
$$
    0<{{\lambda}}_1\le {{\lambda}}_2\le\cdots,\qquad
    {{\lambda}}_j\to\infty.
$$
Note the $\psi_j\in H^1_0(\Omega)$
satisfy $\psi_j|_{\p \Omega}=0$.
The operator $\mathcal L$ is positive
selfadjoint and has compact resolvent. Then,
\[
    \mathcal L\psi_j={{\lambda}}^r_j\psi_j.
\]

By Weyl's law for the Dirichlet Laplacian on bounded Lipschitz domains,
\[
    {\lambda}_j
    =
    \left(
        \frac{(2\pi)^d}{\omega_d|\Omega|}
    \right)^{2/d}
    j^{2/d}
    (1+o(1)),
    \qquad j\to\infty.
\]
In particular, after increasing the constant if necessary, there is a
constant $C_{\mathcal L}>0$ such that
\begin{equation}
\label{eq:lambda-upper-bound}
    {\lambda}_j\le C_{\mathcal L} j^{2/d},
    \qquad j\ge 1.
\end{equation}

Define the spectral Sobolev space associated with the Dirichlet operator by
\[
    H_D^r(\Omega)
    :=
    \mathcal D(\mathcal L^{1/2})
    =
    \mathcal D\bigl((I-\Delta_D)^{r/2}\bigr),
\]
with norm
\[
    \|h\|_{H_D^r(\Omega)}^2
    :=
    \|\mathcal L^{1/2}h\|_{L^2(\Omega)}^2
    =
    \sum_{j=1}^\infty
    {\lambda}^r_j
    |\langle h,\psi_j\rangle_{L^2(\Omega)}|^2 .
\]
For vector-valued functions $h=(h_1,\ldots,h_n)$, the norm is defined
componentwise:
\[
    \|h\|_{H_D^r(\Omega;\mathbb R^n)}^2
    :=
    \sum_{\ell=1}^n
    \|h_\ell\|_{H_D^r(\Omega)}^2 .
\]
As we have assumed that the boundary of $\p\Omega$ is $C^\infty$ smooth or that 
$\Omega$ is a cube, the  operator $\mathcal L$ is selfadjoint and 
its domain $H_D^r(\Omega;\mathbb R^n)$ consists of  Sobolev functions
in $h\in H^r(\Omega;\mathbb R^n)$ that satisfy the Navier boundary
conditions $u|_{\p\Omega}=0$, $(I-\Delta)u|_{\p\Omega}=0,\dots,(I-\Delta)^{r/2}u|_{\p\Omega}=0$.
This is the main reason for the assumptions on the boundary.

Let us next define a capping function $\Psi$ that has the preferred properties:
Let $\Psi:\R^n\to \R^n$, $\Psi(y)=(\psi(y_i))_{i=1}^n$ be a component-vice function
where $\psi:\R\to \R$ is an odd $C^\infty$-smooth function such that
$\psi(y)=y$ for $|y|\le (L+L_1)/2$, $\psi(\R)\subset [-L_1,L_1]$, and $\psi(y)=L_1$ for $y>L_1$.
As described in the main text, we use  the bounded function $\Psi\in C^\infty(\R^n;\R^n)$  to define a Nemytskii operator \cite{NemytskiiIsaia} that caps values
of function to the box $[-L_1,L_1]^n$ and denote $\tilde S_\eta(\gamma)=\Psi\circ S_\eta(\gamma).$

\section{Mapping a sequence of tokens to a sequence of tokens and corresponding mapping
from measures to measures}
\label{app:tokens and measures}


\subsection{Part I: Simplest case when  the input and output grids coincide}

Functions between finite sequences of tokens are closely related to functions between discrete measures. To describe this relation,
let
$({{t}}_1,{{t}}_2,\dots,{{t}}_N) \in ({\RR^D})^N$ be sequences of $N$ tokens in ${\RR^D}$.  All sequences that are same up to a permutation of the order
in the sequence, are considered to be equivalent, and we denote such equivalence classes
of sequences of $N$ tokens by $X_d^N$, and let the union of all these be $X_d = \bigcup_{N=1}^\infty X_d^N$.
A sequence  $({{t}}_1,{{t}}_2,\dots,{{t}}_N)\in X_d^N$ can be identified with the probability measure  $\mu=\sum_{i=1}^{N} \frac 1N \delta_{{{t}}_i}$ that is a discrete measure
supported in the set $\{{{t}}_1,\dots,{{t}}_N\}$. We denote the corresponding identification map by 
\beq\label{iota map}
\iota=\iota_N:({{t}}_1,{{t}}_2,\dots,{{t}}_N)\to \sum_{i=1}^{N} \frac 1N \delta_{{{t}}_i}.
\eeq
Then a map $g:X_d^N\to X_{d}^N$ that for any  $N$  maps a sequence $({{t}}_1,{{t}}_2,\dots,{{t}}_N)$ of $D$-dimensional tokens to  a sequence
$(y_1,y_2,\dots,y_N)$ of $D$-dimensional tokens,
defines a  map $G_N=\iota_N\circ g\circ \iota^{-1}_N$ from discrete probability measures 
$$
\iota_N(X_d^N)=\{\sum_{i=1}^{N} \frac 1N \delta_{{{t}}_i}:\ {{t}}_i\in \R^D\}
$$ 
to itself, $G_N:\iota_N(X_d^N)\to \iota_N(X_d^N)$, defined by
$$
G_N(\sum_{i=1}^{N} \frac 1N \delta_{{{t}}_i})=\sum_{i=1}^{N} \frac 1N \delta_{\hat{{t}}_i},\quad (\hat{{t}}_1,\hat{{t}}_2,\dots,\hat{{t}}_N)=
g({{t}}_1,{{t}}_2,\dots,{{t}}_N),
$$
Under suitable conditions, the maps $G_N$ are compatible and define a map on the discrete probability measures, $G :\mathcal P_{fin}(\RR^{d}) \to \mathcal P_{fin}(\RR^{d})$. When these maps are continuous with respect to Wasserstein metric, see 
\cite{furuya2025transformers} for details, these maps extends to a map from
general probability measures to general probability measures, $\overline G:\Pp(\RR^{d})\to \Pp(\RR^{d})$
that is continuous with respect to the Wasserstein metric. One can consider the Wasserstein distance as a generalization of the permutation invariant distance of sequences of tokens: We recall that the 1-Wasserstein distance of the measures $\mu = \sum_{i=1}^{N} \frac 1N \delta_{{{t}}_i}$ and $\mu' = \sum_{i=1}^{N} \frac 1N \delta_{{{t}}'_i}$ is given by
\ba
W_1(\mu,\mu')=
\min_{\sigma} \frac 1N  \sum_{i=1}^{N} |{{t}}_i-{{t}}'_{\sigma(i)}|,
\ea
where the minimum is taken over the permutations, $\sigma:\{1,2,\dots,N\}\to \{1,2,\dots,N\}$.
When the number, $N$, of tokens grows, the discrete measures $\sum_{i=1}^{N} \frac 1N \delta_{{{t}}_i}$ can converge to continuous measures (or a sum of continuous measures, point measures, and measures that are singular with respect to the standard measure of $\RR^D$.) Thus to understand properties of transformers on arbitrarily long sequences, it is useful to consider mappings between ``general'' measures. 

Let us next consider an approximation of nonlinear operator $A:C(\overline \Omega;{[-L,L]}^n) \to C(\overline \Omega;{[-L,L]}^n)$,
using transformers, where $ \Omega\subset \RR^{d}$. Let
$x_j\in \overline\Omega$, $j=1,2,\dots,N$ be computation grid points in the domain $\Omega$
where both the input and output functions are given (This is simple case where the construction of the 
context functions and a flow diagram are easier to describe).
When $h\in C(\overline \Omega;{[-L,L]}^n)$ is the input function,
we its point values determine a token sequences of $N$ tokens 
\beq
\bigg((x_j,h(x_j))\bigg)_{j\in [N]}
\eeq

With the above preparations, let  consider the function graph transformer  $\overline G:\Pp(\overline \Omega\times [-L,L]^n)\to \Pp(\overline \Omega\times [-L,L]^n)$,
$ \Omega\subset \RR^{d}$, see Definition \ref{def self-attention}, from a point of view of token sequences. 
When $\overline G$  maps the graph measure $\gamma_h$ of a function $h\in C(\overline \Omega;{[-L,L]}^n)$
to  $\overline G(\gamma_h)=\gamma_{J(h)}$ that is the graph measure of a function $J(h)\in C(\overline \Omega;{[-L,L]}^n)$,
we  define 
a sequence-to-sequence map between the token sequences,
\begin{equation}\label{large g}
\overline g={\bf j}_{x,\varepsilon}  \circ \overline G\circ \iota:X_{d+n}^N\to X_{d+n}^N
\end{equation}
where ${\bf j}_{x,\varepsilon}$ is approximation of the map $\iota_N^{-1}$, associated to grid points 
$x=(x_1,\dots,x_N)$, that is given by
\beq\label{j x map}
{\bf j}_{x,\varepsilon}:\gamma \to 
 \bigg ((x_1,{\bf F}_\varepsilon \gamma(x_1))\, ,\, (x_2,{\bf F}_\varepsilon \gamma(x_2))\,,\dots,\,
 (x_N,{\bf F}_\varepsilon \gamma(x_N)) \bigg ),
\eeq
where ${\bf F}_\varepsilon$ is the convolution operator with a mollifier function,
see formula \eqref{f mollified 1B}.
This mapping 
has  the form  (when the domain deformation map $\Phi:\overline \Omega\to \overline \Omega$ 
in Def.\ \ref{def self-attention} is the identity map)
\begin{equation}\label{small g}
\overline g:\bigg((x_1,h(x_1)),(x_2,h(x_2)),\dots
,(x_N,h(x_N))\bigg)\to \bigg((x_1,u_1),(x_2,u_2),\dots
,(x_N,u_N)\bigg),
\end{equation}
where $x_j\in \overline\Omega$ are computation grid points in the domain $\Omega$
and $u_j\approx J(h)(x_j)$. Our considerations
show that $u_j$ converges to $J(h)(x_j)$ when the grid points $x_j$ are suitably
chosen and $N\to \infty$. Moreover, 
 we show that the function graph transformers can be approximations by traditional 
softmax-based transformers $\overline G^{\mathrm{tran}}=G_{\Gamma_{\eta,m}}^{\mathrm{tran}}$, given by
Theorem \ref{thm:approximation-trans self-attention Takashi}, that satisfies
\beq
\label{sup G tran formula B}
  \sup_{\mu \in {\mathcal P}(\overline\Omega\times [-L,L]^{n})} W_1(\overline G^{\mathrm{tran}}(\mu), \overline G(\mu)) \leq \varepsilon,
\eeq
where $\varepsilon>0$ depends on the depth of the transformer $\overline G^{\mathrm{tran}}$ and 
$W_1$ is the 1-Wasserstein distance.
This defines a sequence of tokens to sequence of tokens map
\begin{equation}\label{large g tran}
\overline g^{\mathrm{tran}}={\bf j}_{x,\varepsilon}  \circ \overline G^{\mathrm{tran}}\circ \iota:X_{d+n}^N\to X_{d+n}^N.
\end{equation}

\commented{
%
%
%

\begin{figure}[h]
\centering
\begin{tikzpicture}[
    >=Latex,
    font=\footnotesize,
    tokpos/.style ={draw, rounded corners=1.2pt, fill=blue!10,  thick,
                    minimum width=0.55cm, minimum height=0.55cm,
                    inner sep=1pt, font=\scriptsize},
    tokval/.style ={draw, rounded corners=1.2pt, fill=orange!18, thick,
                    minimum width=0.95cm, minimum height=0.55cm,
                    inner sep=1pt, font=\scriptsize},
    tokout/.style ={draw, rounded corners=1.2pt, fill=green!22,  thick,
                    minimum width=0.55cm, minimum height=0.55cm,
                    inner sep=1pt, font=\scriptsize},
    iotabox/.style={draw, thick, rounded corners=2pt, fill=gray!18,
                    minimum height=0.65cm, minimum width=1.0cm, font=\small},
    measbox/.style={draw, thick, rounded corners=2pt, fill=yellow!22,
                    minimum height=0.85cm, minimum width=4.0cm,
                    align=center, font=\footnotesize},
    attblk/.style ={draw, thick, rounded corners=2pt, fill=red!14,
                    minimum height=1.05cm, minimum width=2.6cm,
                    align=center, font=\scriptsize},
    mlpblk/.style ={draw, thick, rounded corners=2pt, fill=cyan!22,
                    minimum height=1.05cm, minimum width=2.0cm,
                    align=center, font=\scriptsize},
    arr/.style    ={->, thick},
    arr2/.style   ={->, very thick},
]

\def\cx{0}
\def\railx{-4.6}

\node[tokpos] (x1) at (\cx-2.55,0) {$x_1$};
\node[tokval, right=0.5pt of x1] (h1) {$h(x_1)$};
\node[tokpos, right=5pt of h1]   (x2) {$x_2$};
\node[tokval, right=0.5pt of x2] (h2) {$h(x_2)$};
\node[right=3pt of h2]           (din) {$\cdots$};
\node[tokpos, right=3pt of din]  (xN) {$x_N$};
\node[tokval, right=0.5pt of xN] (hN) {$h(x_N)$};

\node[font=\small\itshape, above=4pt of $(x1.north)!0.5!(hN.north)$]
 {Input tokens $\bigl((x_i,\,h(x_i))\bigr)_{i=1}^{N}\in X_{d+n}^{N}$};

\node[iotabox] (iota) at (\cx,-1.1) {$\iota_N$};
\foreach \n in {x1,h1,x2,h2,xN,hN}{%
  \draw[->, thin, gray!70] (\n.south) -- (\n.south |- iota.north);
}

\node[measbox] (mu0) at (\cx,-2.2)
  {$\mu_{0}=\dfrac{1}{N}\sum\limits_{i=1}^{N}\delta_{(x_i,h(x_i))}$};
\draw[arr2] (iota.south) -- (mu0.north);

\node[attblk] (G1) at (\cx,-3.7)
  {$\Gamma_{\theta_1}$\quad\scriptsize\itshape multi-head softmax attention};
\node[mlpblk] (F1) at (\cx,-4.95)
  {$F_{\xi_1}$\quad\scriptsize\itshape MLP};
\draw[arr2] (mu0.south) -- (G1.north);
\draw[arr2] (G1.south)  -- (F1.north);

\node[font=\large] (vdots) at (\cx,-5.85) {$\vdots$};
\draw[arr2, shorten <=2pt, shorten >=2pt] (F1.south) -- (vdots.north);

\node[attblk] (GL) at (\cx,-6.95)
  {$\Gamma_{\theta_L}$\quad\scriptsize\itshape multi-head softmax attention};
\node[mlpblk] (FL) at (\cx,-8.2)
  {$F_{\xi_L}$\quad\scriptsize\itshape MLP};
\draw[arr2, shorten <=2pt] (vdots.south) -- (GL.north);
\draw[arr2] (GL.south) -- (FL.north);

\draw[decorate, decoration={brace, amplitude=5pt, raise=4pt}, thick]
  (G1.north east) -- (FL.south east)
  node[midway, right=12pt, align=left, font=\scriptsize] {%
    $\Gamma^{\mathrm{tran}}=$\\
    $F_{\xi_L}\diamond\Gamma_{\theta_L}$\\
    $\diamond\;\cdots\;\diamond$\\
    $F_{\xi_1}\diamond\Gamma_{\theta_1}$};

\node[measbox, fill=yellow!35] (muL) at (\cx,-9.5)
  {$\mu_{L}=G^{\mathrm{tran}}(\mu_{0})=\Gamma^{\mathrm{tran}}(\mu_{0},\cdot)_{\#}\,\mu_{0}$};
\draw[arr2] (FL.south) -- (muL.north);

\node[iotabox] (iotainv) at (\cx,-10.6) {${\bf j}_{x,\varepsilon}$}; 
\draw[arr2] (muL.south) -- (iotainv.north);

\node[tokpos] (ox1) at (\cx-2.55,-11.7) {$x_1$};
\node[tokout, right=0.5pt of ox1] (u1) {$u_1$};
\node[tokpos, right=5pt of u1]    (ox2) {$x_2$};
\node[tokout, right=0.5pt of ox2] (u2) {$u_2$};
\node[right=3pt of u2]            (dout) {$\cdots$};
\node[tokpos, right=3pt of dout]  (oxN) {$x_N$};
\node[tokout, right=0.5pt of oxN] (uN) {$u_N$};

\foreach \n in {ox1,u1,ox2,u2,oxN,uN}{%
  \draw[->, thin, gray!70] (\n.north |- iotainv.south) -- (\n.north);
}

\node[font=\small\itshape, below=4pt of $(ox1.south)!0.5!(uN.south)$, align=center]
  {Output tokens $\bigl(x_i,u_i\bigr)_{i=1}^N$,\quad $u_i\approx J(h)(x_i)$};

\coordinate (Rtop)  at (\railx, -0.3);
\coordinate (Rbot)  at (\railx, -11.4);
\draw[blue!55, thick, dashed, ->]
  (Rtop) -- (Rbot)
  node[midway, left, font=\scriptsize\itshape, blue!55!black, align=center]
    {position\\[-1pt]coordinate\\[-1pt]$x_i$ preserved\\[-1pt]by $\pi_d$};
\foreach \src in {x1,x2,xN}{%
  \draw[blue!55, dashed, semithick]
    (\src.west) -- ++(-0.18,0) |- ($(Rtop)+(0,-0.05)$);
}
\foreach \dst in {ox1,ox2,oxN}{%
  \draw[blue!55, dashed, semithick]
    ($(Rbot)+(0,-0.05)$) -| ($(\dst.west)+(-0.18,0)$) -- (\dst.west);
}

\node[draw, dashed, rounded corners=2pt, fill=white,
      align=left, font=\scriptsize, inner sep=4pt]
      at (\cx+5.6,-2.2)
  {Eq.~\eqref{small g}:\\[1pt]
   $\overline g\;=\;{\bf j}_{x,\varepsilon}\circ \overline G\circ \iota_N$\\[1pt]
   Eq.~\eqref{sup G tran formula B}:\\[1pt]
   $\overline G\;\approx\;G^{\mathrm{tran}}$};

\end{tikzpicture}

\caption{Architecture of the function-graph transformer
$\overline g:X_d^N\to X_d^N$ defined by Eq.~\eqref{small g}.
The token sequence $\bigl(x_i,h(x_i)\bigr)_i$ is encoded by $\iota_N$
into the empirical measure $\mu_0$, processed by
$L$ residual softmax self-attention layers $\Gamma_{\theta_\ell}$
(see Eq.~\eqref{Gamma theta})
alternated with MLP layers $F_{\xi_\ell}$ to produce
$\mu_L=G^{\mathrm{tran}}(\mu_0)$ as in Eq.~\eqref{G tran formula},
and decoded back by ${\bf j}_{x,\varepsilon}\approx \iota_N^{-1}$ to the output token sequence
$\bigl(x_i,u_i\bigr)_i$ with $u_i\approx J(h)(x_i)$.
The position coordinate $x_i$ is preserved at every layer by the
projection $\pi_d$ (dashed blue rail on the left). The 
additional positional variable $x_i$ have the effect that width
of the softmax layers $\Gamma_{\theta_i}$ in the transformer (see Theorem \ref{thm:approximation-trans self-attention Takashi})
needs to be larger than in the earlier universal approximation results \cite{furuya2025transformers}.
}
\label{fig:gbar-architecture} 
\end{figure}

\begin{figure}[t]
\centering
\begin{tikzpicture}[
    >=Latex,
    font=\footnotesize,
    inp/.style    ={draw, rounded corners=1.5pt, fill=yellow!22, thick,
                    minimum height=0.6cm, minimum width=1.6cm, font=\small},
    measin/.style ={draw, rounded corners=1.5pt, fill=yellow!35, thick,
                    minimum height=0.6cm, minimum width=2.4cm, font=\small},
    pidblk/.style ={draw, thick, rounded corners=2pt, fill=blue!14,
                    minimum height=0.55cm, minimum width=1.6cm, font=\small},
    headblk/.style={draw, thick, rounded corners=2pt, fill=red!10,
                    minimum height=2.5cm, minimum width=8.4cm,
                    inner sep=5pt, align=center, font=\scriptsize},
    outblk/.style ={draw, rounded corners=2pt, fill=yellow!40, thick,
                    minimum height=0.75cm, minimum width=9.6cm, font=\small},
    arr/.style    ={->, thick},
    fork/.style   ={circle, fill=black, inner sep=0pt, minimum size=3pt},
]

\node[inp]    (z)  at (-3.6, 0) {$z=(x,v)$};
\node[measin] (mu) at ( 3.6, 0) {$\mu\in\mathcal{P}(\mathbb{R}^d{\times}\mathbb{R}^{d_\ell})$};

\node[fork] (zfork) at ($(z.south)+(0,-0.35)$) {};
\draw[thick] (z.south) -- (zfork);

\node[pidblk] (pix) at (-5.6, -2.0) {$\pi_d(z)=x$};
\draw[arr, rounded corners=3pt]
  (zfork) -| (pix.north);

\node[headblk] (heads) at (0, -2.0)
  {{\bfseries Multi-head softmax attention\ (residual)}\\[2pt]
   For each head $h=1,\dots,H$:\\[2pt]
   $\displaystyle
     \mathrm{Attn}^{h}(\mu,z)
     \;=\;\int_{\mathbb{R}^{d+d_\ell}}\!
     \frac{\exp\!\bigl(\tfrac{1}{\sqrt{k}}\langle Q^{h}z,K^{h}z'\rangle\bigr)}
          {\displaystyle\int\exp\!\bigl(\tfrac{1}{\sqrt{k}}\langle Q^{h}z,K^{h}z''\rangle\bigr)\,d\mu(z'')}
     \;V^{h}z'\,d\mu(z')$\\[6pt]
   Block output:\quad
   $\displaystyle v'\;=\;z\;+\;\sum_{h=1}^{H}W^{h}\,\mathrm{Attn}^{h}(\mu,z)$};

\draw[arr, rounded corners=3pt]
   (zfork) -- ($(zfork)+(0,-0.4)$)
   -| ($(heads.west)+(0,0.6)$);
\draw[arr, rounded corners=3pt]
   (mu.south) |- ($(heads.east)+(0,0.6)$);

\node[outblk] (outnode) at (0,-4.6)
  {$\Gamma_{\theta_\ell}(\mu,z)\;=\;
     \bigl(\,\pi_d(z)\,,\,
        \;z+\textstyle\sum_{h=1}^{H} W^{h}\,\mathrm{Attn}^{h}(\mu,z)\bigr)$};
\draw[arr] (heads.south) -- node[right, font=\scriptsize\itshape]{$v'$} (outnode.north);
\draw[arr, rounded corners=4pt]
  (pix.south) -- ($(pix.south)+(0,-2.3)$) -- (outnode.west);

\node[draw, dashed, rounded corners=2pt, fill=white,
      align=left, font=\scriptsize, inner sep=4pt]
      at (5.4, -5.7)
  {Analogously\\[1pt]
   $F_{\xi_\ell}(z)=(\pi_d(z),\,H_\ell(z))$,\\[1pt]
   where $H_\ell$ is an MLP.};

\end{tikzpicture}

\caption{Internal structure of one residual self-attention layer
$\Gamma_{\theta_\ell}$ from Eq.~\eqref{Gamma theta}.
The position coordinate $x=\pi_d(z)$ bypasses the attention block on
the left, while the feature coordinate $z$ is updated by the residual
sum of $H$ softmax attention heads attending over the empirical
measure $\mu$.  This guarantees the graph-preserving property
$\Gamma_{\theta_\ell}(\mu,z)\in\{x\}\times\mathbb{R}^{d_{\ell}'}$.
The MLP layer $F_{\xi_\ell}(z)=(\pi_d(z),H_\ell(z))$ has the same
position-preserving structure, so the entire composition
$\Gamma^{\mathrm{tran}}=F_{\xi_L}\diamond\Gamma_{\theta_L}\diamond
\cdots\diamond F_{\xi_1}\diamond\Gamma_{\theta_1}$ in
Eq.~\eqref{G tran formula} preserves graphs.}
\label{fig:gamma-theta-detail}
\end{figure}

\vfill
\newpage

$\ $

\vfill
\newpage

$\ $

\vfill
\newpage 

}


\subsection{Part II: The case when  the input and output grids are different}
\label{app:tokens and measures II}

Let us next consider a nonlinear operator $A:C(\overline \Omega;{[-L,L]}^n) \to C(\overline \Omega;{[-L,L]}^n)$,
where $ \Omega\subset \RR^{d}$. Let
$x_j\in \overline\Omega$, $j=1,2,\dots,J$ are computation grid points in the domain $\Omega$
where the input function is given and $z_k\in \overline\Omega$, $k=1,2,\dots,K$ are another 
computation grid points in the domain $\Omega$ where the output function as given.

When $h\in \hbox{cl}_{H^{-s}(\Omega;\R^n)}(C(\overline \Omega;{[-L,L]}^n))$, or alternatively,
$h\in C(\overline \Omega;{[-L,L]}^n)$, is the input function, we consider the regularized input function
$h_\tau=R_\tau(h)$,
$$
R_\tau h = \rho_\tau \ast (h \cdot (\rho_\tau \ast {\bf 1}_{\Omega_\tau})),
$$
where $\rho_\tau\in C^\infty_c(B(0,\tau))$ is a smooth mollifier function, see \eqref{def R tau}. These
point values determine  the token sequences of $N=J\cdot K$ tokens (that are tensorized)
\beq
\bigg((x_j,h_\tau(x_j));(z_k,0))\bigg)_{(j,k)\in [J]\times [K]}
\eeq
Observe that this sequence of has a large number, $N=J\cdot K$, of tokens.
However, even for linear maps we observe that if one wants to parametrize a general linear map (i.e., a matrix) $A_{lin}=(h_\tau(x_j))_{j\in [J]}\to (g(x_k)))_{k\in  [K]}$, one needs $N=J\cdot K$ matrix elements.


Let us also consider the function graph transformer  
$$
\overline G:\Pp((\overline \Omega\times [-L,L]^n)\times (\overline \Omega\times [-L,L]^n))\to \Pp((\overline \Omega\times [-L,L]^n)
\times (\overline \Omega\times [-L,L]^n)),
$$ 
see Definition \ref{def self-attention}, from a point of view of token sequences
 (when the domain deformation map $\Phi:\overline \Omega\to \overline \Omega$ 
in Def.\ \ref{def self-attention} is the identity map). 
When $\overline G$  maps $\gamma_h\otimes \nu$, where  $h\in C(\overline \Omega;{[-L,L]}^n)$ and $\nu=\tfrac 1K\sum_{k=1}^K \delta_{(z_k,0)}$, 
to  $$\overline G(\gamma_h\otimes \nu)=\gamma_h\otimes  \tfrac 1K\sum_{k=1}^K \delta_{(z_k,J(h)(z_k))},$$ where
${J(h)}\in C(\overline \Omega;{[-L,L]}^n)$ is approximately the function   
$A(h)\in C(\overline \Omega;{[-L,L]}^n)$. Also, we use the map $(\rho_2)_\#:\mu_1\otimes \nu_1 \to  \nu$.
Then, we define the sequence-to-sequence map between the token sequences,
\begin{equation}\label{large g eps}
\overline g={\bf j}_{\vec z,\varepsilon} \circ (\rho_2)_\# \circ \overline G\circ \iota:X_{d+n}^N\to X_{d+n}^K
\end{equation}
where ${\bf j}_{\vec z,\varepsilon}$ is a map similar to  \eqref{j x map}, but is associated to
$\vec z=(z_1,\dots,z_K)$, that is given by
\beq
{\bf j}_{\vec z,\varepsilon}:\gamma\to 
 \bigg((z_1,{\bf F}_\varepsilon \gamma(z_1))\, ,\, (z_2,{\bf F}_\varepsilon \gamma(z_2))\,,\dots,\,
 (z_K,{\bf F}_\varepsilon \gamma(z_K))\bigg),
\eeq
where ${\bf F}_\varepsilon$ is the convolution operator with a mollifier function,
see formula \eqref{f mollified 1B}. This gives a
map the form 

\begin{align}\label{small g2}
\overline g:\bigg((x_j,h(x_j));(z_k,0))\bigg)_{(j,k)\in [J]\times [K]}\to
 \bigg((z_1,t_1),(z_2,t_2),\dots, (z_K,t_K)\bigg)
\end{align}

where 
and $t_k$ is approximately equal to $J(h)(z_k)$, and our considerations
show that $t_k$ converges to $J(h)(z_k)$ when the grid points $z_k$ are suitably
chosen, $J,K\to \infty$ and $\varepsilon\to 0$.

\section{Regularized functions $R_\tau h$ converge to $h$.}
\label{app:prop:main convolutions A}

In this section we show the following result:

\begin{proposition}\label{prop:main convolutions A}
Let $s>0$ with $s-\tfrac12\notin\Z$.
Then for every $h\in H^{-s}(\Omega)$,
$
  \lim_{\tau\to 0^+}\,\bigl\|R_\tau h - h\bigr\|_{H^{-s}(\Omega)} = 0.
$
\end{proposition}

\subsection{Setting and notation}

Let $\Omega\subset\R^d$ be a bounded open set with $C^\infty$-smooth boundary, or a cube $[0,1]^d$,  , and let $d(x):=\dist(x,\partial\Omega)$.
For $s>0$ we use the convention
\[
  H^{s}_0(\Omega) := \overline{C_c^\infty(\Omega)}^{\,\|\cdot\|_{H^s(\Omega)}},\qquad
  H^{-s}(\Omega) := \bigl(H^{s}_0(\Omega)\bigr)^{\!*},
\]
where $\overline{C_c^\infty(\Omega)}^{\,\|\cdot\|_{H^s(\Omega)}}$ is the closure
of the set $C_c^\infty(\Omega)$ in ${H^s(\Omega)}$ and $H^{-s}(\Omega)$ is the topological dual of $H^s_0(\Omega)$, with duality pairing
$\langle \cdot,\cdot\rangle_{H^{-s},H^s_0}$ extending the $L^2(\Omega)$ inner product.

Fix $\rho\in C_c^\infty(\R^d)$ even, nonnegative, with $\int_{\R^d}\rho{d} x=1$ and
$\supp\rho\subset\overline{B(0,1)}$. Set
\[
  \rho_\tau(x):=\tau^{-d}\rho(x/\tau),\qquad
  \Omega_\tau:=\{x\in\Omega:d(x)>4\tau\},\qquad
  \chi_\tau:=\rho_\tau\ast \mathbf{1}_{\Omega_\tau},
\]
\[
  T_r:=\{x\in\R^d:d(x)<r\}.
\]
Let $\tau_0>0$ be small enough that $\Omega_{\tau_0}\neq\emptyset$. The regularization studied is
\[
  R_\tau h := \rho_\tau\ast\bigl(h\cdot\chi_\tau\bigr),\qquad h\in H^{-s}(\Omega),\ \tau\in(0,\tau_0].
\]
Observe that since $\rho_\tau$ is non-negative and its integral is 1, we have for $h\in C(\Omega)$
\beq\label{R tau L infty}
\|  R_\tau h\|_{L^\infty(\Omega)}\leq \| h\|_{L^\infty(\Omega)}.
\eeq

Since $\chi_\tau\in C_c^\infty(\Omega)$ (Lemma~\ref{lem:chitau} below), the product $h\chi_\tau$ is
defined as the distribution (i.e., a generalized function)
\[
  \langle h\chi_\tau,\psi\rangle_{\mathcal{D}'(\R^d),\mathcal{D}(\R^d)}
  :=\langle h,\chi_\tau\psi\rangle_{H^{-s},H^s_0}\qquad(\psi\in\mathcal{D}(\R^d)),
\]
which lies in the space of compactly supported distributions $\mathcal{E}'(\R^d)$ with support in $\supp\chi_\tau\Subset\Omega$. Convolution with
$\rho_\tau$ thus yields $R_\tau h\in C_c^\infty(\R^d)$ with support in $\overline\Omega$.

\subsection{Approximation result via convolutions}

\begin{proposition}\label{prop:main convolutions}
Let $\Omega\subset\R^d$ be an open bounded set with $C^\infty$-smooth boundary, or a cube $[0,1]^d$. Let $s>0$ with $s-\tfrac12\notin\Z$.
Then for every $h\in H^{-s}(\Omega)$,
\[
  \lim_{\tau\to 0^+}\,\bigl\|R_\tau h - h\bigr\|_{H^{-s}(\Omega)} = 0.
\]
\end{proposition}

The proof proceeds via Banach--Steinhaus (uniform boundedness) plus a density argument. We first collect the necessary lemmas.

\subsection*{Auxiliary lemmas}

\begin{lemma}[Geometry of $\chi_\tau$]\label{lem:chitau}
For all $\tau\in(0,\tau_0]$:
\begin{enumerate}
\item[\textnormal{(i)}] $\chi_\tau\in C^\infty(\R^d)$, $0\le\chi_\tau\le 1$;
\item[\textnormal{(ii)}] $\supp\chi_\tau\subset \{x:d(x)\ge\tau\}\subset\Omega$;
\item[\textnormal{(iii)}] $\chi_\tau(x)=1$ whenever $d(x)\ge 5\tau$;
\item[\textnormal{(iv)}] $\|\partial^\alpha\chi_\tau\|_{L^\infty(\R^d)}\le C_\alpha\,\tau^{-|\alpha|}$ for every multi-index $\alpha$, with $C_\alpha=\|\partial^\alpha\rho\|_{L^1(\R^d)}$;
\item[\textnormal{(v)}] for every multi-index $\alpha$ with $|\alpha|\ge 1$, $\supp(\partial^\alpha\chi_\tau)\subset T_{5\tau}$.
\end{enumerate}
\end{lemma}

\begin{proof}
(i) $\chi_\tau$ is the convolution of an $L^\infty$ indicator with a $C_c^\infty$ kernel, hence
$C^\infty$; the bound $0\le\chi_\tau\le 1$ follows from $0\le\mathbf{1}_{\Omega_\tau}\le 1$ and
$\int\rho_\tau dx=1$.

(ii) If $\chi_\tau(x)\neq 0$, there exists $y\in\Omega_\tau$ with $|x-y|<\tau$; then
$d(x)\ge d(y)-|x-y|>4\tau-\tau=3\tau$.

(iii) If $d(x)\ge 5\tau$, then for every $|x-y|\le\tau$ one has $d(y)\ge 5\tau-\tau=4\tau$, so
$B(x,\tau)\subset\Omega_\tau$ and $\chi_\tau(x)=\int_{B(x,\tau)}\rho_\tau(x-y){d} y=1$.

(iv) $\partial^\alpha\chi_\tau=(\partial^\alpha\rho_\tau)\ast\mathbf{1}_{\Omega_\tau}$, so by Young's inequality,
$\|\partial^\alpha\chi_\tau\|_\infty\le\|\partial^\alpha\rho_\tau\|_{L^1(\R^d)}=\tau^{-|\alpha|}\|\partial^\alpha\rho\|_{L^1(\R^d)}$.

(v) On the open set $\{d>5\tau\}$ one has $\chi_\tau\equiv 1$ by (iii); on the open set
$\R^d\setminus\{d\ge\tau\}$ one has $\chi_\tau\equiv 0$ by (ii). Hence every derivative of
$\chi_\tau$ of order $\alpha$, $|\alpha|\ge 1$ vanishes there, so its support lies in
$\{d\ge\tau\}\setminus\{d>5\tau\}\subset T_{5\tau}$.
\qed\end{proof}

\begin{lemma}[Density of $C_c^\infty(\Omega)$ in $H^{-s}(\Omega)$]\label{lem:density}
For every $s>0$, $C_c^\infty(\Omega)$ is dense in $H^{-s}(\Omega)$.
\end{lemma}

\begin{proof}
The inclusion $H^s_0(\Omega)\hookrightarrow L^2(\Omega)$ is continuous and dense
($C_c^\infty(\Omega)$ is dense in both). Taking duals (of a continuous, dense embedding of Hilbert
spaces) gives a continuous embedding $L^2(\Omega)\hookrightarrow H^{-s}(\Omega)$ with dense image.
Since $C_c^\infty(\Omega)$ is dense in $L^2(\Omega)$, it is dense in $H^{-s}(\Omega)$.
\qed\end{proof}

\begin{lemma}[Adjoint identity]\label{lem:adjoint}
Define, for $\varphi\in H^s_0(\Omega)$ with zero extension $\widetilde\varphi\in H^s(\R^d)$,
\[
  A_\tau\varphi := \chi_\tau\cdot(\rho_\tau\ast\widetilde\varphi).
\]
Then $A_\tau\varphi\in C_c^\infty(\Omega)$, and for every $h\in H^{-s}(\Omega)$,
\[
  \langle R_\tau h,\varphi\rangle_{H^{-s},H^s_0} = \langle h,A_\tau\varphi\rangle_{H^{-s},H^s_0}.
\]
In particular,
\beq\label{equality of norms}
  \|R_\tau\|_{H^{-s}(\Omega)\to H^{-s}(\Omega)}
  = \|A_\tau\|_{H^s_0(\Omega)\to H^s_0(\Omega)}.
\eeq
\end{lemma}

\begin{proof}
The function $\rho_\tau\ast\widetilde\varphi$ is $C^\infty$ on $\R^d$, and $\chi_\tau\in C_c^\infty(\Omega)$,
so $A_\tau\varphi\in C_c^\infty(\Omega)\subset H^s_0(\Omega)$.

Next we use the duality: It holds that $R_\tau h\in C_c^\infty(\R^d;\R^n)\subset L^2(\Omega;\R^n)$, so
\[
  \langle R_\tau h,\varphi\rangle = \int_\Omega R_\tau h\,\varphi{d} x
  = \int_{\R^d}(\rho_\tau\ast u)(x)\,\widetilde\varphi(x){d} x,\qquad u:=h\chi_\tau\in\mathcal{E}'(\R^d).
\]
Because $\rho_\tau$ is even, the standard adjoint identity for convolution gives
\[
  \int_{\R^d}(\rho_\tau\ast u)\,\widetilde\varphi{d} x
  = \langle u,\rho_\tau\ast\widetilde\varphi\rangle_{\mathcal{E}',C^\infty},
\]
which can be checked for test functions and extended by approximation since $u$ has compact support.
Then by the definition of $h\chi_\tau$ as a distribution,
\[
  \langle u,\rho_\tau\ast\widetilde\varphi\rangle = \langle h,\chi_\tau\cdot(\rho_\tau\ast\widetilde\varphi)\rangle_{H^{-s},H^s_0}
  = \langle h,A_\tau\varphi\rangle.
\]
Identifying $H^{-s}(\Omega)$ with $(H^s_0(\Omega))^{\!*}$ shows that $R_\tau$ has $A_\tau$ as
its Banach-space adjoint, hence we obtain the identity \eqref{equality of norms}.
\qed\end{proof}

\begin{lemma}[Iterated Hardy inequality on a Lipschitz domain]\label{lem:hardy}
Let $\Omega\subset\R^d$ be a bounded Lipschitz domain and $j\in\Z_+$. There exists
$C=C(\Omega,j)$ such that
\[
  \int_\Omega \frac{|g(x)|^2}{d(x)^{2j}}{d} x \le C\,\|g\|_{H^j(\Omega)}^2 \qquad\text{for all }
  g\in H^j_0(\Omega).
\]
\end{lemma}

\begin{proof} The claim follows from references  on Sobolev spaces:
For $j=1$ the claim  follows from the Hardy inequality on Lipschitz domains; see
\cite[Thm.~1.4.4.4]{Grisvard} or \cite[Thm.~21.3]{Ne}. For higher $j$ one iterates the argument: the 
operator $g(x)\mapsto \tfrac 1 {d(x)} g(x)$ maps $H^j_0(\Omega)$ continuously into $H^{j-1}_0(\Omega)$ by
\cite[Thm.~1.4.4.4]{Grisvard}, and the result
follows. See also \cite[Thm.~4.36]{AdamsFournier} or \cite{BrezisMarcus}.
\qed\end{proof}

\begin{lemma}[Boundary-layer mollification estimate]\label{lem:bdy}
Let $j\in\Z_+\cup\{0\}$. There exists $C=C(\Omega,\rho,j)$ such that for every
$g\in H^j_0(\Omega)$ (with the convention $H^0_0(\Omega):=L^2(\Omega)$) and every $\tau\in(0,\tau_0]$,
\[
  \bigl\|\rho_\tau\ast\widetilde g\bigr\|_{L^2(T_{5\tau})}
  \le C\,\tau^{j}\,\|g\|_{H^j(\Omega)},
\]
where $\widetilde g$ denotes the zero extension of $g$ to $\R^d$.
\end{lemma}

\begin{proof}
By the Cauchy--Schwarz inequality applied with the probability measure $\rho_\tau(x-\cdot){d} y$,
\[
  \bigl|(\rho_\tau\ast\widetilde g)(x)\bigr|^2
  \le 1\cdot \int_{\R^d}\rho_\tau(x-y)\,|\widetilde g(y)|^2{d} y
  =\bigl(\rho_\tau\ast |\widetilde g|^2\bigr)(x).
\]
Integrating over $T_{5\tau}$ and applying Fubini,
\[
  \bigl\|\rho_\tau\ast\widetilde g\bigr\|_{L^2(T_{5\tau})}^2
  \le \int_{\R^d}|\widetilde g(y)|^2\Bigl(\int_{T_{5\tau}}\rho_\tau(x-y){d} x\Bigr){d} y.
\]
The inner integral is bounded by $\int_{\R^d}\rho_\tau(x-y){d} x=1$, and is nonzero only when
$\dist(y,T_{5\tau})\le\tau$, i.e.\ $y\in T_{6\tau}$. Combined with $\widetilde g=0$ outside
$\overline\Omega$,
\[
  \bigl\|\rho_\tau\ast\widetilde g\bigr\|_{L^2(T_{5\tau})}^2
  \le \int_{T_{6\tau}\cap\Omega}|g(y)|^2{d} y.
\]
For $j=0$ this gives $\|\rho_\tau\ast\widetilde g\|_{L^2(T_{5\tau})}\le\|g\|_{L^2(\Omega)}$,
i.e.\ the claim with $\tau^0=1$.
For $j\ge 1$ we apply Lemma~\ref{lem:hardy}:
\[
  \int_{T_{6\tau}\cap\Omega}|g|^2{d} y
  = \int_{T_{6\tau}\cap\Omega}\frac{|g|^2}{d^{2j}}\,d^{2j}{d} y
  \le (6\tau)^{2j}\int_\Omega\frac{|g|^2}{d^{2j}}{d} y
  \le C(6\tau)^{2j}\,\|g\|_{H^j(\Omega)}^2.
\]
Taking square roots of both sides of the above inequality completes the proof.
\qed\end{proof}

\begin{lemma}[Uniform boundedness of $A_\tau$ at integer order]\label{lem:UBint}
For each $k\in\mathbb Z_+\cup\{0\}$ there exists $C=C(\Omega,\rho,k)$ such that
\[
  \|A_\tau\varphi\|_{H^k(\Omega)}\le C\,\|\varphi\|_{H^k(\Omega)}\qquad
  \text{for all }\varphi\in H^k_0(\Omega),\ \tau\in(0,\tau_0].
\]
\end{lemma}

\begin{proof}
The case $k=0$: $\|A_\tau\varphi\|_{L^2}=\|\chi_\tau(\rho_\tau\ast\widetilde\varphi)\|_{L^2}\le\|\rho_\tau\ast\widetilde\varphi\|_{L^2}\le\|\widetilde\varphi\|_{L^2}=\|\varphi\|_{L^2(\Omega)}$, using
$\|\chi_\tau\|_\infty\le1$ and Young's inequality, as mollification is an $L^2$-contraction.

Assume $k\ge 1$. Recall $A_\tau\varphi\in C_c^\infty(\Omega)$, so its zero extension to $\R^d$ is itself
a $C_c^\infty(\R^d)$ function and
\[
  \|A_\tau\varphi\|_{H^k(\Omega)} = \|A_\tau\varphi\|_{H^k(\R^d)}
  =\Bigl(\sum_{|\alpha|\le k}\|\partial^\alpha A_\tau\varphi\|_{L^2(\R^d)}^2\Bigr)^{1/2}.
\]
By Leibniz rule, for $|\alpha|\le k$
\[
  \partial^\alpha A_\tau\varphi
  = \sum_{\beta\le\alpha}\binom{\alpha}{\beta}\,(\partial^\beta\chi_\tau)\cdot \partial^{\alpha-\beta}(\rho_\tau\ast\widetilde\varphi).
\]
We use that for $\beta\le\alpha$ and $\varphi\in H^k_0(\Omega)$, $\partial^{\alpha-\beta}\varphi\in H^{k-|\alpha-\beta|}_0(\Omega)$
and $\partial^{\alpha-\beta}\widetilde\varphi=\widetilde{\partial^{\alpha-\beta}\varphi}$ in
$\mathcal{D}'(\R^d)$, where
$\widetilde{\partial^{\alpha-\beta}\varphi}$ is the zero-extension of ${\partial^{\alpha-\beta}\varphi}$
to $\R^d$. Consequently,
$\partial^{\alpha-\beta}(\rho_\tau\ast\widetilde\varphi)=\rho_\tau\ast\widetilde{\partial^{\alpha-\beta}\varphi}$.

\medskip
\noindent\emph{Term $\beta=0$.} Using $\|\chi_\tau\|_\infty\le1$ and Young's inequality,
\[
  \bigl\|\chi_\tau\cdot\rho_\tau\ast\widetilde{\partial^\alpha\varphi}\bigr\|_{L^2}
  \le\|\partial^\alpha\varphi\|_{L^2(\Omega)}\le\|\varphi\|_{H^k(\Omega)}.
\]

\medskip
\noindent\emph{Term $|\beta|\ge 1$.} By Lemma~\ref{lem:chitau}(iv,v), $\partial^\beta\chi_\tau$ is
supported in $T_{5\tau}$ with $\|\partial^\beta\chi_\tau\|_\infty\le C_\beta\tau^{-|\beta|}$, so
\beq\label{line1}
  \bigl\|(\partial^\beta\chi_\tau)\cdot\rho_\tau\ast\widetilde{\partial^{\alpha-\beta}\varphi}\bigr\|_{L^2}
  \le C_\beta\,\tau^{-|\beta|}\,\bigl\|\rho_\tau\ast\widetilde{\partial^{\alpha-\beta}\varphi}\bigr\|_{L^2(T_{5\tau})}.
\eeq
Set $g:=\partial^{\alpha-\beta}\varphi\in H^{k-|\alpha-\beta|}_0(\Omega)$. Since
$|\alpha|\le k$ implies $k-|\alpha-\beta|=k-|\alpha|+|\beta|\ge|\beta|$, we have
$g\in H^{|\beta|}_0(\Omega)$ with $\|g\|_{H^{|\beta|}}\le\|\varphi\|_{H^k}$.
Lemma~\ref{lem:bdy} with $j=|\beta|$ gives
\beq\label{line2}
  \bigl\|\rho_\tau\ast\widetilde g\bigr\|_{L^2(T_{5\tau})}\le C\,\tau^{|\beta|}\,\|g\|_{H^{|\beta|}(\Omega)}\le C\,\tau^{|\beta|}\,\|\varphi\|_{H^k(\Omega)}.
\eeq
Combining the formulas \eqref{line1} and \eqref{line2}, we obtain
\[
  \bigl\|(\partial^\beta\chi_\tau)\cdot\rho_\tau\ast\widetilde{\partial^{\alpha-\beta}\varphi}\bigr\|_{L^2}
  \le C\,\|\varphi\|_{H^k(\Omega)},
\]
with $C$ independent of $\tau$ (the powers $\tau^{-|\beta|}$ and $\tau^{|\beta|}$ cancel exactly).

\medskip
Summing over $\beta\le\alpha$ and over $|\alpha|\le k$ yields
$\|A_\tau\varphi\|_{H^k(\R^d)}\le C(k)\,\|\varphi\|_{H^k(\Omega)}$, which is the claim because
$\|A_\tau\varphi\|_{H^k(\Omega)}=\|A_\tau\varphi\|_{H^k(\R^d)}$.
\qed\end{proof}

\begin{lemma}[Uniform boundedness of $A_\tau$ at fractional order]\label{lem:UBfrac}
Let $s>0$ with $s-\tfrac12\notin\Z$. There exists $C=C(\Omega,\rho,s)$ such that
\[
  \|A_\tau\varphi\|_{H^s_0(\Omega)}\le C\,\|\varphi\|_{H^s_0(\Omega)}\qquad
  \text{for all }\varphi\in H^s_0(\Omega),\ \tau\in(0,\tau_0].
\]
Equivalently, $\sup_{\tau\in(0,\tau_0]}\|R_\tau\|_{H^{-s}(\Omega)\to H^{-s}(\Omega)}<\infty$.
\end{lemma}

\begin{proof}
If $s\in{\Z_+}$, the claim is Lemma~\ref{lem:UBint}. Assume $s\notin{\Z_+}$ and choose $N:=\lceil s\rceil\in{\Z_+}$,
so that $\theta:=s/N\in(0,1)$. By Lemma~\ref{lem:UBint}, the linear maps
$A_\tau:L^2(\Omega)\to L^2(\Omega)$ and $A_\tau:H^N_0(\Omega)\to H^N_0(\Omega)$ are bounded with
norms uniformly bounded by some $M_0,M_N$ independent of $\tau\in(0,\tau_0]$.

The standard interpolation identity for the spaces $H^s_0$ on a Lipschitz domain reads
\[
  \bigl[L^2(\Omega),\,H^N_0(\Omega)\bigr]_\theta
  = H^{\theta N}_0(\Omega) = H^s_0(\Omega),
\]
provided that  $s=\theta N\notin \tfrac12+\Z$, i.e.\ $s-\tfrac12\notin\Z$, 
see \cite[Thm.~B.8 and Thm.~3.33]{McLean}
or \cite[Ch.~1, \S 9 and Thm.~11.6]{LionsMagenes1}.

The complex interpolation theorem for bounded linear operators
\cite[Thm.~B.2(iii)]{McLean} then yields
\[
  \|A_\tau\|_{H^s_0(\Omega)\to H^s_0(\Omega)} \le M_0^{1-\theta}M_N^{\theta},
\]
which is uniform in $\tau\in(0,\tau_0]$. The equivalent statement on $R_\tau$ follows from
Lemma~\ref{lem:adjoint}.
\qed\end{proof}

\begin{remark}[On the integer case]
For integer $s\in{\Z_+}$ no interpolation is needed: Lemma~\ref{lem:UBint} gives uniform
boundedness directly, so the conclusion of Proposition~\ref{prop:main convolutions} holds for every
$s\in{\Z_+}$ (which all satisfy $s-\tfrac12\notin\Z$).
\end{remark}

\begin{lemma}\label{lem:smooth}
For every $h\in C_c^\infty(\Omega)$,
\[
  \lim_{\tau\to 0^+}\|R_\tau h - h\|_{H^{-s}(\Omega)} = 0.
\]
\end{lemma}

\begin{proof}
Set $\delta:=\dist(\supp h,\partial\Omega)>0$. For $\tau<\delta/5$, Lemma~\ref{lem:chitau}(iii)
gives $\chi_\tau\equiv 1$ on $\supp (h)$, hence $h\chi_\tau=h$ and $R_\tau h=\rho_\tau\ast h$.
Standard properties of mollifiers yield $\rho_\tau\ast h\to h$ in $C_c^\infty(\R^d)$,
and in particular in $L^2(\Omega)$.
The continuous embedding $L^2(\Omega)\hookrightarrow H^{-s}(\Omega)$ (which is the dual of the
continuous embedding $H^s_0(\Omega)\hookrightarrow L^2(\Omega)$) yields convergence in
$H^{-s}(\Omega)$.
\qed\end{proof}

\begin{proof}[of Proof of Proposition~\ref{prop:main convolutions}]
Let $\eps>0$ and $h\in H^{-s}(\Omega)$. By Lemma~\ref{lem:density} there exists
$h_\eps\in C_c^\infty(\Omega)$ with $\|h-h_\eps\|_{H^{-s}(\Omega)}<\eps$. Let us decompose functions as
\[
  R_\tau h - h = R_\tau(h-h_\eps) + (R_\tau h_\eps - h_\eps) - (h-h_\eps).
\]
By Lemma~\ref{lem:UBfrac} there is $C^*$ independent of $\tau\in(0,\tau_0]$ such that
\[
  \|R_\tau(h-h_\eps)\|_{H^{-s}(\Omega)}\le C^*\,\|h-h_\eps\|_{H^{-s}(\Omega)}<C^*\eps.
\]
By above assumption on $h_\eps$, we have $\|h-h_\eps\|_{H^{-s}(\Omega)}<\eps$. By Lemma~\ref{lem:smooth} there exists
$\tau_\eps\in(0,\tau_0]$ such that
\[
  \|R_\tau h_\eps - h_\eps\|_{H^{-s}(\Omega)}<\eps\qquad\text{for all } \tau\in(0,\tau_\eps).
\]
Combining,
\[
  \|R_\tau h - h\|_{H^{-s}(\Omega)} \le (C^*+2)\,\eps \qquad\text{for all } \tau\in(0,\tau_\eps).
\]
Since $\eps>0$ was arbitrary, $\|R_\tau h-h\|_{H^{-s}(\Omega)}\to 0$ as $\tau\to 0^+$.
\qed\end{proof}

\commented{
\begin{remark}[Where the hypothesis $s-\tfrac12\notin\Z$ is used]
The hypothesis enters \emph{only} in Lemma~\ref{lem:UBfrac}, through the interpolation identity
$[L^2(\Omega),H^N_0(\Omega)]_\theta=H^{\theta N}_0(\Omega)$. At the exceptional values
$s\in\tfrac12+\N$, the right-hand side must be replaced by the Lions--Magenes space, with
strictly stronger norm; the conclusion of Proposition~\ref{prop:main convolutions} can still hold but the
proof requires either that intermediate space or a separate Hardy estimate adapted to the
half-integer scale.
\end{remark}}

\section{Proof of Lemma \ref{lem: function graph transformers architecture}}
\label{app:lem: function graph transformers architecture}

Let us recall the setting.
Let $\Omega\subset\mathbb R^{d}$ be a non-empty bounded open set, and set
\[
    \overline\Omega = \operatorname{cl}(\Omega).
\]
Let $L>0$, $n\in\mathbb Z_+$, and define
\[
    X \coloneqq \overline\Omega\times[-L,L]^n
    \subset \mathbb R^{d+n}.
\]
The space $X$ is equipped with the Euclidean metric inherited from
$\mathbb R^{d+n}$. Let $\overline \lambda_{\overline\Omega}=\tfrac 1{\text{vol}(\Omega)} \lambda_{\overline\Omega}$ denote the normalized
Lebesgue measure on $\overline\Omega$. Thus
$\operatorname{spt}\lambda_{\overline\Omega}=\overline\Omega$.

For $h\in C(\overline\Omega;[-L,L]^n)$ define its graph measure by
\[
    \gamma_h
    \coloneqq
    (\operatorname{id}_{\overline\Omega},h)_{\#}\overline \lambda_{\Omega}
    \in \mathcal P(X).
\]
Set
\[
    \mathcal G
    \coloneqq
    \{\gamma_h:h\in C(\overline\Omega;[-L,L]^n)\}.
\]
Since $X$ is compact, $\mathcal P(X)=\mathcal P_2(X)$, and we equip
$\mathcal P(X)$ with the $1$-Wasserstein distance ${\color{black}W_1}$ induced by the
Euclidean metric on $X$.

Let
\[
    {\Gamma}:\mathcal P(X)\times X\to X
\]
be continuous, and assume:

\[
\tag{P}
    \forall h\in C(\overline\Omega;[-L,L]^n)\ \exists g\in
    C(\overline\Omega;[-L,L]^n)
    \quad\quad\hbox{such that }
    {\Gamma}(\gamma_h,\cdot)_{\#}\gamma_h=\gamma_g.
\]

\[
\tag{L1}
    \|J(h_1)-J(h_2)\|_{C(\overline\Omega;\mathbb R^n)}
    \le
    L_0\,{\color{black}W_1}(\gamma_{h_1},\gamma_{h_2})
\]
for all $h_1,h_2\in C(\overline\Omega;[-L,L]^n)$.

\[
\tag{L2$_{\mathrm{lin}}$}
    |J(h)(x_1)-J(h)(x_2)|_{\mathbb R^n}
    \le
    L_1 |x_1-x_2|
\]
for all $h\in C(\overline\Omega;[-L,L]^n)$ and all
$x_1,x_2\in\overline\Omega$.

Here $J$ is the graph-update map determined by \textup{(P)}; its
well-definedness is proved below.

\begin{proposition}
Assume \textup{(P)}, \textup{(L1)}, and
\textup{(L2$_{\mathrm{lin}}$)}. Then there exists a map
\[
    \widehat p:\mathcal P(X)\times\overline\Omega\to[-L,L]^n
\]
such that
\[
    \widehat p(\gamma_h,x)=J(h)(x)
    \qquad
    \forall h\in C(\overline\Omega;[-L,L]^n),\quad
    \forall x\in\overline\Omega,
\]
and
\begin{align}
\label{Tagged 1}
    |\widehat p(\mu_1,x_1)-\widehat p(\mu_2,x_2)|_{\mathbb R^n}
    \le
    \sqrt n\,L_0\,{\color{black}W_1}(\mu_1,\mu_2)
    +
    \sqrt n\,L_1\,|x_1-x_2|
\end{align}
for all $\mu_1,\mu_2\in\mathcal P(X)$ and all
$x_1,x_2\in\overline\Omega$.

Define
\[
    \widehat {\Gamma}:\mathcal P(X)\times X\to X,
    \qquad
    \widehat {\Gamma}(\mu,(x,y))
    \coloneqq
    (x,\widehat p(\mu,x)).
\]
Then $\widehat f$ is Lipschitz with respect to the product metric
\[
    d\bigl((\mu,z),(\nu,z')\bigr)
    \coloneqq
    {\color{black}W_1}(\mu,\nu)+|z-z'|_{\mathbb R^{d+n}},
    \qquad z,z'\in X,
\]
and
\[
    {\Gamma}(\gamma_h,\cdot)_{\#}\gamma_h
    =
    \widehat {\Gamma}(\gamma_h,\cdot)_{\#}\gamma_h
    =
    \gamma_{J(h)}
\]
for every $h\in C(\overline\Omega;[-L,L]^n)$.
\end{proposition}

\begin{proof}
We split the proof into three steps.

\medskip

\noindent
\textbf{Step 1: Well-definedness of the graph-update map.}

We first show that the function $g$ in \textup{(P)} is unique.

Suppose
\[
    \gamma_{g_1}=\gamma_{g_2}
\]
for some $g_1,g_2\in C(\overline\Omega;[-L,L]^n)$. Consider the continuous
function
\[
    \Phi:X\to\mathbb R,
    \qquad
    \Phi(x,y)\coloneqq |y-g_1(x)|_{\mathbb R^n}^2.
\]
Since $\gamma_{g_1}=\gamma_{g_2}$, we have
\[
    \int_X \Phi\,d\gamma_{g_1}
    =
    \int_X \Phi\,d\gamma_{g_2}.
\]
By the definition of graph measures,
\[
    \int_X \Phi\,d\gamma_{g_1}
    =
    \int_{\overline\Omega}
        |g_1(x)-g_1(x)|_{\mathbb R^n}^2
    \,d\lambda_{\overline\Omega}(x)
    =
    0,
\]
whereas
\[
    \int_X \Phi\,d\gamma_{g_2}
    =
    \int_{\overline\Omega}
        |g_2(x)-g_1(x)|_{\mathbb R^n}^2
    \,d\lambda_{\overline\Omega}(x).
\]
Therefore
\[
    \int_{\overline\Omega}
        |g_2(x)-g_1(x)|_{\mathbb R^n}^2
    \,d\lambda_{\overline\Omega}(x)
    =
    0.
\]
Hence $g_1=g_2$ $\lambda_{\overline\Omega}$-almost everywhere.

Since $g_1-g_2$ is continuous and
$\operatorname{supp}(\lambda_{\overline\Omega})=\overline\Omega$, it follows
that
\[
    g_1=g_2
    \qquad
    \text{on } \overline\Omega.
\]
Thus the graph update in \textup{(P)} is uniquely determined by $h$, and
we may define
\[
    J(h)\coloneqq g,
    \qquad
    {\Gamma}(\gamma_h,\cdot)_{\#}\gamma_h=\gamma_{J(h)}.
\]

\medskip

\noindent
\textbf{Step 2: Scalar McShane extensions.}

Fix $x\in\overline\Omega$ and $i\in\{1,\dots,n\}$. Define
\[
    \psi_{x,i}:\mathcal G\to[-L,L],
    \qquad
    \psi_{x,i}(\gamma_h)\coloneqq J(h)(x)_i.
\]
This is well-defined because $h\mapsto\gamma_h$ is injective by Step 1
applied to $g_1=h_1$ and $g_2=h_2$.

By \textup{(L1)}, for all $h_1,h_2\in C(\overline\Omega;[-L,L]^n)$,
\[
\begin{aligned}
    |\psi_{x,i}(\gamma_{h_1})-\psi_{x,i}(\gamma_{h_2})|
    &=
    |J(h_1)(x)_i-J(h_2)(x)_i|        \\
    &\le
    |J(h_1)(x)-J(h_2)(x)|_{\mathbb R^n} \\
    &\le
    \|J(h_1)-J(h_2)\|_{C(\overline\Omega;\mathbb R^n)} \\
    &\le
    L_0 {\color{black}W_1}(\gamma_{h_1},\gamma_{h_2}).
\end{aligned}
\]
Thus $\psi_{x,i}$ is $L_0$-Lipschitz on $(\mathcal G,{\color{black}W_1})$.

Define its McShane extension by
\begin{align}
\label{Tagged 2}
    \widetilde\psi_{x,i}(\mu)
    \coloneqq
    \inf_{h\in C(\overline\Omega;[-L,L]^n)}
    \bigl\{
        J(h)(x)_i
        +
        L_0 {\color{black}W_1}(\mu,\gamma_h)
    \bigr\},
    \qquad
    \mu\in\mathcal P(X).
\end{align}

We claim that
\begin{align}
\label{Tagged 3}
    \widetilde\psi_{x,i}(\gamma_h)=J(h)(x)_i
\end{align}
for every $h\in C(\overline\Omega;[-L,L]^n)$.

Indeed, from the choice $h$ in the infimum in \eqref{Tagged 2},
\[
    \widetilde\psi_{x,i}(\gamma_h)
    \le
    J(h)(x)_i.
\]
Conversely, for arbitrary $k\in C(\overline\Omega;[-L,L]^n)$, the
$L_0$-Lipschitz property of $\psi_{x,i}$ gives
\[
    J(k)(x)_i
    \ge
    J(h)(x)_i
    -
    L_0 {\color{black}W_1}(\gamma_h,\gamma_k).
\]
Therefore
\[
    J(k)(x)_i
    +
    L_0 {\color{black}W_1}(\gamma_h,\gamma_k)
    \ge
    J(h)(x)_i.
\]
Taking the infimum over $k$ gives
\[
    \widetilde\psi_{x,i}(\gamma_h)
    \ge
    J(h)(x)_i.
\]
Hence \eqref{Tagged 3} holds.

Next, for arbitrary $\mu_1,\mu_2\in\mathcal P(X)$ and arbitrary
$h\in C(\overline\Omega;[-L,L]^n)$, the triangle inequality for ${\color{black}W_1}$
implies
\[
    {\color{black}W_1}(\mu_1,\gamma_h)
    \le
    {\color{black}W_1}(\mu_1,\mu_2)+{\color{black}W_1}(\mu_2,\gamma_h).
\]
Thus
\[
    J(h)(x)_i+L_0{\color{black}W_1}(\mu_1,\gamma_h)
    \le
    L_0{\color{black}W_1}(\mu_1,\mu_2)
    +
    J(h)(x)_i+L_0{\color{black}W_1}(\mu_2,\gamma_h).
\]
Taking the infimum over $h$ gives
\[
    \widetilde\psi_{x,i}(\mu_1)
    \le
    \widetilde\psi_{x,i}(\mu_2)
    +
    L_0{\color{black}W_1}(\mu_1,\mu_2).
\]
Exchanging $\mu_1$ and $\mu_2$, we obtain
\begin{align}
\label{Tagged 4}
    |\widetilde\psi_{x,i}(\mu_1)-\widetilde\psi_{x,i}(\mu_2)|
    \le
    L_0{\color{black}W_1}(\mu_1,\mu_2).
\end{align}

We now prove Lipschitz regularity in the $x$-variable. Fix
$\mu\in\mathcal P(X)$, $x_1,x_2\in\overline\Omega$, and
$\varepsilon>0$. By the definition of the infimum, choose
$h_\varepsilon\in C(\overline\Omega;[-L,L]^n)$ such that
\[
    J(h_\varepsilon)(x_2)_i
    +
    L_0{\color{black}W_1}(\mu,\gamma_{h_\varepsilon})
    \le
    \widetilde\psi_{x_2,i}(\mu)+\varepsilon.
\]
Using the same function $h_\varepsilon$ as a competitor in the infimum
defining $\widetilde\psi_{x_1,i}(\mu)$, we get
\[
    \widetilde\psi_{x_1,i}(\mu)
    \le
    J(h_\varepsilon)(x_1)_i
    +
    L_0{\color{black}W_1}(\mu,\gamma_{h_\varepsilon}).
\]
Hence
\[
\begin{aligned}
    \widetilde\psi_{x_1,i}(\mu)
    -
    \widetilde\psi_{x_2,i}(\mu)
    &\le
    J(h_\varepsilon)(x_1)_i
    -
    J(h_\varepsilon)(x_2)_i
    +
    \varepsilon                                      \\
    &\le
    |J(h_\varepsilon)(x_1)-J(h_\varepsilon)(x_2)|_{\mathbb R^n}
    +
    \varepsilon                                      \\
    &\le
    L_1|x_1-x_2|+\varepsilon,
\end{aligned}
\]
where we used \textup{(L2$_{\mathrm{lin}}$)}. Exchanging $x_1$ and
$x_2$, and then letting $\varepsilon\downarrow0$, gives
\begin{align}
\label{Tagged 5}
    |\widetilde\psi_{x_1,i}(\mu)-\widetilde\psi_{x_2,i}(\mu)|
    \le
    L_1|x_1-x_2|.
\end{align}

\medskip

\noindent
\textbf{Step 3: Construction of $\widehat p$ and $\widehat f$.}

Let
\[
    T_L:\mathbb R\to[-L,L],
    \qquad
    T_L(r)\coloneqq \max\{-L,\min\{L,r\}\}.
\]
The map $T_L$ is $1$-Lipschitz. Define
\[
    \widehat p(\mu,x)_i
    \coloneqq
    T_L\bigl(\widetilde\psi_{x,i}(\mu)\bigr),
    \qquad
    i=1,\dots,n.
\]
Then
\[
    \widehat p:\mathcal P(X)\times\overline\Omega\to[-L,L]^n.
\]

By \eqref{Tagged 3}, since $J(h)(x)_i\in[-L,L]$, the clipping is inactive on
graph measures:
\[
    \widehat p(\gamma_h,x)_i
    =
    T_L\bigl(J(h)(x)_i\bigr)
    =
    J(h)(x)_i.
\]
Therefore
\begin{align}
\label{Tagged 6}
    \widehat p(\gamma_h,x)=J(h)(x)
    \qquad
    \forall h,\ x.
\end{align}

Since $T_L$ is $1$-Lipschitz, estimates \eqref{Tagged 4} and \eqref{Tagged 5} imply
\begin{align}
\label{Tagged 7}
    |\widehat p(\mu_1,x)_i-\widehat p(\mu_2,x)_i|
    \le
    L_0{\color{black}W_1}(\mu_1,\mu_2)
\end{align}
and
\begin{align}
\label{Tagged 8}
    |\widehat p(\mu,x_1)_i-\widehat p(\mu,x_2)_i|
    \le
    L_1|x_1-x_2|
\end{align}
for every coordinate $i$.

Summing the coordinate estimates in the Euclidean norm gives
\begin{align}
\label{Tagged 9}
    |\widehat p(\mu_1,x)-\widehat p(\mu_2,x)|_{\mathbb R^n}
    \le
    \sqrt n\,L_0{\color{black}W_1}(\mu_1,\mu_2)
\end{align}
and
\begin{align}
\label{Tagged 10}
    |\widehat p(\mu,x_1)-\widehat p(\mu,x_2)|_{\mathbb R^n}
    \le
    \sqrt n\,L_1|x_1-x_2|.
\end{align}
Using the triangle inequality,
\[
\begin{aligned}
    |\widehat p(\mu_1,x_1)-\widehat p(\mu_2,x_2)|_{\mathbb R^n}
    &\le
    |\widehat p(\mu_1,x_1)-\widehat p(\mu_2,x_1)|_{\mathbb R^n}
    +
    |\widehat p(\mu_2,x_1)-\widehat p(\mu_2,x_2)|_{\mathbb R^n} \\
    &\le
    \sqrt n\,L_0{\color{black}W_1}(\mu_1,\mu_2)
    +
    \sqrt n\,L_1|x_1-x_2|.
\end{aligned}
\]
This proves \eqref{Tagged 1}.

Now define
\[
    \widehat {\Gamma}(\mu,(x,y))
    \coloneqq
    (x,\widehat p(\mu,x)).
\]
Because $\widehat p(\mu,x)\in[-L,L]^n$, we have
$\widehat {\Gamma}(\mu,z)\in X$ for all $(\mu,z)\in\mathcal P(X)\times X$.

Let $z=(x,y)$ and $z'=(x',y')$. Then
\[
\begin{aligned}
    |\widehat {\Gamma}(\mu,z)-\widehat {\Gamma}(\nu,z')|_{\mathbb R^{d+n}}
    &=
    \bigl(|x-x'|_{\mathbb R^d}^2
    +
    |\widehat p(\mu,x)-\widehat p(\nu,x')|_{\mathbb R^n}^2
    \bigr)^{1/2}                                      \\
    &\le
    |x-x'|_{\mathbb R^d}
    +
    |\widehat p(\mu,x)-\widehat p(\nu,x')|_{\mathbb R^n} \\
    &\le
    \sqrt n\,L_0{\color{black}W_1}(\mu,\nu)
    +
    (1+\sqrt n\,L_1)|x-x'|_{\mathbb R^d}                 \\
    &\le
    \sqrt n\,L_0{\color{black}W_1}(\mu,\nu)
    +
    (1+\sqrt n\,L_1)|z-z'|_{\mathbb R^{d+n}}.
\end{aligned}
\]
Thus $\widehat f$ is Lipschitz on $\mathcal P(X)\times X$ with respect
to the product metric
\[
    d\bigl((\mu,z),(\nu,z')\bigr)
    =
    {\color{black}W_1}(\mu,\nu)+|z-z'|_{\mathbb R^{d+n}}.
\]

Finally, fix $h\in C(\overline\Omega;[-L,L]^n)$. By \eqref{Tagged 6},
\[
    \widehat {\Gamma}(\gamma_h,(x,y))
    =
    (x,J(h)(x))
\]
for all $(x,y)\in X$. Equivalently,
\[
    \widehat {\Gamma}(\gamma_h,\cdot)
    =
    (\operatorname{id}_{\overline\Omega},J(h))\circ \pi_1,
\]
where
\[
    \pi_1:X\to\overline\Omega,
    \qquad
    \pi_1(x,y)=x.
\]
Since
\[
    (\pi_1)_{\#}\gamma_h=\overline \lambda_{\Omega},
\]
we get
\[
\begin{aligned}
    \widehat {\Gamma}(\gamma_h,\cdot)_{\#}\gamma_h
    &=
    \bigl((\operatorname{id}_{\overline\Omega},J(h))\circ\pi_1\bigr)_{\#}
    \gamma_h                                      \\
    &=
    (\operatorname{id}_{\overline\Omega},J(h))_{\#}
    (\pi_1)_{\#}\gamma_h                           \\
    &=
    (\operatorname{id}_{\overline\Omega},J(h))_{\#}
    \overline \lambda_{\Omega}                    \\
    &=
    \gamma_{J(h)}.
\end{aligned}
\]
By the definition of $J$ from \textup{(P)},
\[
    {\Gamma}(\gamma_h,\cdot)_{\#}\gamma_h=\gamma_{J(h)}.
\]
Therefore
\[
    {\Gamma}(\gamma_h,\cdot)_{\#}\gamma_h
    =
    \widehat {\Gamma}(\gamma_h,\cdot)_{\#}\gamma_h
    =
    \gamma_{J(h)}.
\]
The proof is complete.
\qed\end{proof}

The above proves Lemma~\ref{lem: function graph transformers architecture}.

We include also the following result that essentially gives the converse to Lemma~\ref{lem: function graph transformers architecture}.

\begin{lemma}\label{lem: function graph transformers architecture part 2}
 If $\widehat p:{{{\mathcal P}}}(\overline\Omega\times [-L,L]^n)\times ( {\overline\Omega}\times [-L,L]^n)\to [-L,L]^n$
is a continuous function and $\widehat {\Gamma}(\mu,(x,y))$ is defined by formulas 
\eqref{widehat f map} and \eqref{widehat f map2}, then    $\widehat {\Gamma}:{{{\mathcal P}}}(\overline\Omega\times [-L,L]^n) \times ({\overline\Omega}\times [-L,L]^n) \to {\overline\Omega} \times [-L,L]^n$ is a continuous map and for  all continuous 
functions $h:\overline\Omega\to [-L,L]^n$ the graph measure $\gamma_h\in{{{\mathcal P}}}(\overline\Omega\times [-L,L]^n)$
satisfies
\beq
\widehat {\Gamma}(\gamma_h,\cdot)_\#\gamma_h=\gamma_g,\quad \hbox{where }g=\widehat p(\gamma_h,\cdot)\in C(\overline\Omega;[-L,L]^n)
\eeq
\end{lemma}

\begin{proof} Let us
define
\beq
\widehat {\Gamma}(\gamma_h,(x,y))=(x,\widehat p(\gamma_h,x))
\eeq 
Then, as $\widehat {\Gamma}(\gamma_h,(x,y))$ is independent of the variable $y$ we have 
$\widehat p(\gamma_h,\cdot)(x,y)=(\widehat p(\gamma_h,\cdot)\circ \pi_1)(x,y)$ and hence
\beq
\widehat {\Gamma}(\gamma_h,\cdot)_\#\gamma_h&=&(Id\times \widehat p(\gamma_h,\cdot))_\#\gamma_h\\
&=&((Id\times \widehat p(\gamma_h,\cdot)\circ \pi_1)_\#\gamma_h\\
&=&((Id\times \widehat p(\gamma_h,\cdot))_\# \bigg(( \pi_1)_\#\gamma_h\bigg)\\
&=&(Id\times \widehat p(\gamma_h,\cdot))_\# \overline \lambda_{\Omega}\\
&=&\gamma_{\widehat p(\gamma_h,\cdot)}.
\eeq
This proves the claim.
\qed
\end{proof}

\commented{
\section{Alternative definition of function graph transformers}
\label{app:Alternative definition of function graph transformers}

Lemma~\ref{lem: function graph transformers architecture} has the following generalization:

\begin{lemma}\label{lem: function graph transformers architecture GENERALIZATION}
If $\widehat p:{{{\mathcal P}}}(\overline\Omega\times [-L,L]^n)\times  {\overline\Omega}\to [-2L,2L]^n$
is a continuous function and $\widehat {\Gamma}(\mu,(x,y))$ is defined by formulas 
\beq\label{widehat f map2 modified}
\widehat {\Gamma}^{(x)}(\mu,(x,y))=x\quad   \hbox{and}\quad  
\widehat {\Gamma}^{(y)}(\mu,(x,y))=y+\widehat p(\mu,x),
\eeq
then    $\widehat {\Gamma}:{{{\mathcal P}}}(\overline\Omega\times [-L,L]^n)\times ({\overline\Omega}\times[-L,L]^n)\to {\overline\Omega}\times [-2L,2L]^n$ is a continuous map and for  all continuous 
functions $h:\overline\Omega\to [-L,L]^n$ the graph measure $\gamma_h\in{{{\mathcal P}}}(\overline\Omega\times [-L,L]^n)$
satisfies
\beq
\widehat {\Gamma}(\gamma_h,\cdot)_\#\gamma_h=\gamma_g,\quad \hbox{where }g=h+\widehat p(\gamma_h,\cdot)\in C(\overline\Omega;[-2L,2L]^n)
\eeq
\end{lemma}

\begin{proof}
The claim follows immediately from Lemma  \ref{lem: function graph transformers architecture} (i) by changing notations.
\qed\end{proof}

We note that in claim Lemma \ref{lem: function graph transformers architecture GENERALIZATION}, the map $\widehat f$ is the form
\begin{equation}
    \widehat {\Gamma}(\mu,(x,y))=Id(x,y)+(0,\widehat p(\mu,x)),
\end{equation}
where $Id:\R^d\times \R^n\to \R^d\times \R^n$ is the identity map $Id(x,y)=(x,y)$ and $\R^d\times \R^n$ is considered as a $(d+n)$-dimensonal vector space.
}
\subsection*{Justification of the original Definition~\ref{def self-attention}}

Below, we need the following well-known result. This result makes it possible to define  function graph transformer for graph measures
$\gamma_h$ and the approximate the measure $\gamma_h$ by a finite sum of delta-distributions.

\begin{lemma}\label{lem: triangular Wasserstein1}
Let $(X,d)$ be a metric space and let $p\ge1$.  
Assume 
$F_j : X \to X$ converge to $F$ in the $L^\infty$ topology and
that $Lip(F_j)$ are uniformly bounded, that is,
\beq\label{Lip limit}
\|F_j - F\|_{\infty} \to 0 
\qquad\text{and there is $C_0>0$ such that }\qquad 
\operatorname{Lip}(F_j) \le C_0.
\eeq
Assume also that $\mu_j \to \mu$ in $W_p$.  
Then the pushforward measures satisfy
\[
W_p\!\left( (F_j)_{\#}\mu_j ,\, F_{\#}\mu \right) \longrightarrow 0 .
\]
\end{lemma}

\begin{proof}
The condition \eqref{Lip limit} implies the Lipschitz constants $\operatorname{Lip}(F_j)$ are uniformly bounded.  
Using the triangle inequality for $W_p$,
\[
W_p\!\left( (F_j)_{\#}\mu_j ,\, F_{\#}\mu \right)
\le 
W_p\!\left( (F_j)_{\#}\mu_j ,\, (F_j)_{\#}\mu \right)
+
W_p\!\left( (F_j)_{\#}\mu ,\, F_{\#}\mu \right).
\]

For the first term, since $F_j$ is $\operatorname{Lip}(F_j)$-Lipschitz,
\beq\label{Wp term}
W_p\!\left( (F_j)_{\#}\mu_j ,\, (F_j)_{\#}\mu \right)
\le
\operatorname{Lip}(F_j)\, W_p(\mu_j,\mu).
\eeq
Because $\operatorname{Lip}(F_j)$ is uniformly bounded and $W_p(\mu_j,\mu)\to 0$, the term 
\eqref{Wp term} tends to $0$.

For the second term, since $F_j \to F$ uniformly,
\[
W_p\!\left( (F_j)_{\#}\mu ,\, F_{\#}\mu \right)
\le 
\left( \int_X d(F_j(x),F(x))^p \, d\mu(x) \right)^{1/p}
\le 
\|F_j - F\|_{\infty}.
\]
Thus this term also converges to $0$.

Combining the two estimates yields
\[
W_p\!\left( (F_j)_{\#}\mu_j ,\, F_{\#}\mu \right) \to 0,
\]
as desired. \hfill $\square$
\end{proof}
\bigskip

\begin{corollary}\label{cor: 1self}
    If the map 
${\Gamma}:{{{\mathcal P}}}(\overline\Omega\times [-L,L]^n)\times ({\overline\Omega}\times [-L,L]^n)\to {\overline\Omega}\times [-L,L]^n$, ${\Gamma}:(\mu,(x,y))\to {\Gamma}(\mu,(x,y))$ is uniformly continuous in $\mu$ and  uniformly Lipschitz
bounded in $(x,y)$, that is, 
\begin{itemize}
    \item [(i)]
if $\mu_j\to \mu$ in ${{{\mathcal P}}}(\overline\Omega\times [-L,L]^n)$ then
\beq
\lim_{j\to\infty}\sup_{(x,y)\in {\overline\Omega}\times [-L,L]^n} |{\Gamma}(\mu_j,(x,y))-{\Gamma}(\mu,(x,y))|= 0
\eeq

\item [(ii)] There in $L>0$ such that  for all $(x_1,y_1),(x_2,y_2)\in \overline\Omega\times [-L,L]^n$
\beq
\sup_{\mu\in {{{\mathcal P}}}(\overline\Omega\times [-L,L]^n} |{\Gamma}(\mu,(x_1,y_1))-{\Gamma}(\mu,(x_2,y_2))|\le
L(|x_1-x_2|+|y_1-y_2|),
\eeq
\end{itemize}
then the map
$G:{{{\mathcal P}}}(\overline\Omega\times [-L,L]^n)\to {{{\mathcal P}}}(\overline\Omega\times [-L,L]^n)$
given by 
\beq\label{G eta def pre}
G(\mu)={\Gamma}(\mu,\cdot)_\#\mu,
\eeq
is continuous.

\end{corollary}

\begin{proof}
Let $\mu_j \to \mu$ in ${{{\mathcal P}}}(\overline\Omega\times [-L,L]^n)$.
Define $F_j := {\Gamma}(\mu_j,\cdot)$ and $F := {\Gamma}(\mu,\cdot)$, so that
$G(\mu_j) = (F_j)_\#\mu_j$ and $G(\mu) = F_\#\mu$.

We verify the hypotheses of Lemma~\ref{lem: triangular Wasserstein1}
(with $X = \overline\Omega\times [-L,L]^n$ and $p=1$):

\begin{itemize}
\item [1.] By assumption~(i),
\[
\|F_j - F\|_\infty
= \sup_{(x,y)\in \overline\Omega\times [-L,L]^n}
|{\Gamma}(\mu_j,(x,y)) - {\Gamma}(\mu,(x,y))| \to 0.
\]

\item [2.]By assumption~(ii), for every~$j$ and every
$(x_1,y_1),(x_2,y_2)\in \overline\Omega\times [-L,L]^n$,
\[
|F_j(x_1,y_1) - F_j(x_2,y_2)|
= |{\Gamma}(\mu_j,(x_1,y_1)) - {\Gamma}(\mu_j,(x_2,y_2))|
\le L(|x_1 - x_2| + |y_1 - y_2|),
\]
so $\operatorname{Lip}(F_j) \le L$ uniformly in~$j$.

\item [3.]By hypothesis, $\mu_j \to \mu$ in ${{\color{black}W_1}}$.
\end{itemize}

\noindent
Therefore, by Lemma~\ref{lem: triangular Wasserstein1} 
\[
{{\color{black}W_1}}\!\left(G(\mu_j),\, G(\mu)\right)
= {{\color{black}W_1}}\!\left((F_j)_\#\mu_j,\, F_\#\mu\right) \to 0,
\]
which proves that $G:{{{\mathcal P}}}(\overline\Omega\times [-L,L]^n)
\to {{{\mathcal P}}}(\overline\Omega\times [-L,L]^n)$ is continuous.
\hfill $\square$
\end{proof}

\subsection*{An auxiliary lemma on Wasserstein convergence}\label{aux lemma Wasserstein}

We need the following, quite well-known result. We recall its proof as it is essential below.

\begin{lemma}[Uniform map approximation implies uniform $W_1$ approximation]
\label{lem:uniform-map-to-W1}
Let $(X,d_X)$ be a compact metric space and let $(Y,d_Y)$ be a metric
space. Let
\[
    \Gamma_m,\Gamma:\mathcal P(X)\times X\to Y
\]
be Borel maps such that, for every $\mu\in\mathcal P(X)$, the pushforward
measures
\[
    G_m(\mu):=(\Gamma_m(\mu,\cdot))_\#\mu,
    \qquad
    G(\mu):=(\Gamma(\mu,\cdot))_\#\mu
\]
belong to $\mathcal P_1(Y)$. Suppose that
\[
    \Delta_m
    :=
    \sup_{\mu\in\mathcal P(X)}
    \sup_{z\in X}
    d_Y\bigl(\Gamma_m(\mu,z),\Gamma(\mu,z)\bigr)
    <\infty .
\]
Then
\[
    \sup_{\mu\in\mathcal P(X)}
    W_1\bigl(G_m(\mu),G(\mu)\bigr)
    \le \Delta_m .
\]
\end{lemma}

\begin{proof}
Fix $\mu\in\mathcal P(X)$. By the Kantorovich--Rubinstein duality,
\[
    W_1\bigl(G_m(\mu),G(\mu)\bigr)
    =
    \sup_{\operatorname{Lip}(\varphi)\le 1}
    \left|
        \int_Y \varphi(y)\,dG_m(\mu)(y)
        -
        \int_Y \varphi(y)\,dG(\mu)(y)
    \right|,
\]
where the supremum is taken over all real-valued $1$-Lipschitz functions
$\varphi:Y\to\mathbb R$.

Using the definition of pushforward measure, for every such $\varphi$,
\[
\begin{aligned}
    &\left|
        \int_Y \varphi(y)\,dG_m(\mu)(y)
        -
        \int_Y \varphi(y)\,dG(\mu)(y)
    \right|
    \\
    &\qquad =
    \left|
        \int_X
        \Bigl[
            \varphi\bigl(\Gamma_m(\mu,z)\bigr)
            -
            \varphi\bigl(\Gamma(\mu,z)\bigr)
        \Bigr]
        \,d\mu(z)
    \right|
    \\
    &\qquad \le
    \int_X
    \left|
        \varphi\bigl(\Gamma_m(\mu,z)\bigr)
        -
        \varphi\bigl(\Gamma(\mu,z)\bigr)
    \right|
    \,d\mu(z)
    \\
    &\qquad \le
    \int_X
    d_Y\bigl(\Gamma_m(\mu,z),\Gamma(\mu,z)\bigr)
    \,d\mu(z)
    \\
    &\qquad \le
    \Delta_m .
\end{aligned}
\]
Taking the supremum over all $1$-Lipschitz $\varphi$ gives
\[
    W_1\bigl(G_m(\mu),G(\mu)\bigr)
    \le \Delta_m .
\]
Since the bound is independent of $\mu$, taking the supremum over
$\mu\in\mathcal P(X)$ yields
\[
    \sup_{\mu\in\mathcal P(X)}
    W_1\bigl(G_m(\mu),G(\mu)\bigr)
    \le \Delta_m .
\]
This proves the claim.\qed
\end{proof}

\section{Generalized patching and tokenization}
\label{section: Patching}



Let $D = \bigcup_{\rho=1}^P \overline D_\rho$ where $D_\rho\subset \R^d$,  $D_\rho=D^0+q_\rho$ are open sets, such that $D_{\rho_1} \cap D_{\rho_2} = \emptyset$ for $\rho_1 \not = \rho_2$.

\subsection{Patching for continuous functions}

For $f:\ D \to \R^n$ be continuous, $f(x) = (f_i(x))_{i=1}^n$, let $\hat{\mathbf{E}}(f) :\ D \to \R^C$ be
\begin{equation}
  \hat{\mathbf{E}}(f)(x)=( \hat{\mathbf{E}}_c(f)(x))_{c=1}^C
   = \sum_{\rho=1}^P \sum_{i=1}^n \mathbf{F}_c^i \left(
   \int_{D_{\rho}} W_{\color{black}c}(x) f_i(x) dx \right) \mathbb{I}_{D_{\rho}}(x) ,
\end{equation}
where $\mathbf{F} = (\mathbf{F}_{c}^i)_{c\in [C],i\in [n]} \in \mathbb{R}^{C \times n}$ and
\[
   W_{\color{black}c}(x) = \sum_{j=1}^J W_{\color{black}c,j} \delta_{x_j+q_\rho}(x) 
\]
{

\subsection{Patching for generalized functions}

Let us next consider the non-smooth function $f\in H^{-s}(D) $ and replace the delta-distributions by smooth test functions. For $f \in H^{-s}(D)$, $f(x)=(f_i(x))_{i=1}^n$, let $\tau > 0$ and define $\hat{\mathbf{E}}(f) :\ D \to \R^C$ to be
\begin{equation}
  \hat{\mathbf{E}}_\tau (f)(x) = ( \hat{\mathbf{E}}_c(f)(x))_{c=1}^C
  = \sum_{\rho=1}^P \sum_{i=1}^n \mathbf{F}_c^i \left(
  \int_{\R^d} W_{\color{black}c}_\tau(x) f_i(x) dx \right) \mathbb{I}_{D_{\rho}}(x) ,
\end{equation}
where 
\[
   W_{\color{black}c}_\tau (x) = \sum_{j=1}^J W_{\color{black}c,j} \rho_\tau(x-({x_j+q_\rho}))
\]
and $\rho_\tau \in C^\infty_0(\R^d)$. Then 
\beq
  \hat{\mathbf{E}}(f)(x)
  &=& (\hat{\mathbf{E}}^{(i)}_c(f)(x))_{c=1,...C, i=1,...,n}
  \\
  &=& \sum_{\rho=1}^P \sum_{i=1}^n \left(\int_{D_{\rho}} \sum_{j=1}^J
  \mathbf{F}_c^iW_{\color{black}c,j}  \rho_\tau(x-({x_j+q_\rho})) f_i(x)dx\right)_{i=1}^n
  \mathbb{I}_{D_{\rho}}(x)
  \\
  &=& \sum_{\rho=1}^P \sum_{i=1}^n \left(\int_{D_{\rho}} \sum_{j=1}^J
  \mathbf{H}^i_{\color{black}c,j} \rho_\tau(x-({x_j+q_\rho}))  f_i(x)dx\right)_{i=1}^n
  \mathbb{I}_{D_{\rho}}(x) ,
\eeq
where $\mathbf{H}^{(i)}_{\color{black}c,j} = \mathbf{F}_c^iW_{\color{black}c,j}$. 

\medskip


\subsection{From a patched function to a measure}

Let us next consider $f\in C(\overline D).$
Let $\mathbf{H}^{(i)} = [\mathbf{H}^{(i)}_{c,j}]_{c \in [C], j \in [J]} \in \R^{C\times J}$ and assume that for all $i=1,...,n$ the linear map $\mathbf{H}^{(i)} :\ \R^{J} \to \R^{C}$ is one-to-one, that is, the matrix $\mathbf{H}^{(i)}$ has a left inverse matrix $\mathbf{K}^{(i)} = [\mathbf{K}^{(i)}_{j,c}]_{j\in [J],c\in [C]} \in \R^{J\times C}$, that is, for all $i$ we have
\beq
  \sum_{c=1}^C \mathbf{K}^{(i)}_{j,c} \mathbf{H}^{(i)}_{c,j'} = \delta_{j,j'} .
\eeq
Then, for $x \in D_\rho$,
\beq
  \sum_{c=1}^C \mathbf{K}^{(i)}_{j,c} \hat{\mathbf{E}}^{(i)}_c(f)(x)= f^{(i)}(x_j+q_\rho) .
\eeq
Let ${\bf N}_\rho :\ (C(\overline D_\rho;{[-L,L]}^n))^C \to {\mathcal P}(D_\rho\times \R^n)$ be the operator
\beq
    {\bf N}_\rho(g)=\frac 1{J} \sum_{j=1}^J \delta_{(x_j+q_\rho,(\sum_{c=1}^C \mathbf{K}^{(i)}_{j,c} g^{(i)}_c(q_\rho))_{i=1}^n)} 
\eeq
and $\hat {\bf N} :\ C_{pc}(\overline D)^{nJ} \to {\mathcal P}(D\times \R^n)$ be the operator
\beq
  \hat {\bf N}(g)&=&\frac 1{P} \sum_{\rho=1}^P {\bf N}_\rho(g)
  \\
  &=&\frac 1{PJ} \sum_{\rho=1}^P
  \sum_{j=1}^J \delta_{(x_j+q_\rho,(\sum_{c=1}^C \mathbf{K}^{(i)}_{j,c} g^{(i)}_c(q_\rho))_{i=1}^n)} 
\eeq
where $C_{pc}(\overline D)^{nJ}=\bigoplus_{\rho=1}^P (C(D_\rho;{[-L,L]}^n))^J$ is the space of piecewise continuous functions that are continuous in the sets $D_\rho \subset D$. Then,
\beq
  \hat {\bf N}(\hat{\mathbf{E}}_c(f)) = \frac 1{PJ}
  \sum_{\rho=1}^P\sum_{j=1}^J \delta_{(x_j+q_\rho,f(x_j+q_\rho))} .
\eeq

\subsection*{The inverse of patching}

{Next, we consider the map from a  measure approximating a function to a measure approximating a patched  function and the inverse operation.} 
We consider the computational grids 
\beq
  & &G_{large} = \{x_j+q_\rho:\ j=1,2,\dots,J,\ \rho=1,2,\dots,P\} ,
  \\
  & &G_{small} = \{q_\rho:\ \rho=1,2,\dots,P\} .
\eeq
We also consider the operators $\hat J :\ {\mathcal P}(G_{large} \times \R^n) \to {\mathcal P}(G_{small}\times \R^{n})^J$,
\beq
\hat J:\sum_{\rho=1}^P \sum_{j=1}^J a_{j\rho}\delta_{(x_j+q_\rho,y_{j,\rho})}\to  
\bigg(\sum_{\rho=1}^P  a_{j\rho}\delta_{(x_j,y_{j,\rho})}\bigg)_{j=1}^J
\eeq
and its inverse operator
\beq
\hat J^{-1}:
\bigg(\sum_{\rho=1}^P  a_{j\rho}\delta_{(x_j,y_{j,\rho})}\bigg)_{j=1}^J\to \sum_{\rho=1}^P \sum_{j=1}^J a_{j\rho}\delta_{(x_j+q_\rho,y_{j,\rho})}
\eeq
Observe that the map $\hat J$ is given by the pushforward map, $\hat J = T_\#$, where $T :\ D \times \R^{nJ} \to D \times (\R^n)^J$, that is given by
\beq\label{T-map}
  T :\ (x,(y_{j,i})_{i\in [n],j\in [J]}) \to (x+x_j,((y_{j,\rho})_{i\in [n]})_{j\in [J]}) .
\eeq

\commented{Maarten's suggestion: {\color{red} This is typically done as follows: An operator transforms the input function into a piecewise constant function, which is constant within patches (subdivisions of the domain $D$), by taking weighted averages and then transforming these piecewise constant values into a $C$-dimensional latent space. That is, let $D = \cup_{\rho=1}^P D_{\rho}$ (patches) and $f$ be an input function, then
\begin{equation}
   \boldsymbol{v}(x) = \hat{\mathbf{E}}(f)(x)
   = \sum_{\rho=1}^N
   \mathbf{F} \left(\int_{D_{\rho}} W(x) f(x) \, dx\right) \mathbb{I}_{D_{\rho}}(x) ,
\end{equation}
where $\mathbf{F} \in \mathbb{R}^{C \times n}$ and
\[
   (W(x))_i = \sum_j W_{ij} \delta_{x_j}(x) 
\]
in which $i = (i_1,\cdots,i_n)$ and $j$ are multi-indices. The weight $W$ are shared amongst all the patches.} Does this require that $f \in L^1_{\mathrm{loc}}$ (see the salt bodies in the experiment)?}

\section{Auxiliary results on approximation of function graph transformers}
\label{App: genreal approximation of function graph transformers}

To prove universal approximation results in the main text, we consider approximation of measure-theoretic transformers in general.

\begin{proposition}\label{prop: approx by Gn}
Let $\varepsilon > 0$, $N \in \mathbb{N}$, and $h \in C(\overline\Omega;[-L,L]^n)$.
Suppose that the maps
$G_n, {G} : {{{\mathcal P}}}(\overline\Omega \times [-L,L]^n)
\to {{{\mathcal P}}}(\overline\Omega \times [-L,L]^n)$
satisfy
\beq\label{pointwise conv Gn}
\lim_{n \to \infty} {\color{black}W_1}\!\left(G_n(\mu),\, {G}(\mu)\right) = 0
\quad \text{for all } \mu \in {{{\mathcal P}}}(\overline\Omega \times [-L,L]^n).
\eeq
Then for each $h \in C(\overline\Omega;[-L,L]^n)$
\beq\label{goal A prop}
{\bf F}_\varepsilon \circ {G} \circ {\bf M}_N(h)
= \lim_{n \to \infty} {\bf F}_\varepsilon \circ G_n \circ {\bf M}_N(h) \in C(\overline\Omega;\R^n),
\eeq
where the convergence holds in $C(\overline\Omega;\mathbb{R}^n)$
$($and hence also in $H^{-s}(\Omega;\mathbb{R}^n)$ for any $s \ge 0$$)$.
Similar statement is valid if $G_n$ is replaced by $G_\eta$, $\eta>0$ and the limit
$n\to \infty$ is replaced by $\eta\to 0.$
\end{proposition}


\begin{proof}
The proof proceeds in two steps:
we first show that ${\bf F}_\varepsilon$ is continuous,
and then compose with the pointwise convergence of~$G_n$.

\medskip
\noindent\textbf{Step 1.} We show that for fixed $\varepsilon > 0$,
the map ${\bf F}_\varepsilon : {{{\mathcal P}}}(\overline\Omega \times [-L,L]^n) \to C(\overline\Omega;\R^n)$
is continuous, where ${{{\mathcal P}}}(\overline\Omega \times [-L,L]^n)$ is
equipped with the ${\color{black}W_1}$-topology.

Recall from~\eqref{f mollified 1B} that, that for $\gamma \in {{{\mathcal P}}}(\overline\Omega \times [-L,L]^n)$,
\[
({\bf F}_\varepsilon \gamma)_i(x)
= \frac{\text{Vol}(\Omega)}{\varepsilon^d}
\int_{\overline\Omega \times [-L,L]^n}
\rho\!\left(\frac{x' - x}{\varepsilon}\right) y_i \, d\gamma(x', y),
\quad i = 1, \ldots, n.
\]
For each $x \in \overline\Omega$ and $i \in \{1,\ldots,n\}$,
define $\varphi_{x,i} : \overline\Omega \times [-L,L]^n \to \R$ by
\[
\varphi_{x,i}(x', y) = \frac{{\text{Vol}(\Omega)}}{\varepsilon^d}\,
\rho\!\left(\frac{x' - x}{\varepsilon}\right) y_i.
\]
Since $\rho \in C_0^\infty(\R^d)$ and $|y_i| \le L$,
the function $\varphi_{x,i}$ is continuous and bounded on
$\overline\Omega \times [-L,L]^n$, uniformly in $x$:
\beq\label{phi bound}
\|\varphi_{x,i}\|_\infty
\le \frac{L{\text{Vol}(\Omega)}}{\varepsilon^d} \|\rho\|_\infty
=: C_\varepsilon
\quad \text{for all } x \in \overline\Omega.
\eeq

Now let $\gamma_m \to \gamma$ in ${\color{black}W_1}$ on ${{{\mathcal P}}}(\overline\Omega \times [-L,L]^n)$.
Since $\overline\Omega \times [-L,L]^n$ is compact, convergence in ${\color{black}W_1}$
implies weak convergence of measures together with convergence
of total mass (see, e.g., Villani, \cite{Villani}, Theorem~6.9).
In particular, for each fixed $x$ and~$i$,
\beq\label{pointwise Feps}
({\bf F}_\varepsilon \gamma_m)_i(x)
= \int_{\overline\Omega \times [-L,L]^n} \varphi_{x,i} \, d\gamma_m
\to \int_{\overline\Omega \times [-L,L]^n} \varphi_{x,i} \, d\gamma
= ({\bf F}_\varepsilon \gamma)_i(x).
\eeq

It remains to upgrade pointwise convergence~\eqref{pointwise Feps}
to uniform convergence in $x \in \overline\Omega$.
Since $\rho$ is smooth, the family of functions
$\{x \mapsto \varphi_{x,i}(x',y)\}$ satisfies, for all
$(x',y) \in \overline\Omega \times [-L,L]^n$,
\[
|\varphi_{x_1,i}(x',y) - \varphi_{x_2,i}(x',y)|
\le \frac{L{\text{Vol}(\Omega)}}{\varepsilon^{d+1}} \|\nabla \rho\|_\infty \, |x_1 - x_2|.
\]
Hence, for any positive measure $\nu$ on $\overline\Omega \times [-L,L]^n$,
\beq\label{equicont Feps}
\left|({\bf F}_\varepsilon \nu)_i(x_1) - ({\bf F}_\varepsilon \nu)_i(x_2)\right|
&\le& \int |\varphi_{x_1,i} - \varphi_{x_2,i}| \, d\nu \nonumber\\
&\le& \frac{L{\text{Vol}(\Omega)}}{\varepsilon^{d+1}} \|\nabla \rho\|_\infty \,
\nu(\overline\Omega \times [-L,L]^n) \, |x_1 - x_2|.
\eeq
Since $\gamma_m \to \gamma$ in ${\color{black}W_1}$ on a compact space,
the total masses $\gamma_m(\overline\Omega \times [-L,L]^n)$
converge to $\gamma(\overline\Omega \times [-L,L]^n)$ and are thus
uniformly bounded by some $C_1 > 0$.
By~\eqref{equicont Feps}, the functions
$\{x \mapsto ({\bf F}_\varepsilon \gamma_m)_i(x)\}_{m \ge 1}$
are equicontinuous and uniformly bounded (by $C_\varepsilon C_1$, using~\eqref{phi bound}).
Combined with the pointwise convergence~\eqref{pointwise Feps},
the Arzel\`a--Ascoli theorem yields
\[
\|{\bf F}_\varepsilon \gamma_m - {\bf F}_\varepsilon \gamma\|_{C(\overline\Omega;\R^n)} \to 0.
\]
This proves that ${\bf F}_\varepsilon$ is continuous.

\medskip
\noindent\textbf{Step 2.}
Let $\mu_0 := {\bf M}_N(h) \in {{{\mathcal P}}}(\overline\Omega \times [-L,L]^n)$,
which is fixed.
By the pointwise convergence assumption~\eqref{pointwise conv Gn},
\[
G_n(\mu_0) \to {G}(\mu_0)
\quad \text{in } {\color{black}W_1}.
\]
Applying Step~1,
\[
{\bf F}_\varepsilon(G_n(\mu_0)) \to {\bf F}_\varepsilon({G}(\mu_0))
\quad \text{in } C(\overline\Omega;\R^n),
\]
which is precisely~\eqref{goal A prop}.
\hfill $\square$
\end{proof}

\begin{remark}\label{rem: continuous parameter goes to zero}
The proof of {Proposition} \ref{prop: approx by Gn} when a limit $n\to \infty$ of the integer valued index $n$ of
$G_n$ is changed to the limit $\eta\to 0$ of a  continuous valued parameter $\eta$ (when one consider $G_\eta$ instead of $G_n$) is analogous. 
\end{remark}

\subsection{Proof of \eqref{discrete M2} and \eqref{discrete M3}}
\label{app: proof of M2 and M3}

\subsubsection{Recalling the setup}

Let $\Omega\subset\R^d$ be a {nonempty} bounded open set,
$D\eqdef\diam(\Omega)<\infty$, and ${\overline {\lambda}}_\Omega$ its normalized Lebesgue measure.
{Since $\Omega$ is nonempty and open, $\text{Vol}(\Omega)>0$.}
Define the normalised uniform probability measure
\begin{equation}\label{eq:barlambda}
\bar\lambda_\Omega\eqdef\frac{1}{\text{Vol}(\Omega)}\lambda_\Omega\in\mathcal P(\overline\Omega).
\end{equation}

\begin{definition}[Graph measure]\label{def:graph}
{
For a Borel map $h:\Omega\to\R^n$ we define
\begin{equation}\label{eq:gamma}
\gamma_h\in\mathcal P(\overline\Omega\times\R^n),
\qquad
\int_{\overline\Omega\times[-L,L]^n}\phi\,d\gamma_h
=
\frac{1}{\text{Vol}(\Omega)}
\int_\Omega \phi(x,h(x))\,dx
\quad
\forall \phi\in C_b(\overline\Omega\times \R^n).
\end{equation}
Equivalently, if $h$ is defined on $\overline\Omega$, then
$\gamma_h=(\id,h)_\#\bar\lambda_\Omega$.
If $h\in L^p(\Omega;\R^n)$, then $\gamma_h$ has finite $p$-th moment on
$\overline\Omega\times\R^n$.
The definition is independent of the representative of $h$ up to
$\bar\lambda_\Omega$-a.e. equality.
}
\end{definition}

The points $\{{\bar x}_j\}_{j\ge1}\subset\Omega$ are \emph{regularly distributed with density
$\kappa\equiv1$} if the empirical measure
$\sigma_N\eqdef\tfrac1N\sum_{j=1}^N\delta_{{\bar x}_j}$ satisfies
\begin{equation}\label{eq:regular}
\frac1N\sum_{j=1}^N\varphi({\bar x}_j)
\xrightarrow[N\to\infty]{}\int_{\overline\Omega}\varphi\,d\bar\lambda_\Omega
\qquad\forall\varphi\in C(\overline\Omega),
\end{equation}
i.e.\ $\sigma_N\rightharpoonup\bar\lambda_\Omega$ narrowly. For $h\in C(\overline\Omega;\R^n)$ the
empirical graph measure is
\begin{equation}\label{eq:MN}
\mathbf M_N(h)=(\id\times h)_\#\sigma_N=\frac1N\sum_{j=1}^N\delta_{({\bar x}_j,h({\bar x}_j))}
\in\mathcal P(\overline\Omega\times\R^n).
\end{equation}

{
 Let
$\rho_\tau\in C_0^\infty(\R^d)$ be a standard mollifier with
$\supp\rho_\tau\subset B(0,\tau)$ and $\int_{\R^d}\rho_\tau dx=1$. Define
\begin{equation}\label{eq:Rtau}
\Omega_\tau
\eqdef
\{x\in\Omega:\dist(x,\partial\Omega)>4\tau\},
\qquad
\psi_\tau
\eqdef
\rho_\tau*\mathbf 1_{\Omega_\tau},
\qquad
R_\tau h
\eqdef
\rho_\tau * \mathcal E_0(\psi_\tau h).
\end{equation}
Here $\mathcal E_0$ denotes the canonical extension by zero to $\R^d$ of the compactly
supported function or distribution $\psi_\tau h$ on $\Omega$. For vector-valued $h$ the
definition is componentwise. The empirical regularised graph measure is
\[
\mathbf M_{N,\tau}(h)\eqdef\mathbf M_N(R_\tau h).
\]
The enlarged margin $4\tau$ in the definition of $\Omega_\tau$ is used only to ensure
$\supp R_\tau h\Subset\Omega$.
}

The Wasserstein distance
$W_p$, $p\in[1,\infty)$, on $\mathcal P(\overline\Omega\times\R^n)$ uses the Euclidean cost;
see \cite[Def.~6.1]{Villani2009}.

\subsubsection{Proof of \eqref{discrete M2}}

\begin{proposition}[\eqref{discrete M2}]\label{thm:M2}
Under~\eqref{eq:regular}, for every $h\in C(\overline\Omega;\R^n)$ and every $p\in[1,\infty)$,
\[
\lim_{N\to\infty}W_p\!\left(\mathbf M_N(h),\gamma_h\right)=0.
\]
\end{proposition}

\begin{proof}
Set $L\eqdef\norm{h}_\infty<\infty$ and $K\eqdef\overline\Omega\times[-L,L]^n$, which is compact.
Both $\mathbf M_N(h)$ and $\gamma_h$ are supported in $K$. For $\phi\in C(K)$, the composition
$\Phi(x)\eqdef\phi(x,h(x))$ is in $C(\overline\Omega)$, so~\eqref{eq:regular} gives
\[
\int_K\phi\,d\mathbf M_N(h)=\tfrac1N\sum_{j=1}^N\Phi({\bar x}_j)
\to\tfrac{1}{\text{Vol}(\Omega)}\int_\Omega\Phi=\int_K\phi\,d\gamma_h,
\]
i.e.\ $\mathbf M_N(h)\rightharpoonup\gamma_h$ narrowly. On a compact Polish space, narrow
convergence is equivalent to $W_p$-convergence for every $p\in[1,\infty)$
(\cite[Thm.~6.9]{Villani2009}), so $W_p(\mathbf M_N(h),\gamma_h)\to0$.
\qed\end{proof}

\subsubsection{Stability of graph measures and proof of \eqref{discrete M2}B}

{
\begin{lemma}[Graph measure stability]\label{lem:stab}
Let $p\in[1,\infty)$ and let $f,g:\Omega\to\R^n$ be Borel maps with
$f,g\in L^p(\Omega;\R^n)$. Then
\begin{equation}\label{eq:stab}
W_p(\gamma_f,\gamma_g)
\le
\norm{f-g}_{L^p(\bar\lambda_\Omega)}
=
\text{Vol}(\Omega)^{-1/p}\norm{f-g}_{L^p(\Omega)}.
\end{equation}
In particular, if $f,g\in C(\overline\Omega;\R^n)$, then
\[
W_p(\gamma_f,\gamma_g)\le\norm{f-g}_{C(\overline\Omega)}.
\]
\end{lemma}

\begin{proof}
The Borel map
\[
\Psi:\Omega\to(\overline\Omega\times\R^n)^2,
\qquad
\Psi(x)=((x,f(x)),(x,g(x))),
\]
pushes $\bar\lambda_\Omega$ to a coupling of $\gamma_f$ and $\gamma_g$. Hence
\[
W_p(\gamma_f,\gamma_g)^p
\le
\int_\Omega
\norm{(x,f(x))-(x,g(x))}^p\,d\bar\lambda_\Omega(x)
=
\int_\Omega \abs{f(x)-g(x)}^p\,d\bar\lambda_\Omega(x).
\]
Taking $p$-th roots proves the claim.
\qed\end{proof}
}

{
\begin{lemma}[Smoothness and $L^p$ approximation by $R_\tau$]\label{lem:Rtau}
Fix $\tau>0$ and $s\in{\Z_+}$. For every $h\in H^{-s}(\Omega;\R^n)$,
\[
R_\tau h\in C_c^\infty(\Omega;\R^n)\subset C(\overline\Omega;\R^n).
\]
Moreover, for every $p\in[1,\infty)$ and every $h\in L^p(\Omega;\R^n)$,
\begin{equation}\label{eq:Rtau_Lp}
\norm{R_\tau h-h}_{L^p(\Omega)}
\xrightarrow[\tau\downarrow0]{}0.
\end{equation}
\end{lemma}

\begin{proof}
It is enough to argue componentwise. By the definition of $\Omega_\tau$ and the support
condition $\supp(\rho_\tau)\subset B(0,\tau)$,
\[
\supp(\psi_\tau)
=
\supp(\rho_\tau*\mathbf 1_{\Omega_\tau})
\subset
\{x\in\Omega:\dist(x,\partial\Omega)\ge 2\tau\}
\Subset\Omega.
\]
Hence $\psi_\tau h$ is a compactly supported distribution in $\Omega$, and its extension
$\mathcal E_0(\psi_\tau h)$ is a compactly supported distribution on $\R^d$.
Convolution with $\rho_\tau\in C_0^\infty(\R^d)$ gives
$R_\tau h\in C^\infty(\R^d)$. Its support satisfies
\[
\supp( R_\tau h)
\subset
\supp(\psi_\tau h)+\supp\rho_\tau
\subset
\{x\in\Omega:\dist(x,\partial\Omega)\ge\tau\}
\Subset\Omega,
\]
so $R_\tau h\in C_c^\infty(\Omega)$.

Now let $h\in L^p(\Omega;\R^n)$ and let $\widetilde h$ denote its extension by zero to $\R^d$.
Since $0\le \psi_\tau\le1$ and $\psi_\tau(x)\to1$ for a.e.\ $x\in\Omega$, dominated convergence
gives
\[
\norm{\mathcal E_0(\psi_\tau h)-\widetilde h}_{L^p(\R^d)}
=
\norm{(\psi_\tau-1)h}_{L^p(\Omega)}
\to0.
\]
Therefore,
\[
\begin{aligned}
\norm{R_\tau h-\widetilde h}_{L^p(\R^d)}
&\le
\norm{\rho_\tau*
\bigl(\mathcal E_0(\psi_\tau h)-\widetilde h\bigr)}_{L^p(\R^d)}
+
\norm{\rho_\tau*\widetilde h-\widetilde h}_{L^p(\R^d)}
\\
&\le
\norm{\mathcal E_0(\psi_\tau h)-\widetilde h}_{L^p(\R^d)}
+
\norm{\rho_\tau*\widetilde h-\widetilde h}_{L^p(\R^d)}
\to0,
\end{aligned}
\]
where the first term uses Young's inequality and the second term is the standard
approximation-of-the-identity theorem. Restricting to $\Omega$ proves~\eqref{eq:Rtau_Lp}.
\qed\end{proof}

\begin{remark}\label{rem:no_uniform_Rtau}
The stronger assertion
\[
\norm{R_\tau h-h}_{C(\overline\Omega)}\to0
\]
is false for the operator~\eqref{eq:Rtau} in general, because $R_\tau h$ is cut off near
$\partial\Omega$. For example, if $\Omega=(0,1)$ and $h\equiv1$, then $R_\tau h$ vanishes
near the boundary while $h$ does not. The correct convergence statement needed below is the
$L^p$ convergence~\eqref{eq:Rtau_Lp}.
\end{remark}
}

\begin{proposition}[\eqref{discrete M2}B]\label{thm:M2b}
Under~\eqref{eq:regular}, for every $h\in C(\overline\Omega;\R^n)$ and every $p\in[1,\infty)$,
\[
\lim_{\tau\downarrow0}\lim_{N\to\infty}\mathbf M_{N,\tau}(h)=\gamma_h\quad\text{in }W_p.
\]
\end{proposition}

\begin{proof}
{
By Lemma~\ref{lem:Rtau}, $R_\tau h\in C_c^\infty(\Omega;\R^n)\subset C(\overline\Omega;\R^n)$
for every $\tau>0$. Applying Theorem~\ref{thm:M2} to $R_\tau h$ gives
\begin{equation}\label{eq:inner}
\lim_{N\to\infty}\mathbf M_{N,\tau}(h)
=
\lim_{N\to\infty}\mathbf M_N(R_\tau h)
=
\gamma_{R_\tau h}
\quad\text{in }W_p.
\end{equation}
By Lemmas~\ref{lem:stab} and~\ref{lem:Rtau},
\begin{equation}\label{eq:outer}
W_p(\gamma_{R_\tau h},\gamma_h)
\le
\norm{R_\tau h-h}_{L^p(\bar\lambda_\Omega)}
=
\text{Vol}(\Omega)^{-1/p}
\norm{R_\tau h-h}_{L^p(\Omega)}
\xrightarrow[\tau\downarrow0]{}0.
\end{equation}
Combining~\eqref{eq:inner} and~\eqref{eq:outer} proves the iterated limit.
}
\qed\end{proof}

\subsubsection{Proof of \eqref{discrete M3}}

\begin{proposition}[\eqref{discrete M3}]\label{thm:M3}
Under~\eqref{eq:regular}, for every $\tau>0$, $s\in\N$, $h\in H^{-s}(\Omega;{\R}^n)$, and
$p\in[1,\infty)$,
\[
\lim_{N\to\infty}W_p\!\left(\mathbf M_{N,\tau}(h),\gamma_{R_\tau h}\right)=0,
\]
where $\gamma_{R_\tau h}=(\id,R_\tau h)_\#\bar\lambda_\Omega$.
\end{proposition}

\begin{proof}
Lemma~\ref{lem:Rtau} gives $R_\tau h\in C(\overline\Omega;\R^n)$. Apply Theorem~\ref{thm:M2} to
$R_\tau h$: $\mathbf M_{N,\tau}(h)=\mathbf M_N(R_\tau h)\to\gamma_{R_\tau h}$ in $W_p$.
\qed
\end{proof}

\commented{
\subsubsection{Explicit inequality for H\"older data}
\label{sec:holder}

{
\begin{lemma}[Push-forward under H\"older maps]\label{lem:push}
Let $(X,d_X)$ and $(Y,d_Y)$ be Polish metric spaces, let $p\in[1,\infty)$,
$\alpha\in(0,1]$, and set
\[
q\eqdef\max\{1,p\alpha\}.
\]
Assume that $T:X\to Y$ satisfies
\[
d_Y(T(x),T(x'))\le C_T\,d_X(x,x')^\alpha
\qquad\forall x,x'\in X.
\]
Then, for every $\mu,\nu\in\mathcal P(X)$ with finite $q$-th moments,
\begin{equation}\label{eq:push}
W_p(T_\#\mu,T_\#\nu)
\le
C_T\,W_q(\mu,\nu)^\alpha.
\end{equation}
\end{lemma}

\begin{proof}
If $p\alpha\ge1$, then $q=p\alpha$. Let $\pi^*\in\Pi(\mu,\nu)$ be optimal for $W_q$.
Then $(T\times T)_\#\pi^*\in\Pi(T_\#\mu,T_\#\nu)$ and
\[
W_p^p(T_\#\mu,T_\#\nu)
\le
\int d_Y(T(x),T(x'))^p\,d\pi^*
\le
C_T^p\int d_X(x,x')^{p\alpha}\,d\pi^*
=
C_T^p W_q(\mu,\nu)^{p\alpha}.
\]
Taking $p$-th roots gives~\eqref{eq:push}.

If $p\alpha<1$, then $q=1$. Let $\pi^*\in\Pi(\mu,\nu)$ be optimal for $W_1$.
Since $r\mapsto r^{p\alpha}$ is concave on $[0,\infty)$,
\[
\int d_X(x,x')^{p\alpha}\,d\pi^*
\le
\left(\int d_X(x,x')\,d\pi^*\right)^{p\alpha}
=
W_1(\mu,\nu)^{p\alpha}.
\]
Consequently
\[
W_p^p(T_\#\mu,T_\#\nu)
\le
C_T^p W_1(\mu,\nu)^{p\alpha},
\]
and taking $p$-th roots again gives~\eqref{eq:push}.
\qed\end{proof}
}

\begin{lemma}[H\"older constant of the graph map]\label{lem:graphmap}
For $h\in C^{0,\alpha}(\overline\Omega;\R^n)$, the map $T=(\id,h)$ is $\alpha$-H\"older with
\[
[T]_{C^{0,\alpha}}
\le
\sqrt{D^{2(1-\alpha)}+[h]_{C^{0,\alpha}}^2}.
\]
\end{lemma}

\begin{proof}
For $r=\abs{x-x'}\in[0,D]$, $r\le D^{1-\alpha}r^\alpha$, so
\[
\norm{T(x)-T(x')}_2^2
=
\abs{x-x'}^2+\abs{h(x)-h(x')}^2
\le
D^{2(1-\alpha)}r^{2\alpha}
+
[h]_{C^{0,\alpha}}^2 r^{2\alpha}.
\]
\qed\end{proof}

{
\begin{proposition}[Quantitative bound]\label{thm:holder}
Let $\Omega$ be bounded with diameter $D$, let
$h\in C^{0,\alpha}(\overline\Omega;\R^n)$, $\alpha\in(0,1]$, and let
$p\in[1,\infty)$. Set
\[
q\eqdef\max\{1,p\alpha\}.
\]
Then, for every finite point set $({\bar x}_j)_{j=1}^N\subset\Omega$,
\begin{equation}\label{eq:holder}
W_p\!\left(\mathbf M_N(h),\gamma_h\right)
\le
\sqrt{D^{2(1-\alpha)}+[h]_{C^{0,\alpha}}^2}\;
W_q(\sigma_N,\bar\lambda_\Omega)^\alpha.
\end{equation}
\end{proposition}

\begin{proof}
Since $\mathbf M_N(h)=T_\#\sigma_N$ and $\gamma_h=T_\#\bar\lambda_\Omega$ with
$T=(\id,h)$, the claim follows from Lemmas~\ref{lem:push} and~\ref{lem:graphmap}.
\qed\end{proof}
}

\subsubsection{Concrete rates}
\label{sec:rates}

{
Throughout this section set
\[
q\eqdef\max\{1,p\alpha\},
\qquad
C_h\eqdef\sqrt{D^{2(1-\alpha)}+[h]_{C^{0,\alpha}}^2}.
\]
Then~\eqref{eq:holder} says
\[
W_p(\mathbf M_N(h),\gamma_h)
\le
C_h\,W_q(\sigma_N,\bar\lambda_\Omega)^\alpha.
\]
}

{
\paragraph{Monte Carlo.}
Assume that $x_1^{(0)},\ldots,x_N^{(0)}$ are i.i.d.\ with law $\bar\lambda_\Omega$.
Since $\bar\lambda_\Omega$ is compactly supported, the Fournier--Guillin estimate
\cite{FournierGuillin2015} gives
\[
\mathbb E\!\left[W_q(\sigma_N,\bar\lambda_\Omega)^q\right]
\le
C
\begin{cases}
N^{-q/d}, & d>2q,\\[2mm]
N^{-1/2}\log(1+N), & d=2q,\\[2mm]
N^{-1/2}, & d<2q,
\end{cases}
\]
where $C$ depends on $d,q,\Omega$.
Since $\alpha\le1\le q$, Jensen's inequality gives
\[
\mathbb E\!\left[W_q(\sigma_N,\bar\lambda_\Omega)^\alpha\right]
\le
\left(
\mathbb E\!\left[W_q(\sigma_N,\bar\lambda_\Omega)^q\right]
\right)^{\alpha/q}.
\]
Therefore
\begin{equation}\label{eq:MC}
\mathbb E\!\left[W_p(\mathbf M_N(h),\gamma_h)\right]
\le
C_h C
\begin{cases}
N^{-\alpha/d}, & d>2q,\\[2mm]
N^{-\alpha/(2q)}\bigl(\log(1+N)\bigr)^{\alpha/q}, & d=2q,\\[2mm]
N^{-\alpha/(2q)}, & d<2q.
\end{cases}
\end{equation}
The constant $C$ may change from line to line.
}

{
\paragraph{Deterministic point sets and quasi-Monte Carlo.}
%
Tensor grids when $N=m^d$, and more generally near-optimal quantisers, satisfy
\begin{equation}\label{eq:quant_upper}
W_q(\sigma_N,\overline{\lambda}^r_{[0,1]^d})
\le
C_{d,q}N^{-1/d}.
\end{equation}
By using~\eqref{eq:holder} we obtain the
bound
\begin{equation}\label{eq:det_rate}
W_p(\mathbf M_N(h),\gamma_h)
\le
C_h C_{d,q}^{\alpha} N^{-\alpha/d}
\end{equation}
for such point sets.  

Classical low-discrepancy point sets, such as Halton, Sobol', or Faure sequences, are designed
to control star discrepancy and integration error for test functions of bounded
Hardy--Krause variation; see \cite{Niederreiter1992,Dick2014}. This does not automatically
control $W_q$, whose dual description for $q=1$ involves all Lipschitz test functions.
Hence a QMC point set may be used in~\eqref{eq:holder} only after a separate Wasserstein
estimate
\[
W_q(\sigma_N,\bar\lambda_\Omega)\le a_N
\]
has been proved. In that case~\eqref{eq:holder} immediately yields
\[
W_p(\mathbf M_N(h),\gamma_h)\le C_h a_N^\alpha.
\]

If points are first constructed on $[0,1]^d$ and transported to $\Omega$, the transfer of a
Wasserstein estimate is valid under a quantitative transport assumption. For example, if
$F:[0,1]^d\to\overline\Omega$ is Lipschitz and
$F_\#\overline{\lambda}^r_{[0,1]^d}=\bar\lambda_\Omega$, then
\[
W_q(F_\#\sigma_N,\bar\lambda_\Omega)
\le
\text{Lip}(F)\,W_q(\sigma_N,\overline{\lambda}^r_{[0,1]^d}).
\]
Arbitrary inverse-CDF constructions in higher dimension, or acceptance--rejection sampling,
do not by themselves preserve low-discrepancy or Wasserstein rates.
}

\subsubsection{Joint limit $(\tau,N)\to(0,\infty)$}
\label{sec:joint}

{
For $h\in L^p(\Omega;\R^n)$, the graph measure $\gamma_h$ is well-defined by
Definition~\ref{def:graph}. The triangle inequality gives
\begin{equation}\label{eq:twoscale}
W_p(\mathbf M_{N,\tau}(h),\gamma_h)
\le
W_p(\mathbf M_{N,\tau}(h),\gamma_{R_\tau h})
+
W_p(\gamma_{R_\tau h},\gamma_h).
\end{equation}
The second term satisfies, by Lemmas~\ref{lem:stab} and~\ref{lem:Rtau},
\begin{equation}\label{eq:outer_joint}
W_p(\gamma_{R_\tau h},\gamma_h)
\le
\text{Vol}(\Omega)^{-1/p}\norm{R_\tau h-h}_{L^p(\Omega)}
\xrightarrow[\tau\downarrow0]{}0.
\end{equation}
For each fixed $\tau>0$, the first term satisfies
\[
W_p(\mathbf M_{N,\tau}(h),\gamma_{R_\tau h})
=
W_p(\mathbf M_N(R_\tau h),\gamma_{R_\tau h})
\xrightarrow[N\to\infty]{}0
\]
by Theorem~\ref{thm:M2}, because $R_\tau h\in C_c^\infty(\Omega;\R^n)$.

Consequently, a diagonal joint limit always exists under regular distribution of the points.
Indeed, choose $\tau_k\downarrow0$ such that
\[
W_p(\gamma_{R_{\tau_k}h},\gamma_h)\le k^{-1}.
\]
For each $k$, choose $N_k$ so large that
\[
W_p(\mathbf M_{N,\tau_k}(h),\gamma_{R_{\tau_k}h})\le k^{-1}
\qquad\text{for all }N\ge N_k.
\]
After increasing $N_k$ if necessary, assume $N_k<N_{k+1}$ and define
$\tau_N\eqdef\tau_k$ whenever $N_k\le N<N_{k+1}$. Then $\tau_N\downarrow0$ and
\[
W_p(\mathbf M_{N,\tau_N}(h),\gamma_h)\to0.
\]
This argument gives convergence but no explicit rate.

A rate for a prescribed sequence $\tau_N\downarrow0$ requires a quantitative estimate of
\[
W_p(\mathbf M_{N,\tau}(h),\gamma_{R_\tau h})
\]
in terms of both $N$ and $\tau$. For example, if one uses the Lipschitz graph estimate
\[
W_p(\mathbf M_N(R_\tau h),\gamma_{R_\tau h})
\le
\sqrt{1+\text{Lip}(R_\tau h)^2}\,
W_p(\sigma_N,\bar\lambda_\Omega),
\]
then one must also control the possible blow-up of $\text{Lip}(R_\tau h)$ as $\tau\downarrow0$.
Such a bound is not available for arbitrary $h\in L^p(\Omega;\R^n)$ without additional
regularity.

For $h\in H^{-s}(\Omega;{\R}^n)\setminus L^1_{\mathrm{loc}}(\Omega;\R^n)$, the unregularised
graph measure $\gamma_h$ is not a well-defined Borel probability measure. Thus
\eqref{discrete M3}{} is structurally a statement about the fixed-$\tau$ regularised graph measure
$\gamma_{R_\tau h}$.
}

}

\subsection{Proof of {Theorem}
\ref{thm:approximation-trans self-attention Takashi}}
\label{app:thm:universal-graph-trans}

Let $\e>0$. Later, we choose a sufficiently small parameter $\e_1>0$ that depends on $\e$.

Target map $f$ has the form ${\widehat{\Gamma}}(\mu,(x,y))=(x,p(\mu,(x,y)))$,
where $p(\mu,(x,y))=\widehat p(\mu,x)$ is independent of the $y$-varialbe (We note 
that $\mu$ is a measure in the $\Omega\times [-L,L]^n$, so $p(\mu,(x,y))$ depends
on a measure defined using the $(x,y)$ variables. To clarify the analysis we consider below
$p(\mu,(x,y))$ as function of all variables $(\mu,(x,y))$.)

We first apply \citep[Theorem 1]{furuya2025transformers} to the map $\widehat p: \mathcal P(\overline\Omega\times [-L,L]^n) \times (\overline\Omega\times [-L,L]^n) \to [-L,L]^{n}$, then there exist a deep transformer $\hat{F}_{\xi_L} \diamond \hat{\Gamma}_{\theta_L} \diamond \ldots \diamond \hat{F}_{\xi_1} \diamond \hat{\Gamma}_{\theta_1}$ such that 
\beq\label{e1 bound}
\sup_{(\mu,(x,y)) \in K \times (\overline\Omega\times [-L,L]^n)} | p(\mu,(x,y))-
\hat{F}_{\xi_L} \diamond \hat{\Gamma}_{\theta_L} \diamond \ldots \diamond \hat{F}_{\xi_1} \diamond \hat{\Gamma}_{\theta_1}(\mu,(x,y))|\leq \varepsilon_1,
\eeq
where $K= P(\overline\Omega\times [-L,L]^n)$  and $\hat{\Gamma}_{\theta_\ell} : \mathcal P(\mathbb{R}^{{\hat d}_\ell}) \times \mathbb{R}^{{\hat d}_\ell} \to \mathbb{R}^{{\hat d}_\ell}$ and $\hat{F}_{\xi_\ell} : \mathbb{R}^{{\hat d}_\ell} \to \mathbb{R}^{{\hat d}_{\ell+1}}$, where 
${\hat d}_1 =d+ n$, ${\hat d}_{L+1} = n$, and ${\hat d}_\ell= 4(d+n)$, for $\ell=2,3,\dots,L$ and
${\hat d}_{\ell}'= 4(d+n)$  for $\ell=1,2,\dots,L+1$.
Here, the bound $4(d+n)$ for ${\hat d}_\ell$ and ${\hat d}_{\ell}'$ is the sum of the 
dimension $(d+n)$
of the domain $\overline\Omega\times [-L,L]^n$ of the map $\hat{\Gamma}_{\theta_\ell}$ and 3 times its co-domain $(d+n)$.
We point out that the bound for the width, $4(d+n)$, is larger than in
the corresponding results in \citep[Theorem 1]{furuya2025transformers}.
Let us now define the maps
$\Gamma_{\theta_\ell} : \mathcal P(\mathbb{R}^d \times \mathbb{R}^{d_\ell}) \times \mathbb{R}^d \times \mathbb{R}^{d_\ell} \to \mathbb{R}^d \times \mathbb{R}^{d_\ell}$ and $F_{\xi_\ell} : \mathbb{R}^d \times \mathbb{R}^{d_\ell} \to \mathbb{R}^d \times \mathbb{R}^{d_{\ell+1}}$,
where 
${d}_1 =d + n$, ${d}_{L+1} = d + n$, and ${d}_\ell= 4(d+n)$, for $\ell=2,3,\dots,L$ and
${d}_{\ell}'= 4(d+n)$  for $\ell=1,2,\dots,L+1$,
by setting
\beq
& & {\Gamma}_{\theta_1}\bigg( \mu, (x,y)\bigg)=\bigg(x, \hat{\Gamma}_{\theta_1}( \mu, (x,y))\bigg),
\eeq
and
\beq
& & {\Gamma}_{\theta_\ell}\bigg(\mu, (x,x',y)\bigg)=\bigg(x, \hat{\Gamma}_{\theta_\ell}(\mu, (x',y))\bigg),
\eeq
for $\ell=2,3,\dots,L$ and
\beq
& & \hat{F}_{\xi_\ell} \bigg(x,x',y\bigg)=\bigg(x,\hat{F}_{\xi_\ell}(x',y)\bigg)
\eeq
for $\ell=1,2,\dots,L-1$ and finally, we define $F_{\xi_L}$ by scaling
$\hat F_{\xi_L}$ by a factor $\frac L{L+\e_1}$ and capping it with the smooth function
$\Psi$. This defines
the function 
\beq
& & F_{\xi_L}\bigg( \mu, (x,x',y)\bigg)=\bigg(x,\Psi(\frac L{L+\e_1}\cdot\,
\hat F_{\xi_L}(\mu, (x',y)))\bigg).
\eeq

Recall that the function $ \widehat p(\mu,(x,y))$ obtains
in values on the interval $[-L,L]^n$. Thus we see that  when the approximation \eqref{e1 bound} is valid, we have $\hat F_{\xi_\ell}( (\mu, (x',y)))\in [-(L+\e_1),(L+\e_1)]^n$.
 Then,
 $$
 \frac L{L+\e_1}\hat F_{\xi_\ell}( (\mu, (x',y)))\in [-L,L]^n,$$
 and thus 
 $$
 \Psi(\frac L{L+\e_1}\hat F_{\xi_\ell}( (\mu, (x',y)))
 =\frac L{L+\e_1}\hat F_{\xi_\ell}( (\mu, (x',y)))\in [-L,L]^n
 $$
When $\e_1$ is small enough,
 the above construction implies that the putative context function
\ba
{\Gamma}^{\mathrm{put}}(\mu,(x,y))
&=&\bigg(x\,,\, 
\Psi(\frac L{L+\e_1}\cdot
\hat F_{\xi_L} \diamond \hat \Gamma_{\theta_L} \diamond \ldots \diamond \hat F_{\xi_1} \diamond \hat \Gamma_{\theta_1}(\mu,(x,y)))\bigg)\\
&=& F_{\xi_L} \diamond \Gamma_{\theta_L} \diamond \ldots \diamond F_{\xi_1} \diamond \Gamma_{\theta_1}(\mu,(x,y))
\ea
satisfies 
\begin{equation}
    \label{pre Wasserstein limit}
\sup_{(\mu,(x,y)) \in \mathcal P (\overline\Omega\times [-L,L]^n)\times (\overline\Omega\times [-L,L]^n) } 
\bigg|{\Gamma}^{\mathrm{put}}(\mu, (x,y))-{\widehat{\Gamma}}(\mu,(x,y))\bigg| \leq \varepsilon.
\end{equation}

Recall that for $y_0\in [-L,L]^n$, $w_{y_0}:\Omega\times \R^n\to \Omega\times \R^n$
is the map $w_{y_0}(x,y)=(x,y_0)$ and $w_{y_0}^*:F\to F\circ w_{y_0}$  a warping operator.
We define now
\beq
{\Gamma}^{\mathrm{tran}}(\mu, (\cdot,\cdot))=
w_{y_0}^*{\Gamma}^{\mathrm{put}}(\mu,  (\cdot,\cdot))
=
{\Gamma}^{\mathrm{put}}(\mu,  w_{y_0}(\cdot,\cdot)),
\eeq
that is, 
\beq
{\Gamma}^{\mathrm{tran}}(\mu,(x,y))=
{\Gamma}^{\mathrm{put}}(\mu,w_{y_0}(x,y_0))
={\Gamma}^{\mathrm{put}}(\mu,(x,y_0))=
(x\, ,\,{\Gamma}^{\mathrm{put},(y)}(\mu,(x,y_0))),
\eeq
We recall that the target function ${\widehat{\Gamma}}(\mu,(x,y))=(x,\widehat p(\mu,x))$
is independent of the $y$-variable. Moreover, the $x$-coordinates
of ${\Gamma}^{\mathrm{put}}(\mu,(x,y))$ and ${\widehat{\Gamma}}(\mu,(x,y))$
are both equal to $x$.
Thus, formula \eqref{pre Wasserstein limit} implies that 
\beq
    \label{pre Wasserstein limitB}
    & &\hspace{-10mm}
 \sup_{(\mu,(x,y)) \in \mathcal P (\overline\Omega\times [-L,L]^n)\times (\overline\Omega\times [-L,L]^n) } 
\bigg|{\Gamma}^{\mathrm{put},(y)}(\mu, (x,y))-\widehat p(\mu,x)\bigg|
     \\
&=&  
    \sup_{(\mu,(x,y)) \in \mathcal P (\overline\Omega\times [-L,L]^n)\times (\overline\Omega\times [-L,L]^n) } 
\bigg|{\Gamma}^{\mathrm{put}}(\mu, (x,y))-{\widehat{\Gamma}}(\mu,(x,y_0))\bigg| 
 \leq \varepsilon
 \eeq
and as the supremum in \eqref{pre Wasserstein limitB} is taken on over all $y\in [-L,L^n]$
we trivially have for the individual value $y_0$ that
\beq
    \label{pre Wasserstein limitC}
\sup_{(\mu,x) \in \mathcal P (\overline\Omega\times [-L,L]^n)\times \overline\Omega) } 
\bigg|{\Gamma}^{\mathrm{put},(y)}(\mu, (x,y_0))-\widehat p(\mu,x)\bigg| \leq \varepsilon.
\eeq
By definition of ${\Gamma}^{\mathrm{tran}}(\mu, (x,y))$, this implies
\beq
    \label{pre Wasserstein limitD}
\sup_{(\mu,(x,y)) \in \mathcal P (\overline\Omega\times [-L,L]^n)\times (\overline\Omega\times [-L,L]^n) } 
\bigg|{\Gamma}^{\mathrm{tran}}(\mu, (x,y))-{\widehat{\Gamma}}(\mu,(x,y))\bigg| \leq \varepsilon.
\eeq
This and Lemma \ref{lem:uniform-map-to-W1} imply \eqref{sup G tran formula} and \eqref{first coordinate is x}.

\qed

\medskip

Theorem~\ref{thm:approximation-trans self-attention Takashi} means that we can approximate operators mapping graph measures to graph measures by traditional transformers in the product space $\overline{\Omega}\times [-L,L]^n$. Later, we see that by composing the traditional transformers by ${\bf M}_N$ and  ${\bf F}_\varepsilon$,
we can use those to approximate operators $A$ from functions to functions.

\commented{Let's move this below to the section where we consider approximation of nonlinear operatoar $A$
{\color{red}
Let
$$
A : C(\overline\Omega; [-L,L]^{n}) \to C(\overline\Omega;\mathbb{R}^{n})
$$
be contentious in the weak sense that $h_n \rightharpoonup h$ i.e.,
\[
\int_\Omega \varphi(x) h_n(x) dx \to \int_\Omega \varphi(x) h(x) dx, \ \forall \varphi \in C(\overline\Omega; \mathbb{R}^{n})
\]
implies that $A(h_n)\rightharpoonup A(h)$.
}
\Takashi{I think this is good point. In operator learning, the continuity in the strong sense is often assumed for the target operator. Catchphrase : "transformer can relax the continuity assumption in operator learning". We maybe good to provide the examples of $A$ that not strong continuous, but weak continuous.}

Note that $A$ has the form 
\[
A(g) =  {\bf F} \circ  {\bf M} \circ A \circ {\bf F} \circ  {\bf M}(g)  = {\bf F} \circ G \circ {\bf M}(g), \quad g \in C(\overline\Omega; \mathbb{R}^{h})
\]
where $G : \mathcal{M}(\Omega \times [-L,L]^{n}) \to \mathcal{M}(\mathbb{R}^d \times \mathbb{R}^{n})$ is defined by 
\[
G(\mu) := {\Gamma}(\mu, \cdot)_\sharp \mu, \quad \mu \in  \mathcal{M}(\Omega \times [-L,L]^{n})
\]
where ${\Gamma}: \mathcal{M}(\Omega \times [-L,L]^{n}) \times \Omega \times [-L,L]^{n} \to \mathbb{R}^{d} \times \mathbb{R}^{n}$ is defined by 
\[
{\Gamma}(\mu, (x,y)) := (x, A({\bf F}(\mu))(x)).
\]

\begin{theorem}[Main result, Informal]
There exists a sequence $\{\varepsilon_m\}$, $\{N_m\}$, $\{ G_{{\rm tran}, m} \}_{m \in \mathbb{N}}$ of transformers such that for each $g \in C(\Omega;[-L,L]^h)$
\[
\lim_{m \to \infty}
{\bf F}_{\varepsilon_m} \circ G_{{\rm tran}, m} \circ {\bf M}_{N_m}(h) = A(h)
\]
i.e., point-wise convergence in $C(\Omega;[-L,L]^h)$. \Takashi{This may be a limitation. This is due to ${\bf F}_\varepsilon \to {\bf F}$ and ${\bf M}_N \to {\bf M}$.}
\end{theorem}
\begin{proof}
Use \eqref{f mollified 2C} and \eqref{discrete M2}. Note that 
\[
{\Gamma}_\varepsilon(\mu, (x,y)) := (x, A({\bf F}_\varepsilon(\mu))(x)).
\]
is continuous w.r.t. $\mu$ in Wasserstein topology, then we can apply Proposition~\ref{prop:universal-graph-trans}.
\qed\end{proof}
}


\commented{
\section{Proofs in Section~\ref{sec:5}}
}

\section{Proofs in Section~\ref{sec: Expressivity}}

\subsection{Proof of {Lemma} \ref{lem. f and G are continuous}}
\label{app:lem. f and G are continuous}

(i) Let us consider the map 
\beq
G_\eta(\mu)=F_\eta(\mu,\cdot)_\#\mu,
\eeq
where 
\beq\label{map f}
F_\eta(\mu,(x,y))=(x,A({\tilde S}_{\eta}(\mu))(x)),\quad F_\eta(\mu,(\cdot ,\cdot)): \overline \Omega\times \R
\to \overline \Omega\times \R.
\eeq
We observe that $F_\eta(\mu,(x,y))$ is independent of the value of $y$.

By Lemma \ref{lem: S eta continuous}, the
 The map $\tilde S_{\eta}:{{{\mathcal P}}}(\Omega\times [-L,L]^n)\to  L^2(\overline\Omega;{[-L_1,L_1]}^n)$ 
is continuous. This and the assumption \eqref{A assumption3} on the map $A$ imply that $A\circ S_{\eta}:{{{\mathcal P}}}(\Omega\times [-L,L]^n)\to 
C^1(\overline \Omega;{[-L,L]}^n)$ is continuous. 
This implies that the map
${\Gamma}_\eta:{{{\mathcal P}}}(\overline\Omega\times [-L,L]^n)\times ({\overline\Omega}\times [-L,L]^n)\to {\overline\Omega}\times [-L,L]^n$ be the map, given in \eqref{f eta def}, is continuous.
This proves the claim (i).

(ii) 
Next, we consider the function $\mu \to A({\tilde S}_{\eta}(\mu))(\cdot)$ as a map ${{{\mathcal P}}}(\overline\Omega\times [-L,L]^n)\to C^1(\overline\Omega)\subset \hbox{Lip}(\overline\Omega).$

When $\mu_j\to \mu$ as $j\to \infty$ in ${{{\mathcal P}}}(\overline\Omega\times [-L,L]^n)$,
we see that $S_{\eta}(\mu_j)\to S_{\eta}(\mu)$ in $H^{2r}(\Omega)$.
Hence, $A({\tilde S}_{\eta}(\mu_j))\to A({\tilde S}_{\eta}(\mu))$  in $C^1(\overline \Omega)$, and hence
\beq
\lim_{j\to\infty}F_\eta(\mu_j,(\cdot ,\cdot))=F_\eta(\mu,(\cdot ,\cdot))\quad\hbox{in }  L^\infty(\overline \Omega\times [-L,L]^n)\cap \hbox{Lip}(\overline \Omega\times [-L,L]^n).
\eeq
This and {Lemma} \ref{lem: triangular Wasserstein1} with $F_j=F_\eta(\mu_j,\cdot)$ imply that
\beq
\lim_{j\to\infty}G_\eta(\mu_j)=\lim_{j\to\infty}F_\eta(\mu_j,\cdot)_\#\mu_j=G_\eta(\mu)\quad\hbox{in }  {{{\mathcal P}}}(\overline\Omega\times [-L,L]^n)
\eeq
Hence, the map
\beq
G_\eta:{{{\mathcal P}}}(\overline\Omega\times [-L,L]^n)\to {{{\mathcal P}}}(\overline\Omega\times [-L,L]^n)
\eeq
is continuous.
This proves the claim (ii).\hfill $\square$

\subsection{Proof of Lemma \ref{lem: formula G for gamma h}}
\label{app:lem: formula G for gamma h}

We show below that the definition of the map ${\Gamma}_\eta$ in  
\eqref{f eta def} and  the map $G_\eta$ in   \eqref{f eta def}
yield the
claim.

To prove the claim, let us assume that $h\in C(\overline \Omega)$. Recall that $A: C^k(\overline \Omega)
\to  C^1(\overline \Omega)$ is a continuous map.

Assume that $h\in {\Xspace}$ and $\gamma_h$ be 
the measure
\begin{align}\label{integration formula 2}
   \int_{{\overline\Omega} \times \mathbb{R}} \phi(x,y) \, \mathrm{d}\gamma_h(x,y)= \tfrac1 {\text{Vol}(\Omega)}\int_{{\overline\Omega}} \phi(x,h(x)) \, \mathrm{d}\lambda_{\Omega}(x).
\end{align}%

Observe that for  $\phi(x,y) \in C_b(\overline\Omega \times \mathbb{R})$, we have by \eqref{integration formula 2}
\beq\label{key1}
\bra G_\eta(\gamma_h), \phi\ket&=&\bra  {\Gamma}_\eta(\gamma_h,(\cdot,\cdot))_\#\gamma_h ), \phi\ket
\\
&=&\bra  \gamma_h , \phi({\Gamma}_\eta(\gamma_h,(\cdot,\cdot)))\ket
\\
&=& \int_{{\overline\Omega} \times \mathbb{R}} \phi({\Gamma}_\eta(\gamma_h,(x,y)))\, d\gamma_h (x,y)
\\
&=& \int_{{\overline\Omega} \times \mathbb{R}} \phi({\Gamma}_\eta(\gamma_h,(x,y)))\,{\bf 1}(y)\, d\gamma_h (x,y)
\\
&=& \int_{{\overline\Omega} \times \mathbb{R}} \phi(x,A({\tilde S}_{\eta}(\gamma_h))(x))\,{\bf 1}(y)\, d\gamma_h (x,y)
\\
&=& \tfrac1 {\text{Vol}(\Omega)}\int_{{\overline\Omega}} \phi(x,A({\tilde S}_{\eta}(\gamma_h))(x))\,  \mathrm{d}\lambda_{\Omega}(x)
\\
&=&\bra \gamma_{A({\tilde S}_{\eta}(\gamma_h))}, \phi\ket,
\label{key4}
\eeq
where
${\bf 1}(y)=1$ is the constant function having the value 1.
Hence, $ G_\eta(\gamma_h)=\gamma_{A({\tilde S}_{\eta}(\gamma_h))}.$

\commented{
As
$\eta\to 0$, we have by \eqref{minimum continuity3} and  \eqref{integration formula 2}
\beq
& &\lim_{\eta\to 0} \int_{{\overline\Omega}} \phi(x,A({\tilde S}_{\eta}(\gamma_h))(x))\,  \mathrm{d}\overline \lambda_{\Omega}(x)\\
&=&\lim_{\eta\to 0} \int_{{\overline\Omega}} \phi(x,A(h)(x))\,  \mathrm{d}\overline \lambda_{\Omega}(x)\\
&=&\lim_{\eta\to 0} \int_{{\overline\Omega}} \phi(x,y)\,   \mathrm{d}\gamma_{A(h)}(x,y)
\\
&=&\bra \gamma_{A(h)},\phi\ket.
\eeq

Using these, we see that 
\beq
\lim_{\eta\to 0+}G_\eta(\gamma_h)=\lim_{j\to\infty}F_\eta(\gamma_h,\cdot)_\#\gamma_h=\gamma_{A(h)}\quad\hbox{in }  {{{\mathcal P}}}(\overline\Omega\times [-L,L]^n)
\eeq

This gives a draft for the proof of the claim 5 above.}
\hfill $\square$

\subsection{Properties of $S_{\eta,K(\eta)}$} 
\label{app:S-regularizers}

We recall that we use two parameters $\eta>0$ and $K\in \mathbb N$. Also $k\in \mathbb Z_+$ and ${\color{black}{r}}>\frac d2+k$ be an even integer, so
that $H^{\color{black}{r}}(\Omega)\subset C^k(\overline \Omega)$ by Sobolev's embedding theorem \citep{AdamsFournier}. Let
 $A_D=(1-\Delta)$,  where $\Delta$ is the Laplacian with Dirichlet boundary condition,
considered as an unbounded selfadjoint operator in $L^2(\Omega)^n$ with domain $\mathcal D(A_D)=H^2(\Omega;\R^n)\cap H^1_0(\Omega;\R^n)$. Let ${\mathcal L}=(I-\Delta)^r$ be an unbounded self-adjoint operator in $L^2(\Omega;\R^n)$.
Let us denote $
    {\color{black}H_D^r(\Omega)}:={\color{black}{\mathcal D}(A_D^{r/2})}$.
Let  $\psi_1, \psi_2, \dots \in H^{r}_D(\Omega)$ be a complete orthonormal basis in $L^2(\Omega;\R^n)$ of eigenfunctions satisfying $(I-\Delta)\psi_j={{\lambda}_j}\psi_j$ and $\psi_j|_{\p \Omega}=0$.
We solve the minimization problem, 
\begin{align}
\label{eq:def-S}
\hspace{-10mm}S_{\eta,K}(\gamma)
=\argmin_{\tilde h\in {\color{black}H_D^r(\Omega)}} \bigg(
  \eta{\mathcal R}_r(\tilde h)+\sum_{k=1}^K\bigg|
  V_0\int_{{\overline\Omega}\times [-L,L]^n} \psi_k(x){{y}}\, d\gamma(x,y)-
  \int_{{\overline\Omega}} \psi_k(x)\,\tilde h(x) dx   \bigg|^2
  \bigg) ,
\end{align}
where $V_0=\text{Vol}(\Omega)$ and we use  the regularizer ${\mathcal R}_r(\tilde h)=\|\mathcal L^{{\color{black}{r}}/2}\tilde h\|^2_{
L^2({\Omega};\R^n)}$ Note that ${\mathcal R}_r(\tilde h)= \|\tilde h\|^2_{
H^r({\Omega};\R^n)}$ for $\tilde h\in {\color{black}H_D^r(\Omega)}$ and
as  $H^s_0(\Omega)\subset {\color{black}H_D^s(\Omega)}$,
the dual spaces satisfy ${\color{black}H_D^s(\Omega)}^*\subset H^{-s}(\Omega)$.

\begin{lemma}\label{S-regularizers}
If $K(\eta)\to \infty$ as $\eta\to 0$, the regularizers $S_{\eta}=S_{\eta,K(\eta)}$ satisfy  for all $h\in H^{-s}(\Omega;\R^n)$
\beq
\lim_{\tau\to 0}\lim_{\eta\to 0}\|S_{\eta,K(\eta)}(\gamma_{R_{\tau}h})-h\|_{H^{-s}(\Omega)}=0.
    \eeq
\end{lemma}

We see below in \eqref{S for gamma f1} that  
in the case when $\gamma=\gamma_h$ we have  
\begin{align}
   S_{\eta,K}(\gamma_h)=\argmin_{\tilde h\in {H^r_D}(\Omega;\R^n)}  \bigg(
   \eta\|(I-\Delta)^{{{r}}/2}\tilde h\|^2_{
L^2({\Omega};\R^n)}+\sum_{k=1}^K\bigg|
   \int_{{\overline\Omega}} \psi_k(x) (h(x)-\tilde h(x)) \,dx   \bigg|^2
  \bigg).
\end{align}


We point out that $ S_{\eta,K}(\mu)$ is a solution of a quadratic minimization problem. Thus, the map
$\mu\to S_{\eta,K}(\mu)$ is a linear map and thus, $$h\to \gamma_h\to S_{\eta,K}(\gamma_h)$$ 
is a linear operator that can be considered as a (linear) neural operator.

Next, we show that the regularization method defines a continuous operator $S_{\eta} :\ \gamma \to S_{\eta}(\gamma)$, e.g. ${{{\mathcal P}}}({\overline\Omega} \times [-L,L]^n) \to {H^r_D}(\Omega;\R^n)$. In particular, this is useful if we assume that $A:{H^r_D}(\Omega;\R^n)\subset C^k(\overline \Omega;{[-L,L]}^n)\to  C^1(\overline \Omega;{[-L,L]}^n)$ is continuous. We need to study the limit when first $K\to \infty$ and then $\eta \to 0$.

Let 
$$
  P_Kh(x)=\sum_{k=1}^K \int_{{\overline\Omega}} \psi_k(x') h(x') \, dx' \, \psi_k(x)
$$
be an orthogonal projection to $K$ first eigenvalues. Then, denoting $V_0={\text{Vol}(\Omega)}$,
\beq
  S_{\eta,K}(\gamma)&=&V_0(\eta(I-\Delta)^{{{r}}}+P_K)^{-1}\bigg(\sum_{k=1}^K
  \bigg[\int_{{\overline\Omega}\times [-L,L]^n} \psi_k(x') {{y}}
    \, d\gamma(x',y)\bigg] \psi_k\bigg)
  \nonumber\\
  &=& \sum_{k=1}^K V_0(\eta{{\lambda}^r_k}+1)^{-1}
  \bigg[\int_{{\overline\Omega}\times [-L,L]^n} \psi_k(x') {{y}}
    \, d\gamma(x',y)\bigg] \psi_k 
  \nonumber\\
  &=& \sum_{k=1}^K 
  \bigg\bra\gamma \ ,\   \psi_k(x) {{y}}
    \bigg\ket V_0(\eta{{\lambda}^r_k}+1)^{-1} \psi_k 
    .
\eeq
We denote
\beq \label{QK def}
  Q_K[\gamma](x) = V_0\sum_{k=1}^K \bigg[\int_{{\overline\Omega}\times [-L,L]^n}
  \psi_k(x'){{y}}\, d\gamma(x',y)\bigg] \psi_k(x)\in C^k({\overline\Omega};\R^n).
\eeq
Then, we see that
\beq \label{S function and QK}
   S_{\eta,K}(\gamma)
   &=&(\eta (I - \Delta)^{{{r}}} + I)^{-1}Q_K[\gamma] .
\eeq

When $\gamma = \gamma_h$, we have  
\beq \label{S for gamma f1}
  \int_{{\overline\Omega}\times [-L,L]^n} \psi_k(x'){{y}}\, d\gamma_{h}(x',y)
  = \frac 1{\text{vol}(\Omega)}\int_{{\overline\Omega}} \psi_k(x') {h}(x')\, d\lambda(x')
\eeq
and thus we obtain
\beq
\nonumber
  S_{\eta,K}(\gamma_h)
   &=&\sum_{k=1}^K (\eta{{\lambda}^r_k}+1)^{-1}
   \bigg[\int_{{\overline\Omega}} \psi_k(x') {h}(x')\, d\lambda(x')\bigg] \psi_k
   \\ \label{S for gamma f1D}
   &=&(\eta(I-\Delta)^{{{r}}}+I)^{-1}P_K h.
\eeq
Below, we also denote
\beq
\nonumber
  S_{\eta,K}h
    \label{S for gamma f1C}
   =(\eta(I-\Delta)^{{{r}}}+I)^{-1}P_K h
\eeq
\commented{We emphasize that  for all $w\in L^2(\Omega)$
\beq
  \|(\eta(I-\Delta)^{{{r}}}+I)^{-1}w-w\|_{L^2(\Omega)}\le (1-(1+\eta)^{-1})\|w\|_{L^2(\Omega)}
  =\frac \eta{1+\eta}\|w\|_{L^2(\Omega)}
\eeq
and hence we can think that $(\eta(I-\Delta)^{{{r}}}+I)^{-1}\approx I$ when $\eta$ is small.}

 \commented{
We remarks that if we assume that $f\in H^{-s_0}(\Omega)$ with  $-s_0>-s$ and use that ${{\lambda}^r_j^r}\sim Cj^{2/d},$ where $\Omega\subset \R^d$,
we can obtain 
\beq
\|S_{\eta,K}(\gamma_f)-f\|_{H^{-s}(\Omega)}\le \omega_{s,s_0}(K,\eta)\|f\|_{H^{-s_0}(\Omega)}
\eeq 
with some explicit $\omega_{s,s_0}(K,\eta)$.}

\commented{
\begin{lemma}\label{S-regularizers A}
The regularizers $S_{\eta}=S_{\eta,K(\eta)}$ satisfy

\begin{enumerate}
 \item  Assume that $K(\eta)\to \infty$ as $\eta\to 0$. Then for all $h\in H^{-s}(\Omega;{\R}^n)$ 
\beq
\lim_{\tau\to 0}\lim_{\eta\to 0}\|S_{\eta,K(\eta)}(\gamma_{R_{\tau}h})-h\|_{H^{-s}(\Omega)}=0.
    \eeq

Also, for fixed $\tau>0$
        \beq
\lim_{\eta\to 0}\lim_{K\to \infty}\|S_{\eta,K}(\gamma_{R_{\tau}h}))-R_{\tau}h\|_{C^k(\overline \Omega)}=0
    \eeq
    and as $\tau\to 0,$
        \beq
\lim_{\tau\to 0}\lim_{\eta\to 0}\lim_{K\to \infty}\|S_{\eta,K}(\gamma_{R_{\tau}h}))-h\|_{H^{-s}(\Omega)}=0
    \eeq
and
            \beq
\lim_{\tau\to 0}\lim_{K\to \infty}\lim_{\eta\to 0} \|S_{\eta,K}(\gamma_{R_{\tau}h}))-h\|_{H^{-s}(\Omega)}=0.
    \eeq

{
\item \footnote{\color{red} We have add details of proof here. Note that this results is not needed if
we use $\Psi$ and the operator ${\tilde S}_{\eta}=\Psi\circ {S}_{\eta}$ that caps the function values
to a box. The key point is thar  we assume that $f\in H^{r}_0(\Omega)$ with  $r>r'>k+ d/2$ and use that ${{\lambda}^r_j^r}\sim Cj^{2r/d},$ where $\Omega\subset \R^d$,
we can obtain 
\beq
\|S_{\eta,K}(\gamma_f)-f\|_{C^k(\Omega)}\le \omega_{r,r'}(K,\eta)\|f\|_{H^{r}(\Omega)}
\eeq 
with some explicit $\omega_{r,r'}(K,\eta)$. Note that $1<c_1<L_1/ L$ 
}
We can choose $K(\eta)=K(\eta)$ to go to zero so rapidly that as $\eta\to 0$, that for all $h\in H^{r}_0(\Omega;{\R}^n)$ 
\beq
\lim_{\eta\to 0} \bigg\|S_{\eta,K(\eta)}(h))-h\bigg\|_{H^{r}_0(\Omega;{\R}^n)}=0.
\eeq
}
\end{enumerate}
\end{lemma}

\begin{proof}
   The claim 1 follows  by using the above formulas for $S_{\eta,K(\eta)}$.
\end{proof}

}

\subsubsection{Quantitative growth of $K(\eta)$}
\label{ssec: K eta Lipschitz}

We define the spectral norm $\|\cdot\|_{r}$ on $\mathcal D(L^{r/2})$ by
\[
  \|h\|_r^2:=\sum_{j=1}^\infty{{\lambda}^r_j}\,|\langle h,\psi_j\rangle_{L^2}|^2.
\]

We now choose $K$ explicitly. Set
\begin{equation}
\label{eq:explicit-K-eta}
    K(\eta)
    :=
    \left\lceil \eta^{-d/(4r)}\right\rceil,
    \qquad
    0<\eta\le 1.
\end{equation}
Then $K(\eta)\to\infty$ as $\eta\downarrow0$, and
\begin{equation}
\label{eq:eta-lambda-K-goes-zero}
    \eta{\lambda}^r_{K(\eta)}
    \longrightarrow 0
    \qquad\text{as }\eta\downarrow0.
\end{equation}
Indeed, using \eqref{eq:lambda-upper-bound} and
$K(\eta)\le 2\eta^{-d/(4r)}$ for $0<\eta\le1$,
\[
    \eta{\lambda}^r_{K(\eta)}
    \le
    C_{\mathcal L}\eta K(\eta)^{2r/d}
    \le
    C_{\mathcal L}2^{2r/d}\eta
    \left(\eta^{-d/(4r)}\right)^{2r/d}
    =
    C_{\mathcal L}2^{2r/d}\eta^{1/2}
    \to0 .
\]

\begin{lemma}
\label{S-regularizers A}
Let $K(\eta)$ be given by \eqref{eq:explicit-K-eta}. Then:

\begin{enumerate}
\item
For every $h\in L^2(\Omega;\mathbb R^n)$,
\[
    \lim_{\eta\downarrow0}
    \|S_{\eta,K(\eta)}h-h\|_{L^2(\Omega;\mathbb R^n)}
    =
    0 .
\]

\item
For every $h\in H_D^r(\Omega;\mathbb R^n)$,
\[
    \lim_{\eta\downarrow0}
    \|S_{\eta,K(\eta)}h-h\|_{H_D^r(\Omega;\mathbb R^n)}
    =
    0 .
\]

\item
If $f\in L^2(\Omega;[-L,L]^n)$, then
\[
    \lim_{\eta\downarrow0}
    \|S_{\eta,K(\eta)}(\gamma_f)-f\|_{L^2(\Omega;\mathbb R^n)}
    =
    0 .
\]
If in addition $f\in H_D^r(\Omega;\mathbb R^n)$, then
\[
    \lim_{\eta\downarrow0}
    \|S_{\eta,K(\eta)}(\gamma_f)-f\|_{H_D^r(\Omega;\mathbb R^n)}
    =
    0 .
\]
\end{enumerate}
\end{lemma}

\begin{proof}
It is enough to prove the scalar-valued case; the vector-valued case follows
by summing over the components.

Let
\[
    h=\sum_{j=1}^\infty a_j\psi_j,
    \qquad
    a_j:=\langle h,\psi_j\rangle_{L^2(\Omega)} .
\]
Then
\[
    S_{\eta,K}h-h
    =
    -\sum_{j=1}^K
    \frac{\eta{\lambda}^r_j}{1+\eta{\lambda}^r_j}a_j\psi_j
    -
    \sum_{j>K}a_j\psi_j .
\]

First assume $h\in L^2(\Omega)$. Since the two sums are orthogonal in
$L^2(\Omega)$,
\[
\begin{aligned}
    \|S_{\eta,K}h-h\|_{L^2(\Omega)}^2
    &=
    \sum_{j=1}^K
    \left(
        \frac{\eta{\lambda}^r_j}{1+\eta{\lambda}^r_j}
    \right)^2
    |a_j|^2
    +
    \sum_{j>K}|a_j|^2 .
\end{aligned}
\]
Taking $K=K(\eta)$ and using monotonicity of ${\lambda}^r_j$,
\[
\begin{aligned}
    \sum_{j=1}^{K(\eta)}
    \left(
        \frac{\eta{\lambda}^r_j}{1+\eta{\lambda}^r_j}
    \right)^2
    |a_j|^2
    &\le
    \bigl(\eta{\lambda}^r_{K(\eta)}\bigr)^2
    \sum_{j=1}^{K(\eta)} |a_j|^2
    \\
    &\le
    \bigl(\eta{\lambda}^r_{K(\eta)}\bigr)^2
    \|h\|_{L^2(\Omega)}^2 .
\end{aligned}
\]
By \eqref{eq:eta-lambda-K-goes-zero}, this term tends to zero. Also,
since $K(\eta)\to\infty$ and $\sum_{j=1}^\infty |a_j|^2<\infty$,
\[
    \sum_{j>K(\eta)} |a_j|^2\to0 .
\]
Therefore
\[
    \|S_{\eta,K(\eta)}h-h\|_{L^2(\Omega)}\to0 .
\]

Now assume $h\in H_D^r(\Omega)=\mathcal D(\mathcal L^{1/2})$. Then
\[
    \sum_{j=1}^\infty {\lambda}^r_j |a_j|^2<\infty .
\]
Again by orthogonality, now after applying $\mathcal L^{1/2}$,
\[
\begin{aligned}
    \|S_{\eta,K}h-h\|_{H_D^r(\Omega)}^2
    &=
    \sum_{j=1}^K
    \left(
        \frac{\eta{\lambda}^r_j}{1+\eta{\lambda}^r_j}
    \right)^2
    {\lambda}^r_j |a_j|^2
    +
    \sum_{j>K}{\lambda}^r_j |a_j|^2 .
\end{aligned}
\]
With $K=K(\eta)$,
\[
\begin{aligned}
    \sum_{j=1}^{K(\eta)}
    \left(
        \frac{\eta{\lambda}^r_j}{1+\eta{\lambda}^r_j}
    \right)^2
    {\lambda}^r_j |a_j|^2
    &\le
    \bigl(\eta{\lambda}^r_{K(\eta)}\bigr)^2
    \sum_{j=1}^{K(\eta)} {\lambda}^r_j |a_j|^2
    \\
    &\le
    \bigl(\eta{\lambda}^r_{K(\eta)}\bigr)^2
    \|h\|_{H_D^r(\Omega)}^2 .
\end{aligned}
\]
This tends to zero by \eqref{eq:eta-lambda-K-goes-zero}. The tail term
also tends to zero because
\[
    \sum_{j=1}^\infty {\lambda}^r_j |a_j|^2<\infty
    \qquad\text{and}\qquad
    K(\eta)\to\infty .
\]
Hence
\[
    \|S_{\eta,K(\eta)}h-h\|_{H_D^r(\Omega)}\to0 .
\]

Finally, if $f\in L^2(\Omega;[-L,L]^n)$, then
\[
    S_{\eta,K(\eta)}(\gamma_f)
    =
    S_{\eta,K(\eta)}f .
\]

The preceding $L^2$ and $H_D^r$ convergence results therefore apply
directly to $S_{\eta,K(\eta)}(\gamma_f)$.

\qed
\end{proof}

\begin{corollary}
\label{cor:S-eta-compatible-main}
Let $s\ge0$. Suppose that 
$h\in H^{-s}(\Omega;\mathbb R^n)$. Then
\[
    \lim_{\tau\downarrow0}
    \lim_{\eta\downarrow0}
    \|S_{\eta,K(\eta)}(\gamma_{R_\tau h})-h\|_{H^{-s}(\Omega;\mathbb R^n)}
    =
    0 .
\]
\end{corollary}

\begin{proof}
We have seen above that 
$R_\tau h\in L^2(\Omega;\R^n)$ for all sufficiently small $\tau>0$
and  that $R_\tau h\to h$ in
$H^{-s}(\Omega;\mathbb R^n)$ as $\tau\downarrow0$.

For fixed $\tau>0$, Proposition~\ref{S-regularizers A}
gives
\[
    S_{\eta,K(\eta)}(\gamma_{R_\tau h})
    \to
    R_\tau h
    \qquad\text{in }L^2(\Omega;\mathbb R^n).
\]
Since $s\ge0$, the embedding
\[
    L^2(\Omega;\mathbb R^n)
    \hookrightarrow
    H^{-s}(\Omega;\mathbb R^n)
\]
is continuous. Hence
\[
    S_{\eta,K(\eta)}(\gamma_{R_\tau h})
    \to
    R_\tau h
    \qquad\text{in }H^{-s}(\Omega;\mathbb R^n).
\]
Therefore
\[
\begin{aligned}
    \lim_{\eta\downarrow0}
    \|S_{\eta,K(\eta)}(\gamma_{R_\tau h})-h\|_{H^{-s}}
    &=
    \|R_\tau h-h\|_{H^{-s}} .
\end{aligned}
\]
Letting $\tau\downarrow0$ gives the claim.
\qed
\end{proof}

Below, we choose 
    $K(\eta)$ to be as in \eqref{eq:explicit-K-eta} and we denote 
\beq
S_{\eta}=S_{\eta,K(\eta)}.
\eeq

\subsubsection{Properties of the map $S_{\eta}$}


We point out that when $h \in C(\overline \Omega)$, we have
\begin{align}\label{appl integration formula}
  \int_{{\overline\Omega} \times [-L,L]^n} |\tilde h(x)-y|^2
  \, \mathrm{d}\gamma_{h}(x,y)
  &=\tfrac1 {\text{Vol}(\Omega)} \int_{{\overline\Omega}} |\tilde h(x)-{h}(x)|^2 \, \mathrm{d}\lambda_{\Omega}(x)
\nonumber\\
  &= \tfrac1 {\text{Vol}(\Omega)}\int_{{\overline\Omega}} (|{h}(x)|^2 - 2 \tilde h(x) \cdot {h}(x)
  + |\tilde h(x)|^2) \, \mathrm{d}\lambda_{\Omega}(x) .
\end{align}

\begin{lemma} \label{lem: S eta continuous}
For a fixed $\eta>0$, the map $S_{\eta}:{{{\mathcal P}}}(\Omega\times [-L,L]^n)\to  H^{{2r}}(\Omega;{\R}^n)$ is continuous.
\end{lemma}

\begin{proof}
By \eqref{S function and QK}, we have
\beq\label{S function and QK c1}
\nonumber
  S_{\eta}(\gamma)= S_{\eta,{K(\eta)}}(\gamma)
  &=& {\text{Vol}(\Omega)}\sum_{k=1}^{K(\eta)} (\eta{{\lambda}^r_k}+I)^{-1}
  \bigg[\int_{{\overline\Omega}\times [-L,L]^n} \psi_k(x') {{y}}
    \, d\gamma(x',y)\bigg] \psi_k
  \\
  &=& (\eta (I - \Delta)^{{{{r}}}} + I)^{-1} Q_{K(\eta)}[\gamma] .
\eeq
Here, the map $Q_K$ given in \eqref{QK def} is a continuous map $Q_K :\ {{{\mathcal P}}}(\Omega \times \R^n) \to C^k(\overline\Omega;\R^n) \subset L^2(\overline\Omega;{[-L,L]}^n)$. As $(\eta (I - \Delta)^{{{r}}} + I)^{-1} :\ L^2(\overline\Omega;\R^n) \to H^{{2r}}(\Omega;{\R}^n)$ is continuous, we obtain the claim. \hfill $\square$
\end{proof}

\begin{lemma} \label{lem: claim 3}
Let $h \in C(\overline\Omega;{[-L,L]}^n)$. Then $\lim_{\eta\to 0} S_{\eta}(\gamma_h) = h$ in $L^2(\overline\Omega;{[-L,L]}^n)$. Moreover, when $h \in {H_D^r(\Omega)}\cap C(\overline \Omega;{[-L,L]}^n)$, we have $\lim_{\eta\to 0} S_{\eta}(\gamma_h) = h$ in ${H_D^r(\Omega)}$ and therefore in $C^k(\overline \Omega;\R^n)$.
\end{lemma}

\begin{proof}
By \eqref{S for gamma f1},
\beq \label{S for gamma f12}
  S_{\eta,K(\eta)}(\gamma_h) &=&\sum_{k=1}^{K(\eta)} (\eta{{\lambda}^r_k}+I)^{-1}
  \bigg[\int_{{\overline\Omega}} \psi_k(x')\cdot h(x')\, d\lambda(x')\bigg] \psi_k
  \nonumber\\
  &=&\sum_{k=1}^{K(\eta)}(\eta{{\lambda}^r_k}+1)^{-1}
  \bra h, \psi_k\ket_{L^2(\Omega)}  \psi_k
\eeq
and by using Lebesgue dominated convergence theorem we see that $S_{\eta,K(\eta)}(\gamma_h)\to h$ in $L^2(\Omega)$ as $\eta\to 0$. 

When $h \in {H_D^r(\Omega)}$, we use that the domain of the 
square root $ \sqrt{\mathcal L}$ of the selfadjoint operator 
${\mathcal L}$ is the space
${H_D^r(\Omega)}$, see \cite{LionsMagenesV1}. 
We see by applying
spectral theory, \cite{LionsMagenesV1},
that $\lim_{\eta\to 0} S_{\eta}(\gamma_h) = h$ in ${H_D^r(\Omega)}$
that is, in ${H_D^r(\Omega)}$. \hfill $\square$
\end{proof}

\subsection{Proof of Theorem~\ref{main thm}}
\label{app:main thm}

Let $\tau >0$, $\eta>0$,  and $h\in  \hbox{cl}_{H^{-s}(\Omega;\R^n)}(C(\overline \Omega;
[-L,L]^n))$

We see that  $u_\tau=R_\tau(h)\in {H_D^r(\Omega)}\cap C(\overline \Omega;
[-L,L]^n).$  
Thus, when $\tau> 0$ is small enough,
we see that 
\beq
S_{\eta}(\gamma_{h_\tau})= (\eta(I-\Delta)^{{{r}}}+I)^{-1}P_{K(\eta)}({h_\tau})
\to {h_\tau}
\eeq
in ${H_D^r(\Omega)}\subset  C^k(\overline \Omega;\R^n)$ as $\eta\to 0$
and thus
\beq
\lim_{\eta\to 0}\|S_{\eta}(\gamma_{h_\tau})\|_{L^\infty(\Omega;\R^n)}=\|h_\tau\|_{L^\infty( \Omega;\R^n)}
\eeq
Hence, when $\eta>0$ is small enough, we have
\beq
S_{\eta}(\gamma_{h_\tau})\in C^k(\overline \Omega;[-(L+L_1)/2,(L+L_1)/2]^n.
\eeq
Thus, as $\psi(y_i)=y_i$ for $y\in [-(L+L_1)/2,(L+L_1)/2]$, it holds for sufficiently small  $\eta>0$
that
\beq
{\tilde S}_{\eta}(\gamma_{h_\tau})=\Psi\circ (\eta(I-\Delta)^{{{r}}}+I)^{-1}P_{K(\eta)}({h_\tau})=S_{\eta}(\gamma_{h_\tau}).
\eeq

We can approximate $G_{\Gamma_\eta}$ by 
softmax-based traditional
transformers $G_{\Gamma_{\eta,m}^{\mathrm{tran}}}$ that are given by
Theorem \ref{thm:approximation-trans self-attention Takashi}. The limit $m\to \infty$ can be included in the theorem
by the formula \eqref{sup G tran formula}.
This gives us
\beq
{\bf F}_\varepsilon\circ G_{\eta} \circ 
{\bf M}_{N,\tau} (h)
=\lim_{m\to \infty}{\bf F}_\varepsilon\circ G_{\Gamma_{\eta,m}^{\mathrm{tran}}} \circ 
{\bf M}_{N,\tau} (h)
\eeq
in $C(\overline \Omega;
[-L,L]^n)$.

For fixed $\eta>0$
\beq
A_{\eta,\tau}(h):=\lim_{\varepsilon\to 0}\lim_{N\to \infty}{\bf F}_\varepsilon\circ G_{\Gamma_\eta} \circ 
{\bf M}_{N,\tau} (h)
\eeq
in $H^{-s'}(\Omega;\R^n)$ where the limit $\lim_{N\to \infty} {\bf M}_{N,\tau} (h)=\gamma_{h_\tau}$ converges in ${{{\mathcal P}}}(\overline\Omega\times [-L,L]^n)$.

Then we see that the limit
$\lim_{N\to \infty} S_\eta({\bf M}_{N,\tau} (h))$ converges in ${H_D^r(\Omega)}$. 
The definition of the map $G_{\Gamma_\eta}$ in  
\eqref{f eta def} 
and
{Lemma} \ref{lem: triangular Wasserstein1}  imply that 
$\lim_{N\to \infty}  G_{\Gamma_\eta} \circ {\bf M}_{N,\tau} (h)$ converges
to $\gamma_b$  in ${{{\mathcal P}}}(\overline\Omega\times [-L,L]^n)$ when $b=A({\tilde S}_{\eta}(\gamma_{h_\tau}))$.
Then, 
by  \eqref{f mollified 2C}
to 
\beq
\lim_{\varepsilon\to 0} \bigg(\lim_{N\to \infty} {\bf F}_\varepsilon\circ G_{\Gamma_\eta} \circ {\bf M}_{N,\tau} (h)\bigg)=A({\tilde S}_{\eta}(\gamma_{h_\tau})),
\eeq
where the limit takes place in $H^{-s'}(\Omega)$.

Then,
\beq
A({\tilde S}_{\eta}(\gamma_{h_\tau}))
&=&
A(S_{\eta}(\gamma_{h_\tau}))\\
&=&A (\Psi\circ (\eta(I-\Delta)^{{{r}}}+I)^{-1}P_{K(\eta)}({h_\tau})).
\eeq
Moreover,
$$
A_\eta(h_\tau)=A( (\eta(I-\Delta)^{{{r}}}+I)^{-1}P_{K(\eta)}h_\tau).
$$
As $$\lim_{\eta\to 0} (\eta(I-\Delta)^{{{r}}}+I)^{-1}P_{K(\eta)}h_\tau=h_\tau$$ converges in ${H_D^r(\Omega)}\subset C^k(\overline \Omega),$
this implies that 
$$
\lim_{\eta\to 0} A(\Psi\circ (\eta (I-\Delta)^{{{r}}}+1)^{-1}P_{K(\eta)}h_\tau))=A(h_\tau)
$$ 
converges in $H^{-s'}(\Omega)$.
This implies that
$$\lim_{\tau\to 0} A(h_\tau)=A(h)
$$ converges in $H^{-s'}(\Omega)$.

The proof in the case when  $h\in C^r_c( \Omega;\R^n)\cap C(\overline \Omega;
[-L,L]^n)$ follows similarly as the above arguments for $h_\tau$. \hfill $\square$

\subsection{Proof of Proposition \ref{prop: Map from Omega to D}}
\label{proof of prop: Map from Omega to D}
Recall that we suppose \beq
A:  L^2(\Omega;{[-L,L]}^n)
\to  C^1(\overline D;{[-L,L]}^n).
\eeq
is continuous,  possibly nonlinear, function.
Here, $\Omega\subset \R^d$ and  $D\subset \R^{d'}$ are open bounded sets  with $C^\infty$-smooth boundary, or a cube.  
Below, we use in $\overline\Omega\times [-L,L]^n$ the coordinates $(x,y)$ and
in $\overline D\times [-L,L]^n$ the coordinates $(x',y')$.
On $X=(\overline\Omega\times [-L,L]^n)\times (\overline D\times [-L,L]^n)$
we consider the maps
$\rho_1((x,y),(x',y'))=(x,y)$ and $\rho_2((x,y),(x',y'))=(x',y')$, and denote  $\pi_2(x',y')=y'$.

Our next goal is to express the map $\widehat G_\eta(\mu,x')=A({\tilde S}_{\eta}(\mu))(x')$ using the same self-attention construction as in the measure-to-measure setting.

Let $\mu\in {{{\mathcal P}}}(\overline\Omega\times [-L,L]^n)$ and $x'\in D$, and define the product measure
\begin{equation}
\tilde \mu := \mu \otimes \delta_{(x',0)}
\in {{{\mathcal P}}}\big((\overline\Omega\times [-L,L]^n)\times (D\times [-L,L]^n)\big).
\end{equation}

Define the lifted map $\tilde {\Gamma}_\eta: {{{\mathcal P}}}(X)\times X\to X$,
\begin{equation}\label{key formula 1}
\tilde {\Gamma}_\eta(\tilde \mu,((x,y),(x',y')))
:= \big((x,y),\, {\Gamma}_\eta((\rho_1)_\#\tilde \mu,(x',y'))\big),
\end{equation}
where ${\Gamma}_\eta:{{{\mathcal P}}}\big(\overline\Omega\times [-L,L]^n)\times (D\times [-L,L]^n)\to (D\times [-L,L]^n)$ is
the mapping
\begin{equation}\label{key formula 2}
{\Gamma}_\eta(\mu,(x',y')) = \big(x', A({\tilde S}_{\eta}(\mu))(x')\big).
\end{equation}
We emphasize that here the context function ${\Gamma}_\eta(\mu,(x',y'))$ is independent of the $y'$ variable
and $\tilde {\Gamma}_\eta(\tilde \mu,((x,y),(x',y')))$ is independent of the $y$ and $y'$ variables. However,
${\Gamma}_\eta(\mu,(x',y'))$ is not a function graph transformer, but an auxiliary map just used in this proof.

Let $\tilde G_\eta: {{{\mathcal P}}}(X)\to  {{{\mathcal P}}}(X)$ be the self-attention based transformer with the context function $\tilde {\Gamma}_\eta,$
that is,\footnote{We emphasize that here we use only self-attention with a context map $\tilde {\Gamma}_\eta$ in a product space, that is, we have raised the dimension of
sets where the measures are considered.}
\begin{equation}
\tilde G_\eta(\tilde \mu) := \tilde {\Gamma}_\eta(\tilde \mu,\cdot)_\# \tilde \mu\in {{{\mathcal P}}}\big(X).
\end{equation}
Then
\beq\label{pre-cross}
(\rho_2)_\# \tilde G_\eta(\tilde \mu)
&=&(\rho_2)_\# \bigg(\tilde {\Gamma}_\eta(\tilde \mu,((\cdot,\cdot),(\cdot,\cdot)))
_\#\tilde \mu\bigg)
\\\nonumber
&=&\bigg(\rho_2\circ \tilde {\Gamma}_\eta(\tilde \mu,((\cdot,\cdot),(\cdot,\cdot)))\bigg)
_\#\tilde \mu
\\\nonumber
&=&\bigg({\Gamma}_\eta((\rho_1)_\#\tilde \mu,\rho_2((\cdot,\cdot),(\cdot,\cdot)))\bigg)
_\#\tilde \mu
\\\nonumber
&=&\bigg({\Gamma}_\eta((\rho_1)_\#\tilde \mu,(\cdot,\cdot))\bigg)
_\#((\rho_2)_\#\tilde \mu)
\\ \nonumber
&=&{\Gamma}_\eta(\mu,\cdot)_\# \delta_{(x',0)}\in {{{\mathcal P}}}\big(\overline D\times [-L,L]^n).
\eeq
Since ${\Gamma}_\eta(\mu,(x',y'))$ is independent of $y'$, we obtain
\begin{equation}
{\Gamma}_\eta(\mu,\cdot)_\# \delta_{(x',0)}
= \delta_{(x',A({\tilde S}_{\eta}(\mu))(x'))},
\end{equation}
and hence
\begin{equation}
\widehat G_\eta(\mu,x')
= \langle {\Gamma}_\eta(\mu,\cdot)_\# \delta_{(x',0)}, \pi_2 \rangle
= A({\tilde S}_{\eta}(\mu))(x').
\end{equation}
\qed

\paragraph{Further observations.} The map $(\mu,x')\mapsto \widehat G_\eta(\mu,x')=A({\tilde S}_{\eta}(\mu))(x')$ admits the factorization
\begin{equation}
(\mu,x')
\;\mapsto\;
\mu\otimes \delta_{(x',0)}
\;\mapsto^{\tilde G_\eta}\;
\tilde G_\eta(\tilde \mu)
\;\mapsto^{(\rho_2)_\#}\;
(\rho_2)_\# \tilde G_\eta(\tilde \mu)
\;\mapsto^{
\langle \cdot,\pi_2\rangle} A({\tilde S}_{\eta}(\mu))(x'),
\end{equation}
i.e., as a composition of push-forward of projections and a self-attention map.
If we use $\mu=\gamma_h$, where $h\in C(\overline \Omega)$, and $x'\in \overline D$,
and use for $\eta=0$ the notation $S_0(\gamma_h)=\lim_{\eta\to 0}S_\eta(\gamma_h))=h$,
we have
\begin{align}
(h,x')\mapsto^{{\bf M}\times Id} & (\gamma_h,x')
\;\mapsto\;
\gamma_h \otimes \delta_{(x',0)}
\;\mapsto^{\tilde g_0}\;
\tilde g_0(\gamma_h \otimes \delta_{(x',0)})
\;\mapsto^{(\rho_2)_\#}\\ \nonumber
&\;
(\rho_2)_\# \tilde g_0(\gamma_h \otimes \delta_{(x',0)})
\;\mapsto^{
\langle \cdot,\pi_2\rangle} A(h)(x').
\end{align} 
That is, 
$$
 A(h)(x')=\bra (\rho_2)_\# \tilde g_0(\gamma_h \otimes \delta_{(x',0)}),\pi_2  \ket
$$
or equivalently,
$$
 A(h)(x')=\bra \tilde g_0(\gamma_h \otimes \delta_{(x',0)})\,,\,(\pi_2\circ \rho_2)  \ket.
$$
or, 
$$
 A(h)(x')=\bra \tilde g_0(\gamma_h \otimes \delta_{(x',0)})\,,\,y'  \ket.
$$


Let us also do a computation analogous  to \eqref{pre-cross} for the product
measure $\mu\otimes \nu$, where
$\mu\in {{{\mathcal P}}}(\overline\Omega\times [-L,L]^n)$
and
 $\nu\in {{{\mathcal P}}}(\overline D\times [-L,L]^n)$.
We see that 
\beq\label{cross-attention ideas}
(\rho_2)_\# \tilde G_\eta(\mu\otimes \nu)
&=&(\rho_2)_\# \bigg(\tilde {\Gamma}_\eta(\mu\otimes \nu,((\cdot,\cdot),(\cdot,\cdot)))
_\#(\mu\otimes \nu)\bigg)
\\\nonumber
&=&\bigg(\rho_2\circ \tilde {\Gamma}_\eta((\mu\otimes \nu),((\cdot,\cdot),(\cdot,\cdot)))\bigg)
_\#(\mu\otimes \nu)
\\\nonumber
&=&\bigg({\Gamma}_\eta((\rho_1)_\#(\mu\otimes \nu),\rho_2((\cdot,\cdot),(\cdot,\cdot)))\bigg)
_\#(\mu\otimes \nu)
\\\nonumber
&=&\bigg({\Gamma}_\eta((\rho_1)_\#(\mu\otimes \nu),(\cdot,\cdot))\bigg)
_\#((\rho_2)_\#(\mu\otimes \nu))
\\ \nonumber
&=&{\Gamma}_\eta(\mu,\cdot)_\# \nu\in {{{\mathcal P}}}\big(\overline D\times [-L,L]^n).
\eeq

Next, we prove that
\beq\label{positional encoding Psi2}
(\rho_2)_\#  \tilde {\Gamma}_\eta\big({\bf M}_N(h)\otimes \nu,\cdot \big)_\#\big({\bf M}_N(h)\otimes \nu\big)=
  \sum_{k=1}^K \frac 1K\delta_{(x'_k,y_k')},\quad y_k' = A(h_{\eta,N})(x'_k).
\eeq We
recall that   $\nu=\sum_{k=1}^K \frac 1K\delta_{(x'_k,0)}$ and ${\bf M}_N(h)=\sum_{j=1}^N \frac 1N\delta_{(x_j,h(x_j))}$. Then, for
any $\phi\in C(\overline D\times [-L,L]^n)$
\beq\label{positional encoding Psi2B}
& &\bra (\rho_2)_\#  \tilde {\Gamma}_\eta\big({\bf M}_N(h)\otimes \nu,\cdot \big)_\#\big({\bf M}_N(h)\otimes \nu\big),\phi\ket
\nonumber\\
& &=
\bra  \tilde {\Gamma}_\eta\big({\bf M}_N(h)\otimes \nu,\cdot \big)_\#\big({\bf M}_N(h)\otimes \nu\big),\phi\circ \rho_2\ket
\nonumber\\
& &=
\bra  \big({\bf M}_N(h)\otimes \nu\big),\phi\circ \rho_2\circ \tilde {\Gamma}_\eta\big({\bf M}_N(h)\otimes \nu,\cdot \big)\ket
\nonumber\\
& &=
\bra  \big({\bf M}_N(h)\otimes \nu\big),\phi
\big(x', A({\tilde S}_{\eta}({\bf M}_N(h)))(x')\big)\ket
\nonumber\\
& &=
  \sum_{k=1}^K \frac 1K\phi(x'_k, A({\tilde S}_{\eta}({\bf M}_N(h)))(x'_k)).
\eeq
This proves \eqref{positional encoding Psi2}. By taking limit $\eta\to 0$, we obtain 
\beq\label{positional encoding Psi}
  \lim_{\eta\to 0} (\rho_2)_\#  \tilde {\Gamma}_\eta\big(\gamma_h\otimes \nu ,\cdot\big)_\#\big(\gamma_h\otimes \nu\big)=
 \sum_{k=1}^K \frac 1K\delta_{(x'_k,y_k)},\quad y_k = A(h)(x'_k),\
\nu=\sum_{k=1}^K \frac 1K\delta_{(x'_k,0)}.
\eeq

\subsection{Cross-attention as block-constrained product-measure self-attention}
\label{app:product-measure-cross-attention}

The product-measure construction above gives a self-attention representation on the product space
\[
X=(\overline \Omega\times [-L,L]^n)\times (\overline D\times [-L,L]^n).
\]
It is useful to compare this product-space self-attention mechanism with the cross-attention modules
that are often used in practical query-decoder architectures. The point of the comparison is that
standard cross-attention is not a different kind of measure map: it can be written as a product-measure
self-attention layer with a special block structure in the query, key, value, and output matrices.
However, the converse is not true: a general product-measure self-attention layer is a larger class.

We write
\[
s=(x,y)\in \overline \Omega\times [-L,L]^n,\qquad
r=(x',y')\in \overline D\times [-L,L]^n,
\]
and denote a product token by
\[
\xi=(s,r)=((x,y),(x',y'))\in X.
\]
Let $\tilde \mu$ be a measure on $X$. The attention part of a single product-space self-attention
head, suppressing the graph-preserving projection used in \eqref{Gamma theta}, can be written as
\begin{equation}\label{eq:product-attention-head}
 {\rm Att}^{X}_{\tilde Q,\tilde K,\tilde V}(\tilde \mu,\xi)
 =
 \frac{
 \displaystyle \int_X
 \exp\bigg(\frac{1}{\sqrt{d_{\rm att}}}
 \big\langle \tilde Q\xi,\tilde K\xi'\big\rangle\bigg)
 \tilde V\xi'\, d\tilde \mu(\xi')
 }{
 \displaystyle \int_X
 \exp\bigg(\frac{1}{\sqrt{d_{\rm att}}}
 \big\langle \tilde Q\xi,\tilde K\xi''\big\rangle\bigg)
 d\tilde \mu(\xi'')
 }.
\end{equation}
With respect to the splitting $\xi=(s,r)$, a general product-space attention head has matrices of the
form
\begin{equation}\label{eq:general-product-blocks}
 \tilde Q=\begin{bmatrix}Q_s&Q_r\end{bmatrix},\qquad
 \tilde K=\begin{bmatrix}K_s&K_r\end{bmatrix},\qquad
 \tilde V=\begin{bmatrix}V_s&V_r\end{bmatrix}.
\end{equation}
Thus the attention score between $\xi=(s,r)$ and $\xi'=(s',r')$ may depend on all four variables
$s,r,s',r'$. In the discrete product measure
\[
\tilde \mu_{N,K}={\bf M}_N(h)\otimes \nu_K,
\qquad
\nu_K=\frac{1}{K}\sum_{k=1}^K \delta_{(z_k,0)},
\]
this corresponds to self-attention over the $NK$ tokens
\[
\xi_{ik}=\big((x_i,h(x_i)),(z_k,0)\big),
\qquad i=1,\dots,N,\quad k=1,\dots,K.
\]
A fully general product-space layer may therefore attend from $(i,k)$ to $(j,\ell)$ for all
$i,j,k,\ell$.

We now recover standard query-to-input cross-attention from \eqref{eq:product-attention-head}.  We
choose the block-constrained matrices
\begin{equation}\label{eq:cross-attention-blocks}
 \tilde Q_{\rm ca}=\begin{bmatrix}0&Q_r\end{bmatrix},\qquad
 \tilde K_{\rm ca}=\begin{bmatrix}K_s&0\end{bmatrix},\qquad
 \tilde V_{\rm ca}=\begin{bmatrix}V_s&0\end{bmatrix}.
\end{equation}
Then
\[
\tilde Q_{\rm ca}\xi=Q_r r,
\qquad
\tilde K_{\rm ca}\xi'=K_s s',
\qquad
\tilde V_{\rm ca}\xi'=V_s s'.
\]
For a product measure $\tilde \mu=\mu\otimes \nu$, where $\mu$ is a measure on
$\overline \Omega\times [-L,L]^n$ and $\nu$ is a probability measure on
$\overline D\times [-L,L]^n$, substitution into \eqref{eq:product-attention-head} gives
\begin{align}\label{eq:cross-attention-as-product-attention}
 {\rm Att}^{X}_{\rm ca}(\mu\otimes \nu,(s,r))
 &=
 \frac{
 \displaystyle \int \!\! \int
 \exp\bigg(\frac{1}{\sqrt{d_{\rm att}}}
 \big\langle Q_r r,K_s s'\big\rangle\bigg)
 V_s s'\, d\mu(s')d\nu(r')
 }{
 \displaystyle \int \!\! \int
 \exp\bigg(\frac{1}{\sqrt{d_{\rm att}}}
 \big\langle Q_r r,K_s s''\big\rangle\bigg)
 d\mu(s'')d\nu(r'')
 }
 \\ \label{eq:cross-attention-as-product-attentionII}
 &=
 \frac{
 \displaystyle \int
 \exp\bigg(\frac{1}{\sqrt{d_{\rm att}}}
 \big\langle Q_r r,K_s s'\big\rangle\bigg)
 V_s s'\, d\mu(s')
 }{
 \displaystyle \int
 \exp\bigg(\frac{1}{\sqrt{d_{\rm att}}}
 \big\langle Q_r r,K_s s''\big\rangle\bigg)
 d\mu(s'')
 }.
\end{align}
The integration over the attended query coordinate $r'$ cancels. The right hand side of
\eqref{eq:cross-attention-as-product-attention} is exactly the usual cross-attention map from the
query token $r$ to the input or memory measure $\mu$.
This identity is a layerwise statement for factorized product measures. The block-constrained
input, query, and cross-attention modules considered below preserve this factorization when their
output projections are placed in the corresponding input or query blocks. In contrast, an
unrestricted product-space self-attention layer need not preserve a product form.

\commented{{\color{red}The same cancellation can be seen at the level of finite tokens. Let
$s_j=(x_j,h(x_j))$ and $r_k=(z_k,0)$, and set
\[
C_{kj}=\frac{1}{\sqrt{d_{\rm att}}}\big\langle Q_r r_k,K_s s_j\big\rangle .
\]
The product-space attention logit from $\xi_{ik}=(s_i,r_k)$ to $\xi_{j\ell}=(s_j,r_\ell)$ is
\[
L_{(i,k),(j,\ell)}
=\frac{1}{\sqrt{d_{\rm att}}}
\big\langle \tilde Q_{\rm ca}\xi_{ik},\tilde K_{\rm ca}\xi_{j\ell}\big\rangle
=C_{kj}.
\]
Hence the product attention matrix is a duplicated block matrix: the logit depends on $k$ and $j$, but
not on $i$ or $\ell$. Consequently,
\begin{align}\label{eq:finite-cross-attention-cancellation}
 {\rm Att}^{X}_{\rm ca}(\tilde \mu_{N,K},\xi_{ik})
 &=
 \frac{\displaystyle \sum_{j=1}^N\sum_{\ell=1}^K
 \exp(C_{kj})V_s s_j}{\displaystyle \sum_{j=1}^N\sum_{\ell=1}^K
 \exp(C_{kj})}
 \\
 &=
 \sum_{j=1}^N
 \frac{\exp(C_{kj})}{\sum_{a=1}^N \exp(C_{ka})}V_s s_j .
\end{align}
This is the standard finite-dimensional cross-attention update. The common factor $1/(NK)$ from
$\tilde\mu_{N,K}$ has been suppressed in \eqref{eq:finite-cross-attention-cancellation}, since it
cancels between numerator and denominator. The update is also independent of the input-fiber
index $i$, so the same query update is replicated along the fiber
\[
\{(s_i,r_k):i=1,\dots,N\}.
\]
For the single-query fiber ${\bf M}_N(h)\otimes\delta_{(z_k,0)}$, pairing the product-space output
with $\tilde\pi_2$ recovers the output value associated with $z_k$, as in \eqref{cost function}. For
the full multi-query measure ${\bf M}_N(h)\otimes\nu_K$, the projected output measure has the form
\[
(\rho_2)_\#\tilde\mu^+=\frac1K\sum_{k=1}^K \delta_{(z_k,y_k)}.
\]
Pairing this entire measure with $\pi_2$, or equivalently pairing $\tilde\mu^+$ with $\tilde\pi_2$,
gives only the average $K^{-1}\sum_{k=1}^K y_k$. Thus individual values are recovered by evaluating
the single-query fibers, by reading the atoms at the corresponding query coordinates, or equivalently
by testing against localized functions of $z$.

To obtain the usual decoder update, one also constrains the output projection. Namely, take
\begin{equation}\label{eq:cross-attention-output-block}
 \tilde W_{\rm ca}=\begin{pmatrix}0\\ W_r\end{pmatrix},
 \qquad
 \tilde W_{L,\rm ca}=\begin{pmatrix}I_s&0\\0&L_r\end{pmatrix}.
\end{equation}
Then the product-token update has the form
\begin{equation}\label{eq:cross-attention-update}
 s^+=s,
 \qquad
 r^+=L_r r+W_r {\rm Att}^{X}_{\rm ca}(\mu\otimes\nu,(s,r)).
\end{equation}
Thus the input or memory block is unchanged, while the query block is updated by attention against
$\mu$. If $r=(z,y)$ and one wants to preserve the query coordinate $z$, one imposes the finer block
structure
\[
L_r=\begin{pmatrix}I_z&0\\0&L_y\end{pmatrix},
\qquad
W_r=\begin{pmatrix}0\\W_y\end{pmatrix},
\]
so that $z^+=z$ and only the $y$-component is changed. This condition alone does not force the
updated $y$-component to be independent of the incoming value $y$: indeed,
\[
y^+=L_y y+W_y {\rm Att}^{X}_{\rm ca}(\mu\otimes\nu,(s,(z,y))),
\]
and the attention weights may still depend on $y$ through $Q_r(z,y)$. If the context map is intended
to have the target structure in \eqref{key formula 1 copy}--\eqref{Gamma formula 1B} on all query
tokens, not only on the initialized slice $y=0$, impose in addition
\[
Q_r=\begin{bmatrix}Q_z&0\end{bmatrix},
\qquad
L_y=0,
\]
and impose the same ``ignore $y$'' or overwrite structure on any subsequent pointwise MLP or query
block. Under these stronger constraints, the updated $y$-component depends on the input marginal
$(\rho_1)_\#\tilde\mu$ and on the query coordinate $z$, but not on the incoming value $y$. If query
tokens are always initialized with $y=0$, the weaker block structure above realizes the desired update
on that input slice, but it does not define a $y$-independent context map on all of
$\overline D\times[-L,L]^n$.}}

The same block calculation also identifies other familiar attention modules as special product-space
self-attention layers. With the output projection placed in the corresponding block, one obtains
\begin{align*}
\hbox{input self-attention:}\quad
&(\tilde Q,\tilde K,\tilde V)=
(\begin{bmatrix}Q_s&0\end{bmatrix},\begin{bmatrix}K_s&0\end{bmatrix},\begin{bmatrix}V_s&0\end{bmatrix}),\\
\hbox{query self-attention:}\quad
&(\tilde Q,\tilde K,\tilde V)=
(\begin{bmatrix}0&Q_r\end{bmatrix},\begin{bmatrix}0&K_r\end{bmatrix},\begin{bmatrix}0&V_r\end{bmatrix}),\\
\hbox{query-to-input cross-attention:}\quad
&(\tilde Q,\tilde K,\tilde V)=
(\begin{bmatrix}0&Q_r\end{bmatrix},\begin{bmatrix}K_s&0\end{bmatrix},\begin{bmatrix}V_s&0\end{bmatrix}),\\
\hbox{input-to-query cross-attention:}\quad
&(\tilde Q,\tilde K,\tilde V)=
(\begin{bmatrix}Q_s&0\end{bmatrix},\begin{bmatrix}0&K_r\end{bmatrix},\begin{bmatrix}0&V_r\end{bmatrix}).
\end{align*}
The first two choices update the input block and query block, respectively, and the last two choices
give the two possible directions of cross-attention. 
\commented{{\color{red}For factorized inputs $\mu\otimes\nu$, the
corresponding block-constrained output projections preserve factorization: input self-attention sends
$\mu\otimes\nu$ to $\mu^+\otimes\nu$, query self-attention sends it to $\mu\otimes\nu^+$, and
query-to-input cross-attention sends it to $\mu\otimes\nu^+(\mu)$. This factorization-preserving
property is not shared by a general unrestricted product-space self-attention layer.}}

This yields the following practical hierarchy in terms of computation cost. Full product-measure self-attention on
${\bf M}_N(h)\otimes \nu_K$ uses $NK$ product tokens and allows all couplings
$(i,k)\leftrightarrow (j,\ell)$; its direct attention cost is $O(N^2K^2)$. The fiberwise product
self-attention described in Appendix \ref{app: training transformers architecture} instead evaluates the $K$ measures
\[
{\bf M}_N(h)\otimes \delta_{(z_k,0)},\qquad k=1,\dots,K,
\]
with the same product-space transformer and can treat $k$ as a batch index; this is still literal
product-space self-attention for each query and has attention cost $O(KN^2)$. The block-constrained
cross-attention form \eqref{eq:cross-attention-blocks}--\eqref{eq:cross-attention-as-product-attentionII}
collapses
the product measure by the calculation above. After any chosen input encoding of $\mu$, the
query-to-input cross-attention sublayer has cost $O(NK)$. This count does not include the input
encoder cost, for instance $O(N^2)$ per full input self-attention layer, nor any query self-attention
cost, for instance $O(K^2)$ per full query self-attention layer.

Thus standard cross-attention is a special case of product-measure self-attention, but only after
imposing block constraints on the weight matrices.
The unrestricted product-measure transformer used in the representation theorem is a larger class.
Therefore, the approximation results for unrestricted product-measure self-attention do not by
themselves prove an approximation theorem for every architecture built only from the restricted
cross-attention blocks. \commented{{\color{red}Rather, the calculation shows that the common encoder--decoder architecture
is a block-constrained and fiber-collapsed implementation of the marginal-query structure appearing in
\eqref{key formula 1 copy}--\eqref{Gamma formula 1B}. With this interpretation, the loss
\eqref{cost function} can be used for either the fiberwise product implementation or the corresponding
pure query-to-input cross-attention implementation, replacing the single-query product-space readout
by the equivalent quantity
\begin{equation}\label{eq:cross-attention-loss-form}
L_{\rm ca}(\theta)
=
\sum_{p=1}^P\sum_{k=1}^K
\big|\Gamma_{\theta,{\rm ca}}^{{\rm tran},(y)}({\bf M}_N(h_p);z_k)-A(h_p)(z_k)\big|^2.
\end{equation}
The notation in \eqref{eq:cross-attention-loss-form} emphasizes that, in this pure cross-attention
case, the output at $z_k$ depends on the input empirical graph measure ${\bf M}_N(h_p)$ and the
individual query coordinate $z_k$, but not on unrelated query points in the same batch. If query
self-attention layers are added, the decoder output may also depend on the full query measure
$\nu_K$; in that case the same pointwise loss is used, but the notation should retain this dependence,
for example $\Gamma_{\theta,{\rm ca}}^{{\rm tran},(y)}({\bf M}_N(h_p),\nu_K;z_k)$.}}

\section{Example of a computation of $G_\eta(\mu)$ for a discrete measure, $\mu$}
\label{App: example}

We assume, here, that $h \in C(\overline\Omega)$. The analogous considerations can be made also for $h \in H^{-s}(\Omega)$, but then we start by replacing $h$ by $h_\tau=R_\tau(h)$. Let $x_j\in {\Omega}$, $j=1,2,\dots,N$. Let $y_j=h(x_j)$ and 
\begin{align} 
&\mu={\bf M}_N (h)=\sum_{j=1}^N \frac 1N \delta_{(x_j,y_j)}\in \mathcal M({\overline\Omega}\times \R),\\
\label{S function and QK B} 
 &  S_{\eta}(\mu)(\cdot)
   =\sum_{k=1}^{K(\eta)} V_0(\eta{{\lambda}^r_k}+1)^{-1}
   \bigg[\int_{{\overline\Omega}\times \R} \psi_k(x'){{y}}'\, d\mu(x',y')\bigg] \psi_k(\cdot)\in
   C^k({\overline\Omega}),
\end{align}
see \eqref{S function and QK}.
Denote $h_\eta:=  S_{\eta}(\mu).$
Then,
$
{\Gamma}_\eta(\mu,(x,y))=\big(x,A(h_\eta)(x)\big)\in {\overline\Omega}\times \R
$
and
$$
G_\eta(\mu)={\Gamma}_\eta(\mu,\cdot)_\#\mu=\sum_{j=1}^N \frac 1N\delta_{(x_j,z_j)},\quad \hbox{where }z_j=A(h_\eta)(x_j).
$$

 \section{Training of transformers, architecture, and the loss function}\label{app: training transformers architecture}

In this appendix we consider the representations \eqref{positional encoding Psi} and \eqref{positional encoding Psi2} of nonlinear operator $A$. We recall that by \eqref{key formula 1} and  \eqref{key formula 2},
the map $\tilde {\Gamma}_\eta$ in \eqref{positional encoding Psi}
satisfies
	 \begin{equation}\label{key formula 1 copy}
	\tilde {\Gamma}_\eta(\gamma_h\otimes \nu,((x,y),(x',y')))
	:= \big((x,y),\, {\Gamma}_\eta(\gamma_h,(x',y'))\big), \quad \gamma_h=(\rho_1)_\#(\gamma_h\otimes \nu)
	\end{equation}
	where
	 \begin{equation}\label{Gamma formula 1B}
	{\Gamma}_\eta(\mu,(x',y')) = \big(x', A({\tilde S}_{\eta}(\mu))(x')\big)
	\end{equation}
	is a mapping that is independent of the measure $\nu$, and therefore of the points $x'_j$. This suggests
	that one can learn an approximation for the operator $A$ by training a neural network of the form
	\beq
	\mu={\bf M}_N(h)\to  \bra  \tilde {\Gamma}_\eta(\mu \otimes\delta_{(x',0)},\cdot))_\#(\mu\otimes \delta_{(x',0)}), \tilde \pi_2 \ket
	\eeq

	Observe that the target function ${\Gamma}_\eta(\mu,(x',y'))$ has
	the property that its $y'$-component, $${\Gamma}^{\mathrm{tran},(y')}_\theta(\mu \otimes\nu,((x,y),(x',y'))) =
	{\Gamma}^{\mathrm{tran},(y')}_\theta(\mu \otimes\nu;x')$$ depends only on the variable $x'$.
	Hence,
	 similarly to the formulas \eqref{key formula 1 copy} and \eqref{Gamma formula 1B}, we may assume that
	we design the architecture of the neural networks that we use in training,
	 $$(\mu \otimes\nu,((x,y),(x',y'))) \to\tilde {\Gamma}^{\rm tran}_\theta(\mu \otimes\nu,((x,y),(x',y')))$$
	so that  their $y'$-components, $${\Gamma}^{\mathrm{tran},(y')}_\theta(\mu \otimes\nu,((x,y),(x',y'))) =
	{\Gamma}^{\mathrm{tran},(y')}_\theta(\mu \otimes\nu;x')$$ depend only on the variable $x'$.
It  defines a map
\ba
{\Gamma}^{\rm tran}_\theta&:&{\mathcal P}((\overline \Omega\times [-L,L]^n)\times (\overline D\times [-L,L]^n))
\times ((\overline \Omega\times [-L,L]^n)\times (\overline D\times [-L,L]^n))
\\& &\to
(\overline \Omega\times [-L,L]^n)\times (\overline D\times [-L,L]^n)
\ea
is of the form \eqref{G tran formula} and  depend on parameters $\theta\in \Theta$ and minimize
the cost function
	\begin{equation}\label{cost function}
	  L(\theta)=\sum_{p=1}^P\sum_{k=1}^K\bigg| \bigg \bra
	  \bigg(\tilde {\Gamma}^{\rm tran}_\theta({\bf M}_N(h_p)\otimes\delta_{(x'_k,0)},\cdot)\bigg)_\#({\bf M}_N(h_p)\otimes\delta_{(x'_k,0)}),
	  \tilde \pi_2\bigg\ket-
	  A(h_p)(x'_k)\bigg|^2,
	\end{equation}
	when training data consists of function $h_p\in C(\overline \Omega;{[-L,L]}^n)$, $p=1,2,\dots,P$ and the values
	of the functions $A(h_p)$ at the points $x'_k\in \overline D$, $k=1,\dots,K$.
	In more explicit way, in \eqref{cost function} we have
	\beq
	& &  \bigg \bra
	  \bigg(\tilde {\Gamma}^{\rm tran}_\theta({\bf M}_N(h_p)\otimes\delta_{(x'_k,0)},\cdot)\bigg)_\#({\bf M}_N(h_p)\otimes\delta_{(x'_k,0)}),
	  \tilde \pi_2\bigg\ket
	\\
	&=&\bigg \bra
	  {\bf M}_N(h_p)\otimes\delta_{(x'_k,0)}\, ,\,
	  \tilde \pi_2\bigg(\tilde {\Gamma}^{\rm tran}_\theta({\bf M}_N(h_p)\otimes\delta_{(x'_k,0)},\cdot)\bigg)\bigg\ket
	\\
	&=&\bigg \bra
	  {\bf M}_N(h_p)\otimes\delta_{(x'_k,0)}\, ,\,
	  \pi_2\circ \rho_2\bigg(\tilde {\Gamma}^{\rm tran}_\theta({\bf M}_N(h_p)\otimes\delta_{(x'_k,0)},\cdot)\bigg)\bigg\ket
	 \eeq
	As we have assumed that  $(\mu \otimes\nu,((x,y),(x',y'))) \to\tilde {\Gamma}^{\rm tran}_\theta(\mu \otimes\nu,((x,y),(x',y')))$
	is a context function whose $y'$-component depends only on the variable $x'$,  the above implies that
	 \beq
	& &  \bigg \bra
	  \bigg(\tilde {\Gamma}^{\rm tran}_\theta({\bf M}_N(h_p)\otimes\delta_{(x'_k,0)},\cdot)\bigg)_\#({\bf M}_N(h_p)\otimes\delta_{(x'_k,0)}),
	  \tilde \pi_2\bigg\ket
	  \\
	&=&\bigg \bra
	  {\bf M}_N(h_p)\otimes\delta_{(x'_k,0)}\, ,\,
	  \tilde \pi_2\bigg(\tilde {\Gamma}^{\rm tran}_\theta({\bf M}_N(h_p)\otimes\delta_{(x'_k,0)},\cdot)\bigg)\bigg\ket
	\\
	&=&\bigg \bra
	  {\bf M}_N(h_p)\otimes\delta_{(x'_k,0)}\, ,\,
	  \pi_2\circ \rho_2\bigg(\tilde {\Gamma}^{\rm tran}_\theta({\bf M}_N(h_p)\otimes\delta_{(x'_k,0)},\cdot)\bigg)\bigg\ket
	  \\
	  &=&\bigg \bra
	  {\bf M}_N(h_p)\otimes\delta_{(x'_k,0)}\, ,\,
	  \tilde {\Gamma}^{\mathrm{tran},(y')}_\theta({\bf M}_N(h_p)\otimes\delta_{(x'_k,0)},\cdot)\bigg\ket
	  \\
	&=&\bigg \bra  \delta_{x'_k},\, {\Gamma}^{\mathrm{tran},(y')}_\theta({\bf M}_N(h_p) \otimes \delta_{(x'_k,0)};x')
	 \bigg\ket
	  \\
	&=& {\Gamma}^{\mathrm{tran},(y')}_\theta({\bf M}_N(h_p) \otimes \delta_{(x'_k,0)};x'_k)
	\eeq
	where $\Gamma^{\rm tran}_\theta:{\mathcal P}((\overline \Omega\times [-L,L]^n)\times (\overline D\times [-L,L]^n))
	\times ((\overline \Omega\times [-L,L]^n)\times (\overline D\times [-L,L]^n))\to
	(\overline \Omega\times [-L,L]^n)\times (\overline D\times [-L,L]^n)$ is a composition of MLPs and
	transformers of the form \eqref{G tran formula} and ${\Gamma}^{\mathrm{tran},(y')}_\theta(\mu \otimes\nu;x'_k)$
	is its $y'$-component evaluated at the point $x'_k$ (and at the arbitrary values of $(x,y)$ and $y'$).

\section{Numerical illustrations}
\label{App:experiments}

We now illustrate that a function graph transformer can recover a target operator from samples of its input-output behavior. The purpose of this appendix is to provide a numerical illustration of the theory, not to claim state-of-the-art performance on any particular operator learning task.

\subsection{Setup and learning task}

Our domain is the two-dimensional torus
\[
    \mathbb T^2=(\mathbb R/\mathbb Z)^2,
\]
identified with the unit square with periodic boundary conditions.  Inputs are scalar functions sampled at point clouds.  For an input function $h$ and support points $x_j\in\mathbb T^2$, the model receives the unordered graph tokens
\[
    \{(x_j,h(x_j))\}_{j=1}^N,
    \qquad \text{corresponding to} \qquad
    {\bf M}_N(h)=\frac1N\sum_{j=1}^N\delta_{(x_j,h(x_j))}.
\]
The target operator $A$ is a small, randomly initialized one-layer Fourier neural operator (FNO) \citep{li2021fourier}. The target FNO used Fourier modes $(6,6)$, hidden width $16$, projection width $128$, with a total of $8753$ fixed random parameters.

\subsection{Transformer configurations}

\paragraph{Sinusoidal positional encoding.} We encode coordinates with fixed sinusoidal features rather than raw coordinates to match our torus geometry and avoid artificial boundaries at the edge of the unit square \citep{vaswani2017attention,tancik2020fourier,mildenhall2020nerf}. 
Specifically, each sampled input is represented by tokens \((x_j,r_j^{(0)})\), where \(x_j\in\mathbb T^2\) is the support coordinate and \(r_j^{(0)}\) is an initial latent embedding of the observed value \(h(x_j)\), conditioned on fixed sinusoidal features \(\phi(x_j)\).
The coordinate feature map is $\phi:\mathbb R^2 \to \mathbb R^{24}$,
\[
\phi(x_1, x_2)
:=
\Big(
\sin(2\pi \omega_m x_\ell),
\cos(2\pi \omega_m x_\ell)
\Big)_{\ell=1,2;\,m=1,\dots,6} 
\]
with frequencies
\[
\omega_m = 16^{(m-1)/5},\qquad m=1,\dots,6.
\]
At every transformer layer, the coordinates \(x_j\) remain fixed. The attention and MLP layers update only the latent variables \(r_j^{(\ell)}\), while receiving \(\phi(x_j)\) as deterministic coordinate conditioning. Thus the layer map has the form
\[
    (x_j,r_j^{(\ell)})_{j=1}^N
    \longmapsto
    (x_j,r_j^{(\ell+1)})_{j=1}^N,
\]
which preserves the support points.

\paragraph{Same-domain and query mode.} Our transformer implementation supports two prediction modes.  
In \emph{same-domain} mode, the model predicts output values on the same coordinates as the input samples,
\[
    \{(x_j,h(x_j))\}_{j=1}^N
    \longmapsto
    \{(x_j,A(h)(x_j))\}_{j=1}^N.
\]
For same-domain prediction, the model applies a seven-layer transformer directly to the sampled input graph tokens \((x_j,h(x_j))\). The support coordinates are preserved throughout the network, and a pointwise readout maps the final latent at each input location \(x_j\) to the predicted value \(A(h)(x_j)\).

In \emph{query} mode, the model predicts at arbitrary query coordinates $y_k\in\mathbb T^2$,
\[
    \left(\{ x_j, h(x_j) \}_{j=1}^N, x'_k \right)
    \longmapsto
    A(h)(x'_k).
\]
For query prediction, the model first applies a six-layer input branch to the sampled graph tokens $\{ x_j, h(x_j) \}_{j=1}^N$ and a two-layer pointwise query branch to the query coordinates. For each query point \(y_k\), it then forms the fiber tokens \(((x_j,h(x_j)),(x'_k,0))\), applies a fiberwise product-space self-attention block followed by two residual token-wise MLP blocks, and reads out from the resulting fiber.

Both the same-domain and query mode transformers lift their input tokens to an embedding space with dimension of $144$. All attention blocks use $4$ heads.

\subsection{Training protocol}

Training data was generated online: at each optimization step, we sampled a fresh batch of random input functions \(h\) and fresh point clouds, rather than drawing from a fixed stored dataset. Input functions were sampled online from a truncated periodic Gaussian Fourier field: sine and cosine coefficients were drawn independently with frequency-decaying variance, rescaled to have approximate sup-norm \(0.9\). The target values \(A(h)\) were computed on the fly by discretizing \(h\) onto a $64 \times 64$ torus grid, applying the fixed one-layer FNO teacher, and interpolating the output to the requested locations \(x_j\) or \(x'_k\).

The transformers were trained over \(20{,}000\) training steps with batch size \(32\). After a linear warmup period lasting $1000$ steps, we used a cosine learning rate from \(10^{-3}\) decaying to \(10^{-4}\). Each step of query-mode training used \(K_{\rm train}=16\) query points. 

In both modes, input support points were sampled uniformly at random during training. In query mode, the query points were also sampled uniformly at random; in same-domain mode there are no separate query points, since the model is trained to predict \(A(h)(x_j)\) at the same sampled support locations \(x_j\).

For predictions $\widehat u$ and targets $u=A(h)$, we report
\[
    \mathrm{rel}\,L^2=\frac{\|\widehat u-u\|_{L^2}}{\|u\|_{L^2}},\qquad
    \mathrm{rel}\,L^\infty=\frac{\|\widehat u-u\|_{L^\infty}}{\|u\|_{L^\infty}}.
\]
The \(L^2\) norms are computed by empirical quadrature over the evaluation points for each held-out function, the \(L^\infty\) norms by the corresponding pointwise maximum, and the reported values are averages over a number of held-out functions \(h\) specified in each table caption. Same-domain mode use $N=625$ evaluation points, while query-mode norms use $K=4096$ evaluation points.

Where appropriate, the tables below report $95\%$ confidence interval half-widths computed over per-function error samples.

\subsection{Recovery results}

Table~\ref{tab:fno-teacher-main-new} gives the main recovery results over $128$ held-out functions. Figure \ref{fig:fno-teacher-2d-examples} shows representative samples on input functions and sampling regimes not seen during training.

\begin{figure*}[t]
\centering
\captionsetup[subfigure]{skip=1pt}

\begin{subfigure}{0.95\textwidth}
    \centering
    \includegraphics[width=\linewidth]{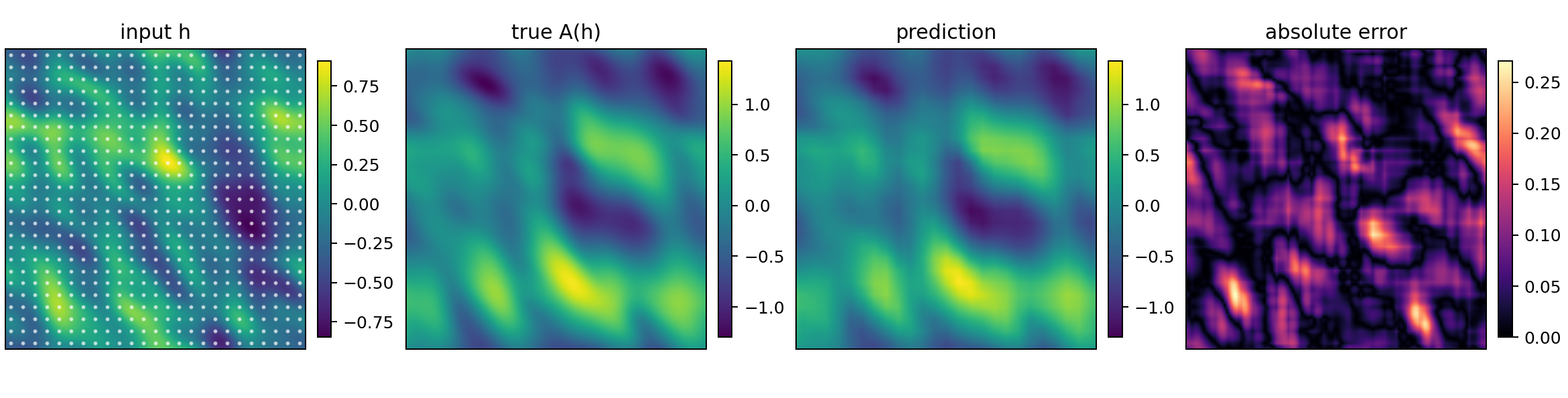}
    \caption{Same-domain mode, grid input sampling.}
\end{subfigure}

\vspace{0.75em}

\begin{subfigure}{0.95\textwidth}
    \centering
    \includegraphics[width=\linewidth]{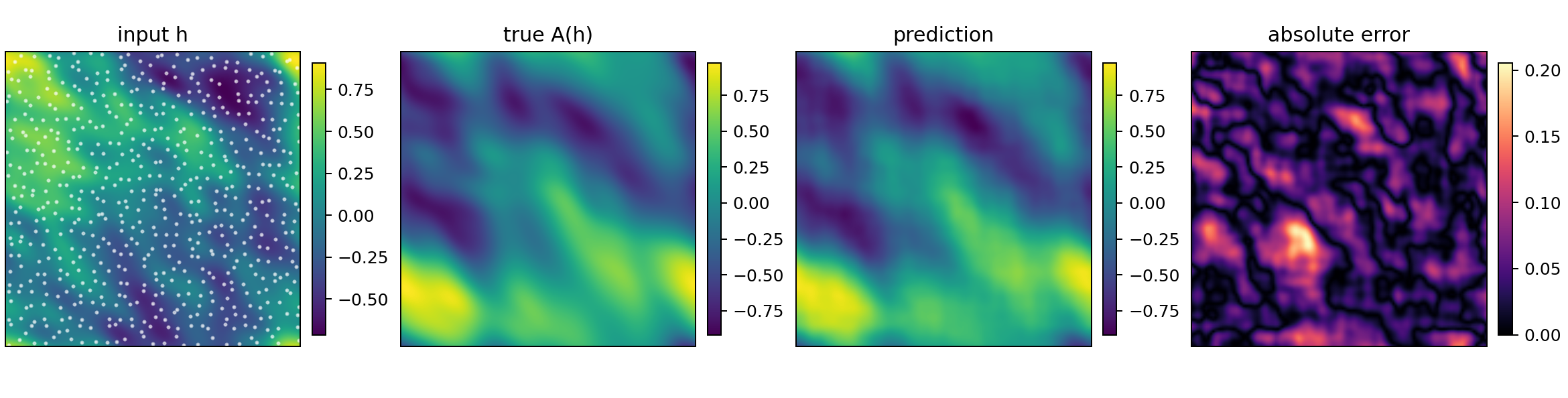}
    \caption{Query mode, jittered-grid input sampling.}
\end{subfigure}

\vspace{0.75em}
\begin{subfigure}{0.95\textwidth}
    \centering
    \includegraphics[width=\linewidth]{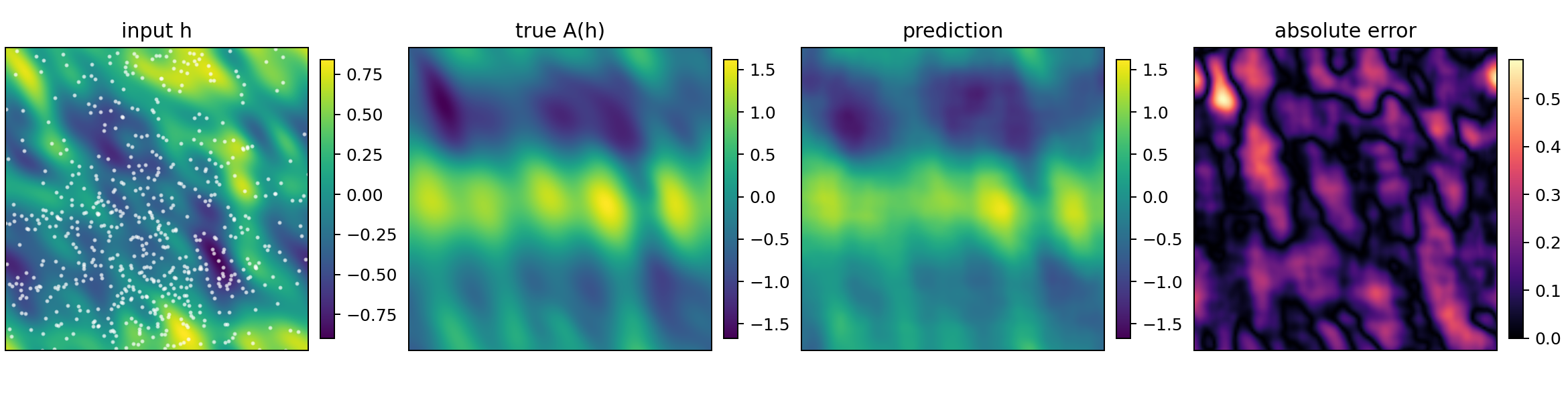}
    \caption{Query mode, Gaussian input sampling.}
\end{subfigure}

\caption{Representative two-dimensional FNO-teacher recovery examples for the trained same-domain and query-mode transformers. The white dots in the input panel mark the sampling locations $\{x_j\}_{j=1}^N$. For the same-domain example (a), the network predicts at the input grid locations, and values at the displayed pixels are obtained by periodic bilinear interpolation from these predicted grid values. For the query-mode examples (b) and (c), the displayed pixel locations are used directly as query points, so all pixel values are returned by the trained network.}
\label{fig:fno-teacher-2d-examples}
\end{figure*}
 
\begin{table}[t]
\centering
\setlength{\tabcolsep}{3pt}
\caption{Main recovery results for a fixed random one-layer FNO teacher. Errors are per-function means over $128$ held-out functions; intervals are $95\%$ confidence interval half-widths. Timings were recorded on a single NVIDIA RTX PRO 6000 Blackwell GPU with 96GB RAM.}
\label{tab:fno-teacher-main-new}
\begin{tabular}{lccccc}
\toprule
mode &
params &
rel. $L^2$ &
rel. $L^\infty$ &
train time \\
\midrule
same-domain &
$2{,}315{,}666$ &
$0.0609\pm0.0031$ &
$0.0689\pm0.0034$ &
$406.7$ s \\
query &
$3{,}381{,}410$ &
$0.1106\pm0.0048$ &
$0.1566\pm0.0065$ &
$1830.2$ s \\
\bottomrule
\end{tabular}
\end{table}

\subsection{Query-mode comparison}

Table~\ref{tab:fno-teacher-baselines-new} compares the product-space query model against the same-domain transformer followed by interpolation to the query coordinates.  Both methods are evaluated at $K=4096$ query locations over $128$ held-out functions.  Product-space self-attention substantially improves arbitrary-coordinate prediction over the interpolation baseline.

\begin{table}[t]
\centering
\setlength{\tabcolsep}{4pt}
\caption{Query-mode vs interpolation. The interpolation baseline first applies the same-domain transformer and then interpolates its prediction to query points. Product-space self-attention provides substantially better accuracy at arbitrary query points than the interpolation baseline.}
\label{tab:fno-teacher-baselines-new}
\begin{tabular}{llcc}
\toprule
& mode & rel. $L^2$ & rel. $L^\infty$ \\
\midrule
product-space self-attention & query &
$0.1117\pm0.0043$ &
$0.1588\pm0.0080$ \\
interpolation & same-domain &
$0.2260\pm0.0067$ &
$0.4764\pm0.0162$ \\
\bottomrule
\end{tabular}
\end{table}

\subsection{Irregular sampling and empirical measure convergence}

We next probe robustness away from the training input resolution and sampling distribution. Table~\ref{tab:fno-teacher-resolution-new} fixes the evaluation set to $K=4096$ query locations and varies the number of observed input samples $N$, using $16$ held-out functions for each resolution. Although the query-mode model was trained only with $N=625$ input samples, the relative $L^2$ error decreases steadily as more input samples are provided, from $0.3582\pm0.0436$ at $N=128$ to $0.0800\pm0.0058$ at $N=2048$. This resolution-transfer behavior is consistent with learning an approximation to the underlying continuous operator, rather than a predictor tied to a fixed input discretization.

\begin{table}[t]
\centering
\setlength{\tabcolsep}{4pt}
\caption{Resolution transfer for the trained query-mode product-space transformer. Each row uses $16$ held-out functions and $K=4096$ query locations. The transformer was trained with $N=625$ input samples, and was never provided inputs with the values of $N$ in this table.}
\label{tab:fno-teacher-resolution-new}
\begin{tabular}{rcc}
\toprule
$N$ & rel. $L^2$ & rel. $L^\infty$ \\
\midrule
$128$  & $0.3582\pm0.0436$ & $0.4404\pm0.0517$ \\
$256$  & $0.2004\pm0.0260$ & $0.2869\pm0.0321$ \\
$512$  & $0.1184\pm0.0115$ & $0.1680\pm0.0151$ \\
$1024$ & $0.0929\pm0.0127$ & $0.1306\pm0.0117$ \\
$2048$ & $0.0800\pm0.0058$ & $0.1205\pm0.0156$ \\
\bottomrule
\end{tabular}
\end{table}

Table~\ref{tab:fno-teacher-sampling-new} evaluates the trained query-mode model at $K=4096$ query locations while changing the input sampling distribution. Each row uses $16$ held-out functions. Although training used uniformly random input samples, the model attains its lowest errors on structured grid and jittered-grid inputs, and degrades gradually under uniform random and Gaussian sampling. We note that Gaussian sampling test is intended as a stress-test, and points $x_1, x_2, \dots$ sampled in this way are not regularly distributed in the sense of Definition \ref{def:regularly_distributed}.

\begin{table}[t]
\centering
\setlength{\tabcolsep}{5pt}
\caption{Irregular-sampling evaluation for the trained query-mode model. Each row uses $16$ held-out functions and $K=4096$ query locations. The model was trained using uniformly random input sampling, and was not trained on the other sampling methods in this table.}
\label{tab:fno-teacher-sampling-new}
\begin{tabular}{lcc}
\toprule
input sampling & rel. $L^2$ & rel. $L^\infty$ \\
\midrule
grid & $0.0714\pm0.0063$ & $0.1079\pm0.0155$ \\
jittered grid & $0.0832\pm0.0095$ & $0.1276\pm0.0200$ \\
uniform random & $0.1202\pm0.0165$ & $0.1821\pm0.0253$ \\
Gaussian & $0.1819\pm0.0247$ & $0.2436\pm0.0322$ \\
\bottomrule
\end{tabular}
\end{table}

\subsection{Limitations}

These numerical experiments are small-scale, and are only intended to illustrate certain aspects of the theory. Our numerical illustration uses an FNO as a target operator $A: C(\mathbb{T}^2, [-L,L]) \to C(\mathbb{T}^2, [-L,L])$, which fits the setting of Theorem \ref{thm:approximation-trans self-attention Takashi}. We do not illustrate a learned operator between negative-order Sobolev spaces, the setting of Theorem \ref{main thm}.



\commented{

\section{Finite sums of point measures and interpolation}
\label{sec:  point measures and interpolation}

\noindent
\paragraph{Interpolation.} We may consider also interpolation in a non-regular grid: When ${\bar x}_j \in \Omega \subset \R^d$, $j=1,2,\dots$ are not regularly distributed points, we denote the corresponding ordered sequence by ${\bf X} = ({\bar x}_j)_{j=1}^\infty$. Then, we define the interpolation measure
\ba
  {\bf M}^{\bf X}_{N,r}(h)
  = \sum_{j=1}^N \rho^{\bf X}(j,N,r) \delta_{({\bar x}_j, h({\bar x}_j))} ,
\ea 
where $r > 0$ is a parameter, ${\rho^{\bf X}(j,N,r)}$ is the relative density of the points
\ba
  \rho^{\bf X}(j,N,r) &=& \frac {\hbox{exp}(-|{\bar x}_j-x_k^{(0)}|^2/r^2)}
  {\sum_{i=1}^N \hbox{exp}(-|x_i^{(0)}-x_k^{(0)}|^2/r^2)} .
\ea 
We also denote
\ba
  {\bf M}^{\bf X}_{N,r,\tau}(h) = \sum_{j=1}^N \rho^{\bf X}(j,N,r)
  \delta_{({\bar x}_j, h_\tau({\bar x}_j))} .
\ea 
We say that ${\bar x}_j \in \Omega \subset \R^d$, $j=1,2,\dots$ are almost regularly distributed points with scales $r :\ \mathbb Z_+ \to \R_+$, if for all $h \in C(\overline\Omega)$
\ba
& &\lim_{N \to \infty} {\bf M}^{\bf X}_{N,\tau,r}(h)
  = \gamma_h\quad \hbox{in the ${\color{black}W_1}$-metric.}
\ea
Then, we could consider replacing $ {\bf M}_N(h)$ and ${\bf M}_{N,\tau}(h)$ by ${\bf M}^{\bf X}_{N,r(N)}(h)$ and ${\bf M}^{\bf X}_{N,r(N),\tau}(h)$, respectively, in our considerations below.

\medskip

\noindent
{\color{red} [Can we develop any relation with the Nadaraya-Watson interpolation used in CViT, eq.(7)?]}
}
\end{document}

%% file: figures/diagramB.tex
\begin{tikzpicture}[line cap=round, line join=round]
\tikzset{
  axis/.style={->, thick},
  base curve/.style={blue, very thick},
  lifted curve/.style={red!70!black, very thick},
  ribbon/.style={gray!45, opacity=0.28},
  sample bar/.style={black, thick},
  sample dot/.style={black},
  diagram node/.style={font=\normalsize, inner sep=3pt},
  diagram arrow/.style={->, very thick},
}

\def\xmin{-0.1}
\def\xmax{2.0}
\def\xaxismax{2.25}
\def\ymin{-0.1}
\def\yaxismax{2.55}
\def\yaxismaxthree{2.75}
\def\zaxismax{1.45}
\def\barxs{-0.04,0.12,0.27,0.44,0.58,0.73,0.91,1.08,1.22,1.39,1.57,1.76,1.94}

\def\eastshift{15.0cm}
\def\topshift{4.3cm}
\def\bottomshift{2.2cm}

\newcommand{\hwest}[1]{sin(#1 r)+cos(3*#1 r)}
\newcommand{\liftwest}[1]{0.5+0.2*sin(2*#1 r)}

\newcommand{\heast}[1]{-0.7-0.9*sin(3*#1 r)+2*cos(0.5*#1 r)}
\newcommand{\lifteast}[1]{0.5+0.2*sin(2*#1 r)}

\newcommand{\DrawThreeDPanel}[5]{%
  \begin{scope}[
    shift={#1},
    x={(1.55cm,0cm)},
    y={(0.46cm,0.27cm)},
    z={(0cm,0.58cm)},
  ]
    \draw[axis] (\xmin,0,0) -- (\xaxismax,0,0) node[below right] {$x$};
    \draw[axis] (0,\ymin,0) -- (0,\yaxismaxthree,0) node[above right] {$y$};
    \draw[axis] (0,0,0) -- (0,0,\zaxismax);

    \fill[
      ribbon,
      domain=\xmin:\xmax,
      samples=200,
      variable=\x,
    ]
      plot ({\x},{#2{\x}},{#3{\x}})
      --
      plot[domain=\xmax:\xmin] ({\x},{#2{\x}},0)
      -- cycle;

    \draw[
      lifted curve,
      domain=\xmin:\xmax,
      samples=200,
      variable=\x,
    ]
      plot ({\x},{#2{\x}},{#3{\x}});

    \foreach \x in \barxs {
      \draw[sample bar]
        ({\x},{#2{\x}},0) --
        ({\x},{#2{\x}},{#3{\x}});
      \fill[sample dot]
        ({\x},{#2{\x}},{#3{\x}}) circle (1.5pt);
    }

    \draw[
      base curve,
      domain=\xmin:\xmax,
      samples=200,
      variable=\x,
    ]
      plot ({\x},{#2{\x}},0);

    \pgfmathsetmacro{\labely}{#2{1.55}}
    \pgfmathsetmacro{\labelz}{#3{1.55}}
    \node[lifted curve, anchor=west, font=\small, yshift=#5] at (1.55,\labely,\labelz+0.3) {$#4$};
  \end{scope}
}

\newcommand{\DrawTwoDPanel}[5]{%
  \begin{scope}[shift={#1}, x=1.55cm, y=0.5cm]
    \draw[axis] (\xmin,0) -- (\xaxismax,0) node[right] {$x$};
    \draw[axis] (0,\ymin) -- (0,\yaxismax) node[above] {$y$};

    \draw[
      base curve,
      domain=\xmin:\xmax,
      samples=200,
      variable=\x,
    ]
      plot ({\x},{#2{\x}});

    \pgfmathsetmacro{\labely}{#2{1.55}}
    \node[blue, anchor=west, font=\small, xshift=#4, yshift=#5] at (1.55,\labely) {$#3$};
  \end{scope}
}

\newcommand{\DrawMiddleDiagram}{%
  \node[diagram node] (SW) at (6.10cm,2.40cm)
    {$C(\overline{\Omega},[-L,L]^n)$};
  \node[diagram node] (NW) at (6.10cm,4.90cm)
    {$\mathcal{P}(\overline{\Omega}\times[-L,L]^n)$};
  \node[diagram node] (NE) at (13.00cm,4.90cm)
    {$\mathcal{P}(\overline{\Omega}\times[-L,L]^n)$};
  \node[diagram node] (SE) at (13.00cm,2.40cm)
    {$C^{1}(\overline{\Omega},[-L,L]^n)$};

  \draw[diagram arrow] (SW.north) -- node[left=6pt] {${\bf M}$} (NW.south);
  \draw[diagram arrow] (NW.east) -- node[below=6pt, align=center] {function graph transformer\\[0.2em]$G_\Gamma(\mu) = \Gamma(\mu, \cdot)_\# \mu$} (NE.west);
  \draw[diagram arrow] (NE.south) -- node[right=6pt] {${\bf F}$} (SE.north);
  \draw[diagram arrow] (SW.east) -- node[above=6pt, align=center] {target operator $A$} (SE.west);

}

\DrawThreeDPanel{(0,\topshift)}{\hwest}{\liftwest}{\gamma_h}{4pt}
\DrawTwoDPanel{(0,\bottomshift)}{\hwest}{h}{4pt}{0pt}

\DrawThreeDPanel{(\eastshift,\topshift)}{\heast}{\lifteast}{G_\Gamma(\gamma_h)}{0pt}
\DrawTwoDPanel{(\eastshift,\bottomshift)}{\heast}{A(h)}{7pt}{3pt}

\DrawMiddleDiagram
\end{tikzpicture}